\documentclass[twoside]{article}

\usepackage{amsmath,amsfonts,bm}

\def\eqref#1{equation~\ref{#1}}

\def\1{\bm{1}}

\DeclareMathAlphabet{\mathsfit}{\encodingdefault}{\sfdefault}{m}{sl}
\SetMathAlphabet{\mathsfit}{bold}{\encodingdefault}{\sfdefault}{bx}{n}

\usepackage{hyperref}
\usepackage{url}
\usepackage{graphicx,appendix}
\usepackage{booktabs}
\usepackage{natbib}
\usepackage{algorithm,algcompatible}

\usepackage{multirow}
\usepackage{makecell,enumitem}

\usepackage{amsmath,enumitem}
\usepackage{amssymb}
\usepackage{mathtools}
\usepackage{amsthm}
\allowdisplaybreaks
\usepackage{caption}
\usepackage{subcaption}

\newtheorem{theorem}{Theorem}
\newtheorem{proposition}{Proposition}
\newtheorem{corollary}{Corollary}
\newtheorem{lemma}{Lemma}

\newtheorem{assumption}{Assumption}

\usepackage[accepted]{aistats2026}

\RequirePackage{lmodern}  

\begin{document}

%

%

\twocolumn[

\aistatstitle{Conditional Vendi Score: Prompt-Aware Diversity Evaluation for Generative AI Models and LLMs}


\aistatsauthor{
Mohammad Jalali$^{1}$
\And
Azim Ospanov$^{1}$
\And
Amin Gohari$^{2}$
\And
Farzan Farnia$^{1}$
}

\aistatsaddress{ $^{1}$Department of Computer Science and Engineering, The Chinese University of Hong Kong\\
$^{2}$Department of Information Engineering, The Chinese University of Hong Kong \\
\texttt{\{mjalali24, aospanov9, farnia\}@cse.cuhk.edu.hk, agohari@ie.cuhk.edu.hk}
} ]

\begin{abstract}
Generative models guided by text prompts are widely evaluated for fidelity and prompt alignment, yet their ability to produce diverse outputs remains underexplored. Existing diversity metrics such as Vendi and RKE, which are based on the von Neumann and Rényi entropies of kernel matrices, were developed for unconditional models and cannot distinguish prompt-induced from model-induced variability. We address this gap by introducing \textit{Conditional-Vendi} and \textit{Conditional-RKE}, diversity measures derived from the conditional entropy of positive semidefinite matrices. These scores isolate model-induced diversity in prompt-guided generation, with Conditional-RKE enjoying an $O(1/\sqrt{n})$ convergence rate. For Conditional-Vendi, we introduce a truncated-spectrum approximation that yields scalable and consistent estimates. Experiments on text-to-image, image-captioning, and LLM tasks show that the conditional scores recover ground-truth diversity orderings and can also guide diffusion models toward more diverse samples. The codebase is available at \href{https://github.com/mjalali/conditional-vendi}{https://github.com/mjalali/conditional-vendi}.
\end{abstract}

\section{Introduction}
\label{sec:intro}
Prompt-guided generative AI systems, including large language models (LLMs) \citep{brown2020language}, text-to-image models \citep{Rombach_2022_CVPR,ramesh2022hierarchical,saharia2022imagen}, and text-to-video models \citep{ho2022videodiffusion,ho2022imagenvideo,openai2024sora}, have achieved remarkable success across a wide range of applications. In these models, sample generation is conditioned on an input text prompt, with the goal of producing outputs aligned to the prompt. This conditional generation mechanism distinguishes prompt-guided models from traditional unconditional generative models \citep{kingma2013auto,goodfellow2014generative}, which aim to mimic the overall data distribution without a guiding input. Because most evaluation metrics for generative models were originally developed in the unconditional setting, the recent literature has sought to design new measures that better capture the properties of text-conditioned models.  

Current evaluation metrics for prompt-guided models primarily emphasize \emph{fidelity}: the quality of generated outputs and their consistency with the input text. A common approach is to compute similarity in a shared embedding space between text and outputs, such as ClipScore \citep{hessel2021clipscore} for text-to-image generation, which uses CLIP embeddings \citep{radford2021learning} to quantify alignment. Such fidelity-based measures ensure that the generated content matches the semantics of the prompt but leave open the question of how to assess the \emph{diversity} of model outputs.

\begin{figure*}
    \centering
    \includegraphics[width=\linewidth]{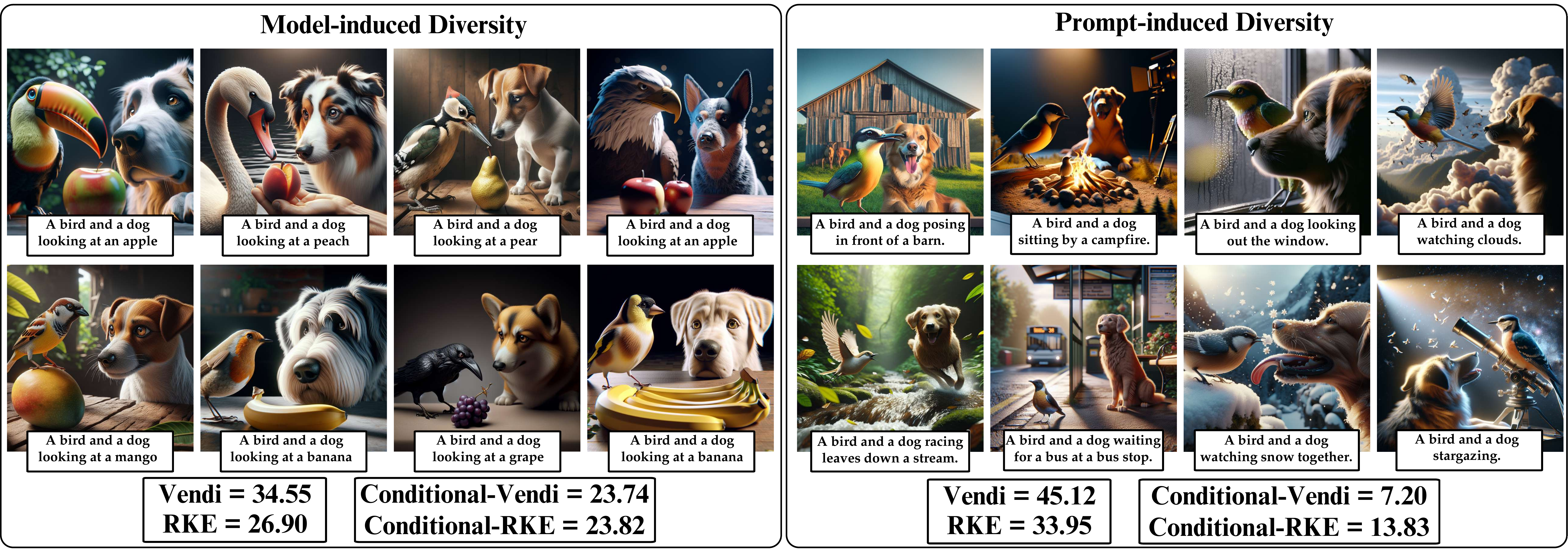}
    \caption{Illustration of \emph{prompt-induced diversity} (left), where variation comes from different prompts, and \emph{model-induced diversity} (right), where variation arises from the generator itself. Unconditional image diversity scores, Vendi and RKE, assign higher diversity to the right side, overlooking prompt effects. Our proposed Conditional-RKE and Conditional-Vendi account for prompts and assign higher diversity to the left side.}
    \label{fig:model_prompt_diversity}
\end{figure*}

Diversity has been extensively studied in unconditional generation, with metrics such as Recall \citep{sajjadi2018assessing,kynkaanniemi2019improved}, Coverage \citep{naeem2020reliable}, Vendi \citep{friedman2022vendi}, and RKE \citep{jalali2023information}. These scores are often applied directly to prompt-guided models, but they conflate two sources of variability: the \emph{prompt-induced diversity}, arising from differences across input prompts, and the \emph{model-induced diversity}, reflecting randomness in outputs for similar prompts. As illustrated in Figure~\ref{fig:model_prompt_diversity}, these two components capture fundamentally different aspects of a model’s behavior, yet existing diversity measures do not disentangle them. This can bias diversity comparisons across models and obscure whether differences are due to richer model variability or merely broader prompt coverage.  

To address this gap, we build on recent entropy-based approaches for unconditional models. The Vendi and RKE scores correspond respectively to the von Neumann entropy and order-2 Rényi entropy of positive semidefinite (PSD) matrices applied to kernel matrices \citep{friedman2022vendi,jalali2023information}. We extend this framework to the \emph{conditional entropy of PSD matrices} \citep{giraldo2014measures} and apply it to kernel matrices, yielding two prompt-aware diversity measures: the \emph{Conditional-Vendi} and \emph{Conditional-RKE} scores. These quantities decompose the kernel-entropy $H(X)$ of generated data $X$ as:
\begin{equation*}
    H\bigl(X\bigr) \, =\, H\bigl(X\, \big|\, T\bigr)\,  + \, I\bigl(X;T\bigr)
\end{equation*}
into the terms of conditional entropy $H(X|T)$ given prompts $T$ and mutual information $I(X;T)$, where $H(X|T)$ serves as a measure of model-induced diversity for a prompt-aware diversity assessment.  

We analyze the statistical behavior of these conditional scores. For Conditional-RKE, we prove an $O(n^{-1/2})$ convergence rate, enabling reliable estimation with moderate samples. For Conditional-Vendi, direct estimation is sample-inefficient due to dimension dependence, affecting its practicality for large-scale generative tasks with sample sizes bounded to several tens of thousands. To address this, we extend the eigenspectrum truncation method for unconditional Vendi score \citep{ospanov2025do} to the conditional setting, which provides scalable and consistent approximate truncated-Conditional-Vendi score.

Figure~\ref{fig:model_prompt_diversity} illustrates the use of (truncated) Conditional-Vendi and Conditional-RKE scores in comparing the diversity of two sets of “dog and bird” samples generated by DALL·E 3. In the first set, the model produces a variety of dog and bird breeds under similar prompts, whereas in the second, the same dog and bird appear in different contextual scenes across diverse prompts. While the prompt-unaware RKE and Vendi scores assign higher diversity to the second set, the prompt-aware Conditional-Vendi and Conditional-RKE scores instead favor the first set, capturing the greater breed-level variation that is orthogonal to prompt differences.

Beyond evaluation, we also leverage prompt-aware diversity scores to guide sample generation in diffusion models. Extending the unconditional Vendi guidance method of \citet{askari2024improving}, we apply Conditional-Vendi guidance to text-conditioned latent diffusion models \citep{rombach2021highresolution}, aiming to promote prompt-aware diversity in sample generation. 

Finally, we validate our framework across text-to-image, text-to-video, and language generation tasks. Using controlled experiments where ground-truth diversity rankings are available, we show that the conditional scores recover the intended rankings and remain computationally tractable at scale. We further demonstrate how Conditional-Vendi can be decomposed across different prompt modes to evaluate diversity conditioned on text categories. In summary, our contributions are as follows:
\begin{itemize}[leftmargin=*]  
    \item We study prompt-aware diversity evaluation for prompt-conditioned generative models.  
    \item We propose \emph{Conditional-Vendi} and \emph{Conditional-RKE} prompt-aware diversity scores.  
    \item We analyze the proposed scores' statistical convergence, and propose a truncation method to reduce the sample complexity of the diversity scores.  
    \item We numerically validate the scores showing correlation with ground-truth model-induced diversity.  
\end{itemize}

\section{Related Work}
\label{sec:related}
\textbf{Evaluation of deep generative models}:  
Metrics for evaluating generative models are generally divided into reference-dependent and reference-free categories \citep{Borji2022}. Reference-dependent metrics compare generated and real data distributions, with common examples including FID \citep{heusel2017gans} and KID \citep{binkowski2018demystifying,wang2023distributedKID}. Other reference-based measures, such as the Inception Score \citep{Salimans2016}, Precision/Recall \citep{sajjadi2018assessing,kynkaanniemi2019improved}, and Density/Coverage \citep{naeem2020reliable}, jointly evaluate fidelity and diversity with respect to a reference dataset.  

Beyond fidelity, several works examine memorization and novelty. These include the authenticity score \citep{alaa2022faithful} and Feature Likelihood Divergence \citep{jiralerspong2023feature} for assessing generalization, as well as the rarity score \citep{han2023rarity} and KEN \citep{zhang24_KEN,zhang2025finc} for quantifying novelty. The memorization metrics are reference-based. In contrast, reference-free evaluations assess quality and diversity directly from the generated data. Notable examples include the Vendi score \citep{friedman2022vendi,pasarkar2023cousins,ospanov2024towards} and RKE score \citep{jalali2023information} for diversity, and \citep{Nguyen2024} for evaluating the quality of generated data. We also note that the diversity-aware online evaluation of generative models has been studied in the related works \citep{hu2025eval,rezaei2024be,hu2025online,hu2025promptwise,jafari2026diversityaware}. We also note the concurrent work by \cite{Ospanov_2025_ICCV} on the prompt-aware diversity evaluation using the Schur complement of CLIP embeddings.  

\textbf{Evaluation of conditional generative models}:  
The evaluation of prompt-based generative models, such as text-to-image and text-to-video systems, has been explored in several recent works. Most metrics focus on measuring alignment between prompts and outputs. A widely used example is CLIPScore \citep{hessel2021clipscore}, which computes cosine similarity in the CLIP embedding space. Other efforts have introduced benchmarks and curated prompt sets to evaluate broader aspects. For instance, HEIM \citep{lee2023holistic} assesses twelve criteria, including text–image alignment, image quality, and bias.  Also, \citet{kim2022mutual} propose the Mutual Information Divergence (MID) score, which fits multivariate Gaussian distributions to text and image representations and estimates their mutual information to quantify relevance in conditional generative models.  

However, alignment- and quality-focused metrics may overlook output diversity. \citet{astolfi2024consistencydiversity} emphasize that metrics centered on style or aesthetics can fail to capture variability across outputs for the same prompt. They propose computing per-prompt diversity using similarity functions and then averaging across prompts. Similarly, \citet{kannen2024aesthetics} extend the Vendi score to the per-prompt setting. Both approaches require generating multiple outputs for each prompt with different seeds. In contrast, our proposed Conditional-Vendi does not require repeated generations; instead, it quantifies model-induced diversity by analyzing variability across prompt types. Our theoretical results interpret Conditional-Vendi as an aggregation of diversity scores across prompt categories.  


\vspace{-2mm}
\section{Preliminaries}
\label{sec:prelim}
\subsection{Generative distributions and notation}

We focus on a conditional generative model that produces a random output $X\in\mathcal{X}$ given an input text prompt $T\in\mathcal{T}$ according to a conditional distribution $P_{X|T}$. For a prompt $T=t$, the model outputs a sample drawn from $P_{X|T=t}$. We consider $n$ sample pairs $\{(t_i,x_i)\}_{i=1}^n$ where the prompts $t_i$ are drawn independently from $P_T$ and, conditional on $t_i$, the outputs $x_i$ are drawn from $P_{X|T=t_i}$ independently across $i$. 

\subsection{Entropy-based diversity scores for unconditional generative models}

Consider generated samples $x_1,\ldots,x_n\in\mathcal{X}$ drawn i.i.d. from the distribution $P_X$ of an unconditional generative model. For a kernel function $k:\mathcal{X}\times\mathcal{X}\to\mathbb{R}$, the (Gram) kernel matrix $K\in\mathbb{R}^{n\times n}$ is
\begin{equation}\label{Eq: Kernel matrix}
    K \,=\, \begin{bmatrix}
    k(x_1,x_1) & \cdots & k(x_1,x_n)\\
    \vdots & \ddots & \vdots\\
    k(x_n,x_1) & \cdots & k(x_n,x_n)
    \end{bmatrix}.
\end{equation}
A function $k$ is called a positive semidefinite (PSD) kernel if the matrix $K$ is PSD for every $n\in\mathbb{N}$ and any choice of $x_1,\ldots,x_n\in\mathcal{X}$. A widely used example is the Gaussian (RBF) kernel with bandwidth $\sigma>0$: $
    k(x,x') = \exp\!\bigl(-{\|x-x'\|_2^2}/{2\sigma^2}\bigr) $.

Suppose $k$ is normalized, i.e., $k(x,x)=1$ for every $x\in\mathcal{X}$. Then the eigenvalues $\lambda_1,\ldots,\lambda_n\ge 0$ of $\tfrac{1}{n}K$ sum to $1$, thus they form a probability distribution. The \emph{von Neumann entropy} of the unit-trace PSD matrix $\tfrac{1}{n}K$ is
\begin{equation}
    H\!\bigl(\tfrac{1}{n}K\bigr) \,:=\, -\mathrm{Tr}\Bigl(\tfrac{1}{n}K \log\!\bigl(\tfrac{1}{n}K\bigr)\Bigr)
    \,=\, \sum_{i=1}^n \lambda_i \log\!\frac{1}{\lambda_i}
\end{equation} \citet{friedman2022vendi} propose using the von Neumann entropy of the normalized kernel matrix to define the \emph{Vendi diversity score}:\vspace{-1mm}
\begin{equation}\label{Eq: Vendi Score}
     \mathrm{Vendi}(x_{1:n})\! :=\! \exp\!\Bigl( H\bigl(\tfrac{1}{n}K\bigr)\Bigr)
     \!=\! \exp\!\Bigl(\sum_{i=1}^n \lambda_i \log \tfrac{1}{\lambda_i}\Bigr)\vspace{-1mm}
\end{equation}
Also, \citet{jalali2023information} propose using \emph{order-2 Rényi entropy} of the normalized kernel matrix, $H_2\!\bigl(\tfrac{1}{n}K\bigr)=\log\!\bigl(1/\sum_{i=1}^n \lambda_i^2\bigr)$, to define the \emph{Rényi Kernel Entropy (RKE)}  score. Denoting the Frobenius norm by $\|\cdot\|_F$, the following holds since for symmetric $\frac{1}{n}K$, the sum of squared eigenvalues is the squared-Frobenius norm:\vspace{-1mm}
\begin{equation}\label{Eq: RKE Score}
     \mathrm{RKE}(x_{1:n}) \,:=\, \exp\!\Bigl(H_2\!\bigl(\tfrac{1}{n}K\bigr)\Bigr)
     \,=\, {\bigl\|\frac{1}{n}K\bigr\|_F^{\,-2}}
\end{equation}

\section{Conditional Vendi and RKE Prompt-Aware Diversity Scores}\vspace{-1mm}
\label{sec:algorithm}
To extend Vendi and RKE to prompt-aware diversity evaluation, we replace the (unconditional) entropy in these scores with the \emph{conditional} entropy of the kernel matrix of generated samples given the kernel matrix of the corresponding prompts. We follow the matrix-based definitions of \citet{giraldo2014measures}, where the joint entropy of two PSD matrices $A$ and $B$ is defined via the von Neumann entropy of the trace-normalized Hadamard (elementwise) product $\tfrac{1}{\mathrm{Tr}(A\odot B)}(A\odot B)$, and the conditional entropy is $H(A\!\mid\!B)=H(A,B)-H(B)$. The Schur product theorem ensures $A\odot B$ is PSD, and therefore the entropy is well-defined.

To evaluate diversity of generated output $X$ given input text prompt $T$, we consider a normalized kernel for outputs $k_{\mathcal{X}}:\mathcal{X}\times\mathcal{X}\to\mathbb{R}$ with $k_{\mathcal{X}}(x,x)=1$ for every $x\in\mathcal{X}$, and a normalized kernel for prompts $k_{\mathcal{T}}:\mathcal{T}\times\mathcal{T}\to\mathbb{R}$ with $k_{\mathcal{T}}(t,t)=1$ for every $t\in\mathcal{T}$. Given pairs $\{(t_i,x_i)\}_{i=1}^n$, define the kernel matrices
\[
K_T=\bigl[k_{\mathcal{T}}(t_i,t_j)\bigr]_{i,j=1}^n,\qquad
K_X=\bigl[k_{\mathcal{X}}(x_i,x_j)\bigr]_{i,j=1}^n.
\]
By normalization, the diagonal entries of $K_T$, $K_X$, and $K_T\odot K_X$ are all $1$, hence $\mathrm{Tr}(K_T)=\mathrm{Tr}(K_X)=\mathrm{Tr}(K_T\odot K_X)=n$. Following the definitions of conditional entropy, the matrix-based conditional entropy specialized for these kernel matrices becomes
\begin{align}\label{Eq: Conditional Entropy}
    H\!\bigl(\tfrac{1}{n}K_X \,\big\vert\, \tfrac{1}{n}K_T\bigr)
    :=&\, H\!\left(\tfrac{1}{n}K_X,\tfrac{1}{n}K_T\right)
    - H\!\left(\tfrac{1}{n}K_T\right) \\
    =&\, H\!\left(\tfrac{1}{n}(K_X\odot K_T)\right)
    - H\!\left(\tfrac{1}{n}K_T\right). \nonumber
\end{align}

\textbf{Conditional-Vendi and Information-Vendi.}
Replacing the von Neumann entropy in Vendi with the conditional entropy in \eqref{Eq: Conditional Entropy}, we define
\begin{align}
    &\mathrm{Conditional}\text{-}\mathrm{Vendi}\bigl(x_{1:n} \,\big\vert\, t_{1:n}\bigr)
     \\
     :=\: &\exp\Bigl(H\!\left(\tfrac{1}{n}(K_X\odot K_T)\right)- H\!\left(\tfrac{1}{n}K_T\right)\Bigr). \nonumber
\end{align}
We also define the matrix-based information:
\begin{align}
    &\mathrm{Information}\text{-}\mathrm{Vendi}\bigl(x_{1:n}; t_{1:n}\bigr)
     \\
    := &\exp\Bigl(H\!\left(\tfrac{1}{n}K_X\right)
    + H\!\left(\tfrac{1}{n}K_T\right)- H\!\left(\tfrac{1}{n}(K_X\odot K_T)\right)\Bigr). \nonumber
\end{align}
These yield the following decomposition of Vendi:
\begin{align}    \mathrm{Vendi}\bigl(x_{1:n}\bigr)
   \: =\;\; &\mathrm{Conditional}\text{-}\mathrm{Vendi}\bigl(x_{1:n} \,\big\vert\, t_{1:n}\bigr) \\
    \qquad &\times \mathrm{Information}\text{-}\mathrm{Vendi}\bigl(x_{1:n} \, ; \, t_{1:n}\bigr). \nonumber
\end{align}

\textbf{Conditional-RKE and Information-RKE.}
Similarly, by using the order-2 Rényi entropy $H_2(\cdot)$ in the definitions, we define the Conditional-RKE and Information-RKE scores:
\begin{align}
    &\mathrm{Conditional}\text{-}\mathrm{RKE}\bigl(x_{1:n} \,\big\vert\, t_{1:n}\bigr)
     \\
    := &\exp\Bigl(H_2\!\left(\tfrac{1}{n}(K_X\odot K_T)\right)- H_2\!\left(\tfrac{1}{n}K_T\right)\Bigr),\nonumber \\
&\mathrm{Information}\text{-}\mathrm{RKE}\bigl(x_{1:n}; t_{1:n}\bigr)
     \\
    := &\exp\Bigl(H_2\!\left(\tfrac{1}{n}K_X\right)
    + H_2\!\left(\tfrac{1}{n}K_T\right)- H_2\!\left(\tfrac{1}{n}(K_X\odot K_T)\right)\Bigr).\nonumber
\end{align}
Considering that for the order-2 entropy of a symmetric PSD $M$, we have $H_2(M)=\log\!\bigl(1/\sum_i \lambda_i(M)^2\bigr)$ and $\sum_i \lambda_i(M)^2=\|M\|_F^2$, and thus the above formulations can be written as follows using their unit diagonals:
\begin{align}\label{Eq: Cond-RKE-defintiion}
\mathrm{Conditional}\text{-}\mathrm{RKE}(x_{1:n} \,\big\vert\, t_{1:n})
&= \frac{\bigl\|K_T\bigr\|_F^{\,2}}
       {\bigl\|K_X\odot K_T\bigr\|_F^{\,2}}, \\
\mathrm{Information}\text{-}\mathrm{RKE}(x_{1:n} \,;\, t_{1:n})
&= \frac{n^2\cdot \bigl\|K_X\odot K_T\bigr\|_F^{\,2}}
       {\bigl\|K_X\bigr\|_F^{\,2}\cdot\,\bigl\|K_T\bigr\|_F^{\,2}}.\nonumber
\end{align}
Therefore, RKE also admits the product decomposition property we discussed for the Vendi score:
\begin{align*}
\mathrm{RKE}\bigl(x_{1:n}\bigr)
&= \mathrm{Conditional}\text{-}\mathrm{RKE}\bigl(x_{1:n} \,\big\vert\, t_{1:n}\bigr) \\
&\quad\times \mathrm{Information}\text{-}\mathrm{RKE}\bigl(x_{1:n}; t_{1:n}\bigr)
\end{align*}

\section{Statistical Convergence of Conditional Diversity Scores}\label{sec: concentration}
\label{sec:convergence}

\begin{figure*}[!t]
    \centering
    \includegraphics[width=\linewidth]{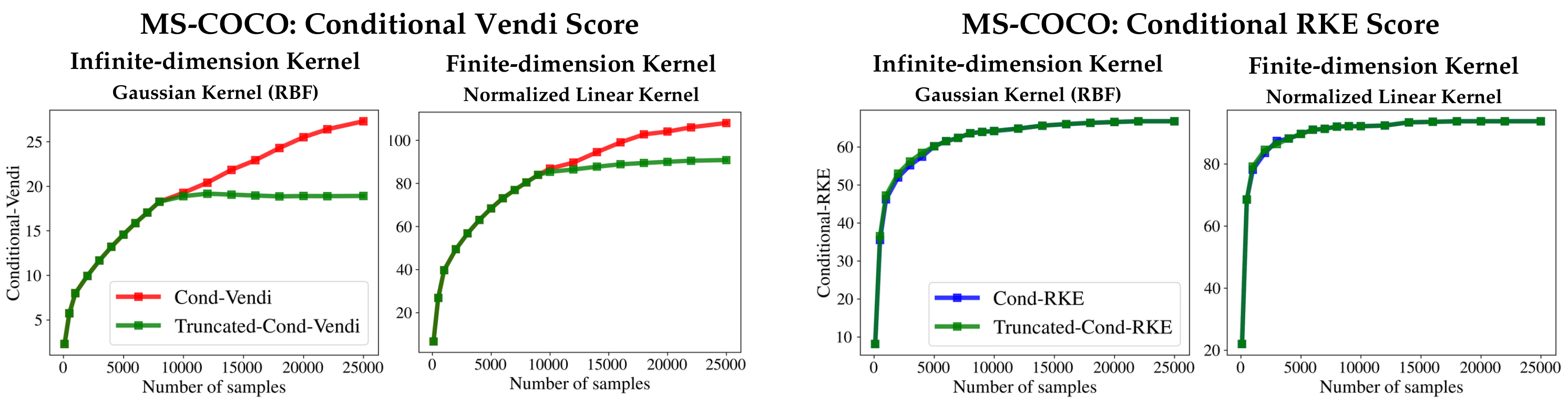}    \caption{Statistical convergence of Conditional-Vendi and RKE scores with different sample sizes on the MS-COCO validation dataset with a finite-dimensional (normalized) linear kernel and infinite-dimensional Gaussian (RBF) kernel.}
    \label{fig:convergence_mscoco_dataset}
\end{figure*}

We now establish finite-sample convergence guarantees for the conditional diversity scores introduced in the previous section. Let $(t_i, x_i)_{i=1}^n$ be i.i.d.\ samples from $P_T \times P_{X|T}$, and let $d_{\mathcal X}$ and $d_{\mathcal T}$ denote the dimensions of the kernel feature maps of $k_\mathcal{X},\, k_\mathcal{T}$. We assume normalized kernels $k_{\mathcal X},k_{\mathcal T}$ with unit self-kernel similarity scores and suppose the following spectral boundedness for the vector of eigenvalues of the population kernel covariance matrices $C_{\mathcal T},\, C_{\mathcal X,\mathcal T}$ of the population distributions $P_T$ and $P_{X|T}$ as 
\begin{align}\label{Eq: spetrcal boundedness}
&0 < m_{\mathcal T} \le \|\lambda(C_{\mathcal T})\|_2 \le M_{\mathcal T}, \nonumber \\
&0 < m_{\mathcal X,\mathcal T} \le \|\lambda(C_{\mathcal X,\mathcal T})\|_2 \le M_{\mathcal X,\mathcal T}.
\end{align}

\textbf{Convergence of Conditional-Vendi Score.}
The following theorem proves a convergence bound on the gap between the empirical and population Conditional-Vendi scores. This result extends the concentration bound of \citet{ospanov2025do} from the unconditional Vendi score to the conditional case.
We defer the theorem proofs to the Appendix, where we also present theoretical results connecting the Conditional-RKE and Conditional-Vendi scores to the component-based aggregation of the unconditional Vendi and RKE scores, given a mixture text distribution with well-separated components.
\begin{theorem}\label{thm:vendi-dim}
Under normalized kernels and spectral boundedness in \eqref{Eq: spetrcal boundedness}, for every $\delta \in (0,1)$, the following holds with probability at least $1-\delta$:
\begin{align*}
&\Bigl|\log\Bigl( \mathrm{Cond\text{-}Vendi}(x_{1:n}\vert t_{1:n})\Bigr)
- \log\Bigl(  \\
&\mathrm{Cond\text{-}Vendi}(P_{T,X})\Bigr)\Bigr|\!\le\!
\sqrt{\!\tfrac{20d_{\mathcal X}d_{\mathcal T}\!\log(4/\delta)}{n}}\!
\log\bigl( n d_{\mathcal X}d_{\mathcal T}\bigr)
\end{align*}
\end{theorem}

The factor $d_{\mathcal X}d_{\mathcal T}$ arises because the Hadamard product $K_X \odot K_T$ corresponds to a tensorized feature space of dimension $d_{\mathcal X}\times  d_{\mathcal T}$. For example, CLIP embeddings with normalized linear kernels ($k_{\text{lin}}(z,z')=\langle z,z'\rangle/\Vert z\Vert_2\Vert z'\Vert_2$) lead to $d_{\mathcal X}= d_{\mathcal T}= 512$, which yields $d_{\mathcal X}d_{\mathcal T}\approx 2.6\!\times\!10^5$. This implies the inefficient sample complexity of the Conditional-Vendi score, given the $O(n^3)$ computational cost of the eigendecomposition of the $n\times n$ kernel matrices in computing the score.

\textbf{Truncated Conditional-Vendi Score.}
To mitigate the discussed curse of dimensionality, we adopt the spectral truncation technique in \citep{ospanov2025do} for unconditional entropy measures. For an integer hyperparameter $t\in\mathbb N$, let $\lambda^{(t)}(M)$ denote the top-$t$ eigenvalues of $M$, all shifted with the same positive constant $c=\bigl(1-\sum_{i=1}^t \lambda_i(M)\bigr)/t$ to sum up to one, and then the $t$-truncated entropy of $M$ is defined:
\[
H^{(t)}(M) \;:=\; \sum_{i=1}^t \lambda^{(t)}_i(M)\,\log \frac{1}{\lambda^{(t)}_i(M)}
\]
We define the truncated Conditional-Vendi score as
\begin{align}
& \mathrm{Conditional\text{-}Vendi}^{(t)}(x_{1:n}\vert t_{1:n}) \nonumber
\\ := \:&\exp\Bigl(H^{(t)}\!\bigl(\tfrac{1}{n}(K_X\odot K_T)\bigr)
- H^{(t)}\!\bigl(\tfrac{1}{n}K_T\bigr)\Bigr).
\end{align}

\begin{theorem}\label{thm:vendi-trunc}
Under normalized kernels and spectral boundedness in (\ref{Eq: spetrcal boundedness}), for any fixed $t \in \mathbb N$ and $\delta \in (0,1)$, the following holds with probability at least $1-\delta$:
\begin{align*}
&\Bigl|\log\Bigl( \mathrm{Cond\text{-}Vendi}^{(t)}(x_{1:n}\vert t_{1:n})\Bigr)- \log\Bigl( \\
& \mathrm{Cond\text{-}Vendi}^{(t)}(P_{T,X})\Bigr) \Bigr|\le
\sqrt{\frac{20\,t\,\log(4/\delta)}{n}}\log(nt)
\end{align*}
\end{theorem}

The truncated estimator achieves an $\tilde O(\sqrt{t/n})$ rate independent of $d_{\mathcal X},d_{\mathcal T}$. In practice, choosing $t$ between $10^3$ and $10^4$ captures most of the spectral mass while enabling efficient computation via partial eigendecomposition in $O(n^2t)$ time.

\textbf{Convergence of Conditional-RKE Score.}
In contrast, Conditional-RKE avoids dimension dependence as long as the prompt and output underlying entropies are bounded. The closed form in \eqref{Eq: Cond-RKE-defintiion} depends only on Frobenius norms computable in $O(n^2)$. The following is our concentration bound for the Conditional-RKE score.
\begin{theorem}\label{thm:crke}
Under normalized kernels and spectral boundedness in \eqref{Eq: spetrcal boundedness}, for every $\delta\in(0,1)$, the following holds with probability at least $1-\delta$:
\begin{align*}
&\bigl|\mathrm{Cond\text{-}RKE}(x_{1:n}\vert t_{1:n})
- \mathrm{Cond\text{-}RKE}(P_{T,X})\bigr| \\
&\le\;
\frac{32\bigl(\frac{M_{\mathcal T}}{m_{\mathcal X,\mathcal T}}\bigr)^{3}\bigl(\frac{2}{m_{\mathcal T}} + \frac{2M_{\mathcal X,\mathcal T}}{m_{\mathcal T}^2}\bigr)}{\sqrt{n}}\Bigl(1+\sqrt{2\log\tfrac{4}{\delta}}\Bigr).
\end{align*}
\end{theorem}
Figure~\ref{fig:convergence_mscoco_dataset} illustrates the convergence of the original and truncated (with $t$=10,000) Conditional-Vendi and Conditional-RKE scores on the MS-COCO (text,image) benchmark using standard CLIP text and DINOv2 image embeddings. The truncated and Conditional-RKE scores stabilize before 15,000 samples, whereas the Conditional-Vendi does not converge within the computationally feasible size $n$=25,000.  

\section{Conditional-Vendi Guidance for Text-Guided Diffusion Models}\label{sec: maintext guidance}\vspace{-2mm}

\citet{askari2024improving} demonstrate that maximizing the Vendi score during sampling can enhance unconditional diversity in (unconditional) diffusion model sample generation. Building on our proposed formulation of Conditional-Vendi, we extend this idea to text-conditioned latent diffusion models \citep{Rombach_2022_CVPR}. Specifically, we employ the \emph{truncated Conditional-Vendi} score as a guidance objective, which encourages each new latent to increase the prompt-aware entropy of the joint kernel spectrum. This approach ensures that generated outputs diversify along dimensions relevant to prompt variability.

Let $\mathrm{Conditional\text{-}Vendi}^{(t)}(z_{1:n}\vert t_{1:n})$ denote the truncated Conditional-Vendi score in the latent space $\mathcal{Z}$. At  step $\tau$, we augment the classifier-free update \citep{Ho2022ClassifierFreeDG} with a diversity ascent step:
\begin{align}\label{eq:trunc-cv-guidance}
\boldsymbol{z}^{(n)}_{\tau-1}
\;\leftarrow&\;
\mathrm{Sampler}\!\bigl(\boldsymbol{z}^{(n)}_\tau,\,
\hat{\epsilon}_\theta(\boldsymbol{z}^{(n)}_\tau,\tau,t_n)\bigr)\\
\; &+\eta_\tau\,\nabla_{\boldsymbol{z}^{(n)}} \mathrm{Conditional\text{-}Vendi}^{(t)}(z_{1:n}\mid t_{1:n}),\nonumber
\end{align}
where $\eta_\tau>0$ is the guidance scale at iteration $\tau$. Equation~\ref{eq:trunc-cv-guidance} is the prompt-aware analogue of unconditional Vendi guidance, but uses the truncated spectrum of the joint kernel for scalability. 
The same construction can also be applied with the Conditional-RKE score as discussed further in the Appendix.

\begin{figure*}[!t]
    \centering
    \includegraphics[width=\linewidth]{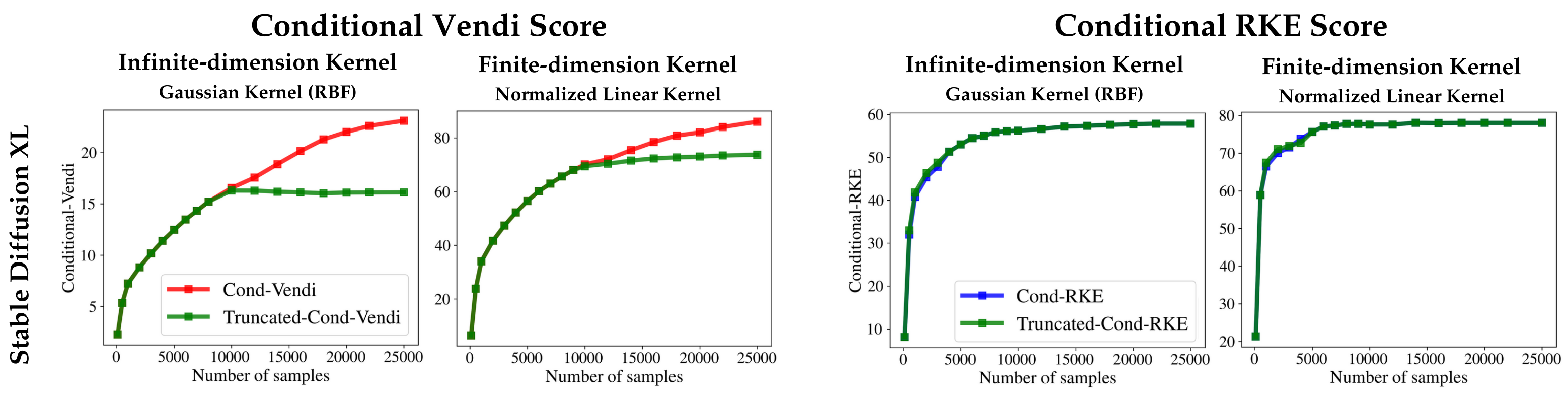}
    \caption{Statistical convergence of Conditional-Vendi scores with different sample sizes on data generated by SDXL using MS-COCO validation set prompts with finite-dimension cosine similarity and infinite-dimension Gaussian kernel. DINOv2 and CLIP embeddings are used for image and text modalities, respectively.}
    \label{fig:convergence_finite_infinite_kernels}
\end{figure*}

\section{Numerical Results}
We numerically evaluated Conditional-Vendi and Conditional-RKE scores for four types of conditional generative models: 1) text-to-image, 2) text-to-video, 3) image-captioning, and 4) Large Language Models. For text-to-image, we tested Flux \citep{flux_2024}, Stable Diffusion 2.1 \citep{Rombach_2022_CVPR}, Stable Diffusion XL \citep{podell2024sdxl}, GigaGAN~\citep{kang2023gigagan}, Kandinsky~\citep{razzhigaev2023kandinsky}, and PixArt~\citep{chen2023pixartalpha, chen2024pixartdelta}. For video, we used VideoCrafter1~\citep{chen2023videocrafter1}, Show-1~\citep{zhang2023show}, and Open-Sora~\citep{opensora}. For image-captioning, we tested BLIP~\citep{li2022blip}, GIT~\citep{wang2022git}, and GPT4o-mini~\citep{openai_gpt-4_2024}. For LLMs, we used Llama \citep{touvron2023llama2} and Gemma \citep{gemmateam2024}.


\textbf{Embeddings used in the evaluation of generative models.} Unlike standard embedding-based scores for text-to-image models such as CLIPScore \citep{hessel2021clipscore}, which require the same embedding model for the text and generated image, our proposed scores allow different feature extractors for text and generated samples. In our experiments, we followed \citep{stein2023exposing, kyn2023},  to use the DINOv2 \citep{oquab2023dinov2} embedding for image data. For text data, we used Gemini \citep{gemini} and CLIP \citep{radford2021learning}, and for video samples, following the video evaluation literature \citep{kim2024stream, Saito2020VIS, unterthiner2019fvd}, we used I3D \citep{i3d}. To select the bandwidth parameter, we followed the previous works \citep{jalali2023information, ospanov2024towards} and we set the truncation parameter $t$ to $10,000$ as suggested in \citep{ospanov2025do}. For the detailed experimental setup, we refer to the Appendix.

\begin{figure*}[t]
    \centering
    \begin{subfigure}{\linewidth}
        \centering
        \includegraphics[width=\linewidth]{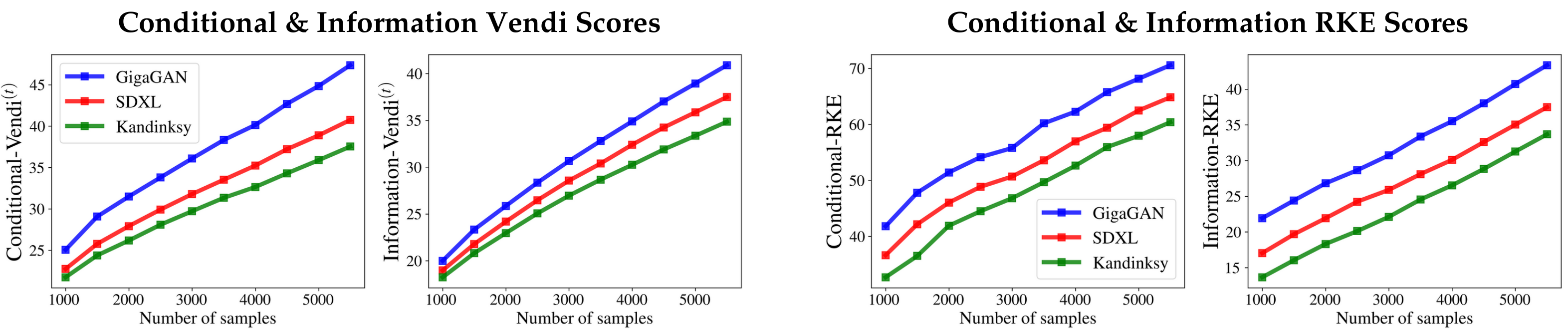}
    \end{subfigure}
    \caption{Conditional and Information Vendi and RKE score comparison across text-to-image models. We clustered MS-COCO prompts into k groups and generated images for each cluster center. Within each cluster, we paired prompts with identical images. The results show increasing diversity and stronger correlation as the number of clusters grows, indicating that clusters become more relevant and diverse with finer partitioning.}
    \label{fig:clustering-text-vendi}
\end{figure*}


\textbf{Convergence Analysis of Conditional-Vendi and Conditional-RKE.}
To assess the convergence of the Conditional-Vendi and Conditional-RKE scores, we conducted experiments for different sample sizes on samples generated with SDXL and Kandinsky using prompts from the MS-COCO 2014 validation set. We used the cosine similarity for the finite-dimensional kernel and the Gaussian kernel for the infinite-dimensional kernel.
Our results, presented in Figure~\ref{fig:convergence_finite_infinite_kernels}, show that for RKE, Conditional-RKE converged, while for Conditional-Vendi, the non-truncated score did not converge; our proposed truncated Conditional-Vendi converged with ~15000 samples. Additional results are provided in the Appendix.


\textbf{Quantifying model-induced diversity via Conditional-Vendi.}
To illustrate how Conditional-Vendi correlates with the model-induced diversity, we considered an experiment were we generated 10 types of animals using Stable Diffusion XL and used two sets of prompts generated using GPT-4o \citep{openai_gpt-4_2024} where in the first set, the types of animals were not specified, while in latter, the animal was explicityly mentioned. As shown in Figure~\ref{fig:animals-sdxl}, increasing the number of animal types led to growth of the Vendi score, regardless of whether the animal type was mentioned in the prompt or not. However, Conditional-Vendi only increased when the types of animal was not specified in the prompts and in the second case, where the types were mentioned in the prompt, Conditional-Vendi slightly increased showing that it only focused on the diversity coming from the background or other parts of images except the animal types. We provided additional experiments with different diffusion models and different concepts in the Appendix~\ref{additional_model_diversity_exp}.

\begin{figure*}[t]
    \centering
    \includegraphics[width=\linewidth]{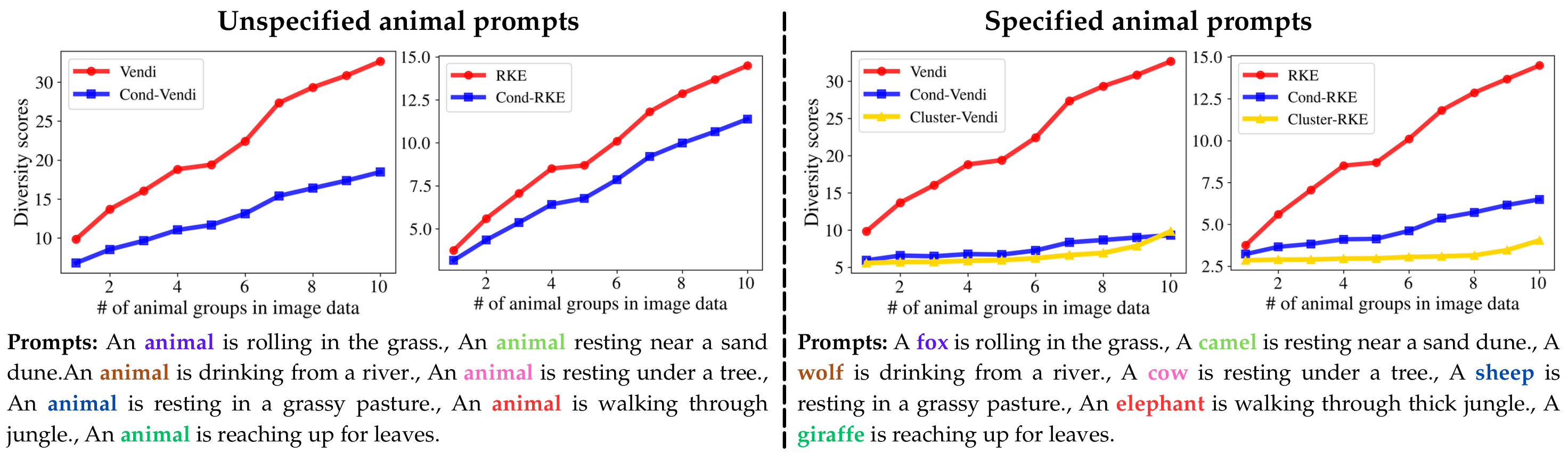}    
    \includegraphics[width=\linewidth]{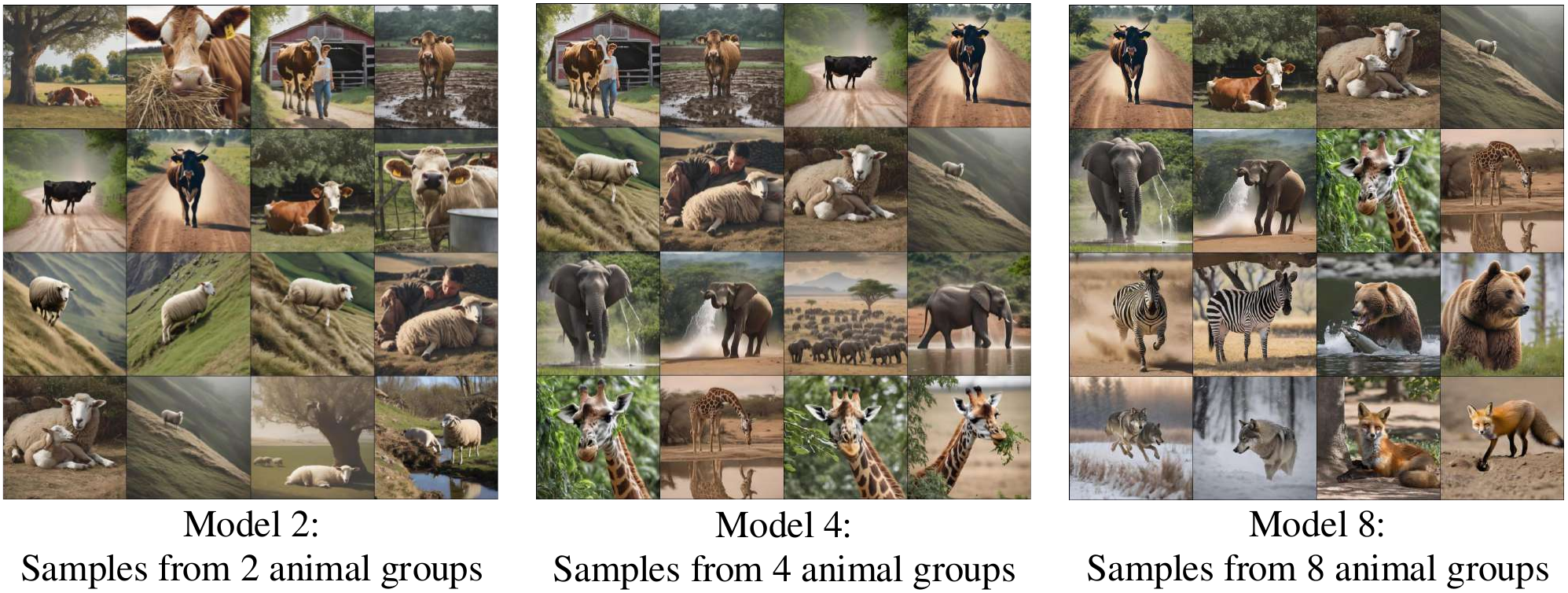}
    \caption{Evaluated Conditional-Vendi and Vendi scores on animal samples generated by SDXL. (Left Plot) We do not specify the animal types in the prompt (Right Plot) we specify the animal types in the prompt.}
    \label{fig:animals-sdxl}
\end{figure*}

\begin{figure*}
    \centering
    \includegraphics[width=\linewidth]{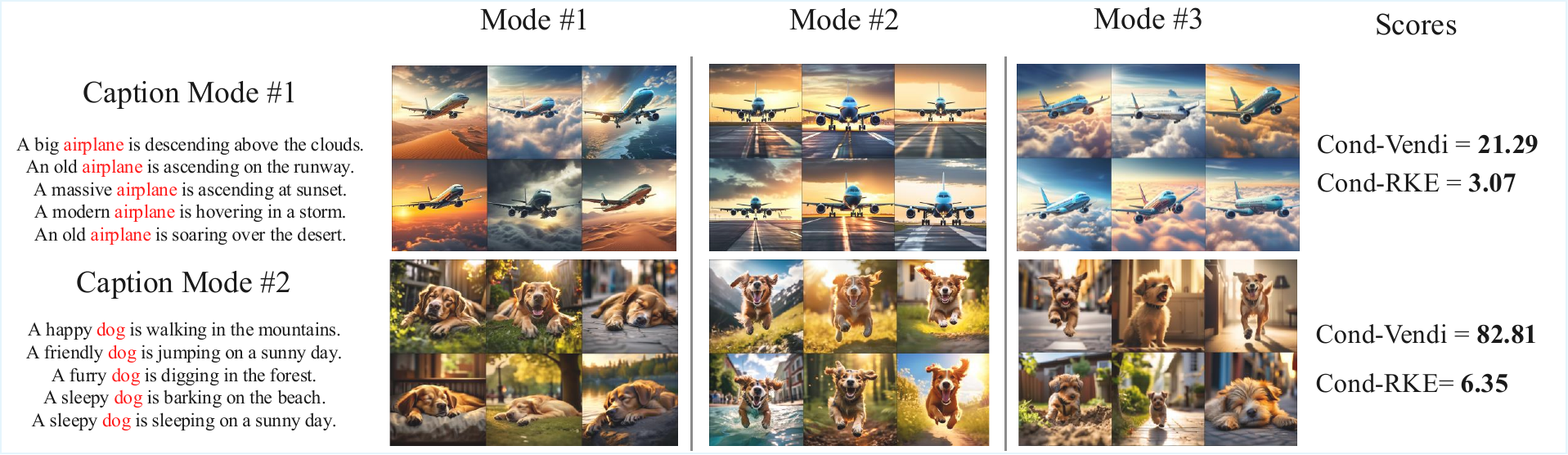}
    \caption{Quantifying image diversity of PixArt-$\alpha$–generated outputs for the top-2 prompt clusters (from 10,000 GPT-4o–generated prompts).}
    \label{fig:inequities_pixart}
\end{figure*}

\textbf{Measuring Conditional-Vendi across prompt types.}
To measure Conditional-Vendi conditioned on the prompt type, we created 10,000 prompts with different categories using GPT-4o and generated the corresponding images with the text-to-image models.
We reported Conditional-Vendi and RKE for the top 3 groups in the text data on PixArt-$\alpha$, Stable Diffusion XL
and FLUX 
text-to-image generative models.
As shown in Figure~\ref{fig:inequities_pixart}, Figure~\ref{fig:inequities_pixart_appendix}, and Figure~\ref{fig:inequities_flux}, we observed the same behavior during these experiments: the Conditional-Vendi score for "dog" prompts was significantly higher than for the "airplane" and "sofa" categories. This observation suggests that the outputs of generative models are unbalanced when presented with different groups of text prompts.

\begin{table*}[t]
\centering
\caption{Quantitative comparison of guidance methods on Stable Diffusion XL}
\resizebox{\linewidth}{!}{%
\begin{tabular}{lcccccc}
\toprule
\textbf{Guidance Method} & \textbf{CLIPScore $\uparrow$} & \textbf{KD$\times 10^2\downarrow$} & \textbf{Cond-Vendi$_\text{DINOv2
}$ $\uparrow$} & \textbf{Vendi$_\text{DINOv2
}$ $\uparrow$} & \textbf{In-batch Sim.$\times 10^2\downarrow$} \\
\midrule
Vendi$_\text{Latent
}$& 30.41 & 35.40 & 29.85 & 309.45 & 80.06 \\
Conditional-Vendi$_\text{Latent
}$& 30.47 & 29.14 & 32.45 & 325.20 & 78.02 \\
\bottomrule
\end{tabular}
} 
\label{tab:diversity_guidance}
\end{table*}

\begin{figure*}[h]
    \centering
    \includegraphics[width=\linewidth]{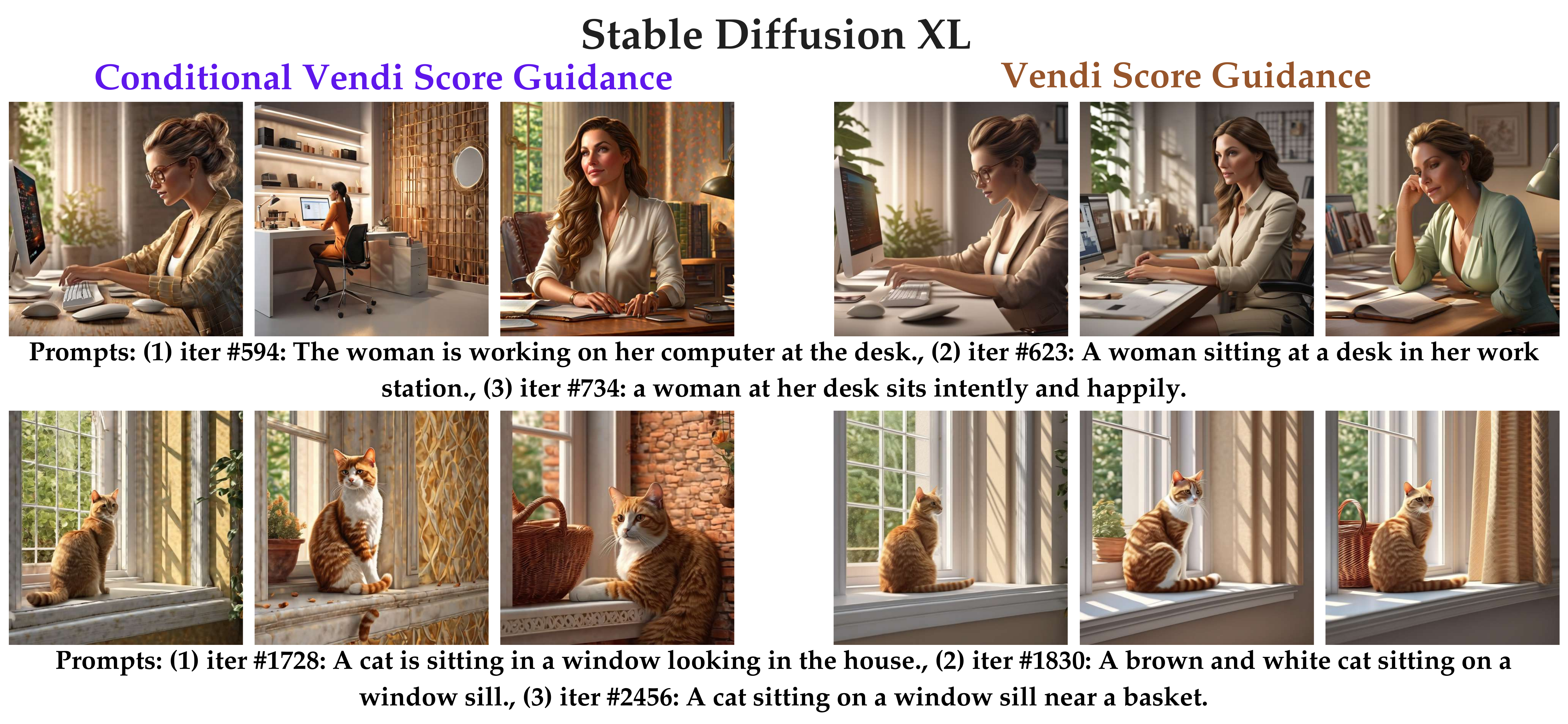}
    \caption{Qualitative comparison of Conditional-Vendi score guidance vs. Vendi score guidance using SD-XL.}
    \label{fig:qualitative_sdxl}
\end{figure*}
\begin{figure*}[h]
    \centering
    \includegraphics[width=0.98\linewidth]{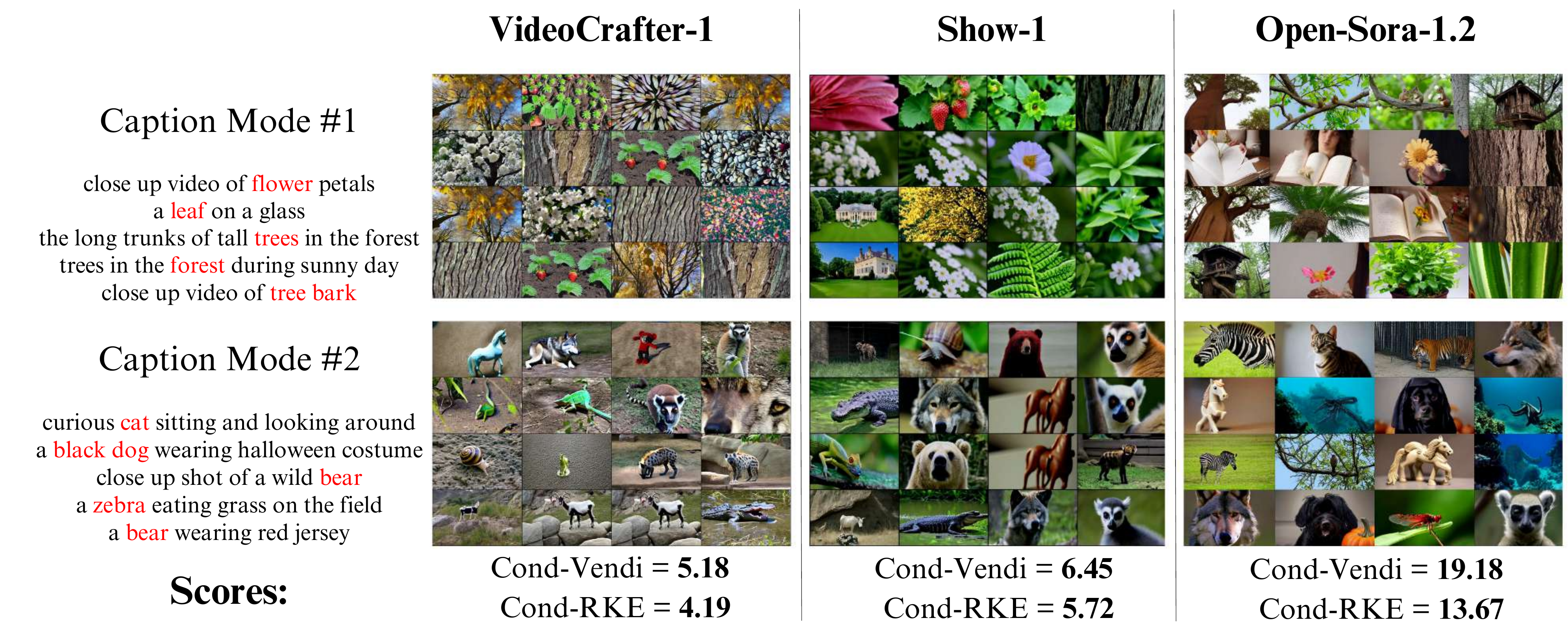}
    \caption{Measuring Conditional-Vendi and Information-Vendi for text-to-video models}
    \label{fig:text-to-video-clustering}
\end{figure*}

\textbf{Text-to-Video Model Evaluation.}
For the experiments on video data, to ensure the fairness of our evaluation, we used VBench samples \citep{huang2023vbench}, which generated samples belong to the 8 content categories. In Figure~\ref{fig:text-to-video-clustering}, we used VideoCrafter-1, Show-1, and Open-Sora-1.2. We observed that VideoCrafter videos look less diverse and, in some cases, may not correlate significantly with the captions when compared to Open-Sora. Confirming this observation, the Conditional-Vendi and Information-Vendi scores were lower for VideoCrafter than those for Open-Sora.

\textbf{Large Language Models Evaluation.}
To evaluate Conditional-Vendi and RKE on LLMs, we varied the temperature parameter and generated 20K short stories with Llama 2 for each temperature setting as shown in Table~\ref{tab:llama_temperature}. We also provided a comparison of the generated prompts in Figure~\ref{fig:llama_samples}. The dataset covered 10 genres, each with 20 distinct subjects and themes. We further tested Conditional-Vendi and Conditional-RKE scores on Gemma 3 \cite{team2025gemma} and Phi 4 Mini \cite{microsoft2025phi4mini}. As shown in Table~\ref{tab:llama_temperature}, both Conditional-Vendi and RKE increase with higher temperatures, indicating that the outputs become more diverse. Additional numerical results on other LLMs are presented in the Appendix.

\begin{table}[h]
\centering
\renewcommand{\arraystretch}{1.15}
\caption{Conditional Vendi and RKE Scores evaluated for Llama~2 with different temperature parameters.}
\resizebox{\linewidth}{!}{%
\begin{tabular}{lcccc}
\toprule
\textbf{Method} & \textbf{$T=0.4$} & \textbf{$T=0.7$} & \textbf{$T=1.0$} & \textbf{$T=1.3$}\\
\midrule
Conditional-Vendi & 45.38 & 46.73 & 48.26 & 49.60 \\
Conditional-RKE & 41.45 & 44.67 & 46.29 & 48.45 \\
\bottomrule
\end{tabular}
} 
\label{tab:llama_temperature}
\end{table}

\textbf{Image-Captioning Evaluation.}
For image captioning, we used 10 classes from the ImageNet dataset as input for BLIP-2, GIT and GPT4o-mini. In Figure~\ref{fig:image-captioning}, we compared captions for the top three groups of images: gas pump, church, and cassette player. GIT generated more diverse captions compared to BLIP, which was confirmed by the Conditional-Vendi scores. On the other hand, GPT4o-mini generated longer and more detailed captions compared to GIT, which was also reflected in the evaluated Conditional-Vendi and Information-Vendi scores.

\begin{figure*}[t]
    \centering
    \includegraphics[width=0.96\linewidth]{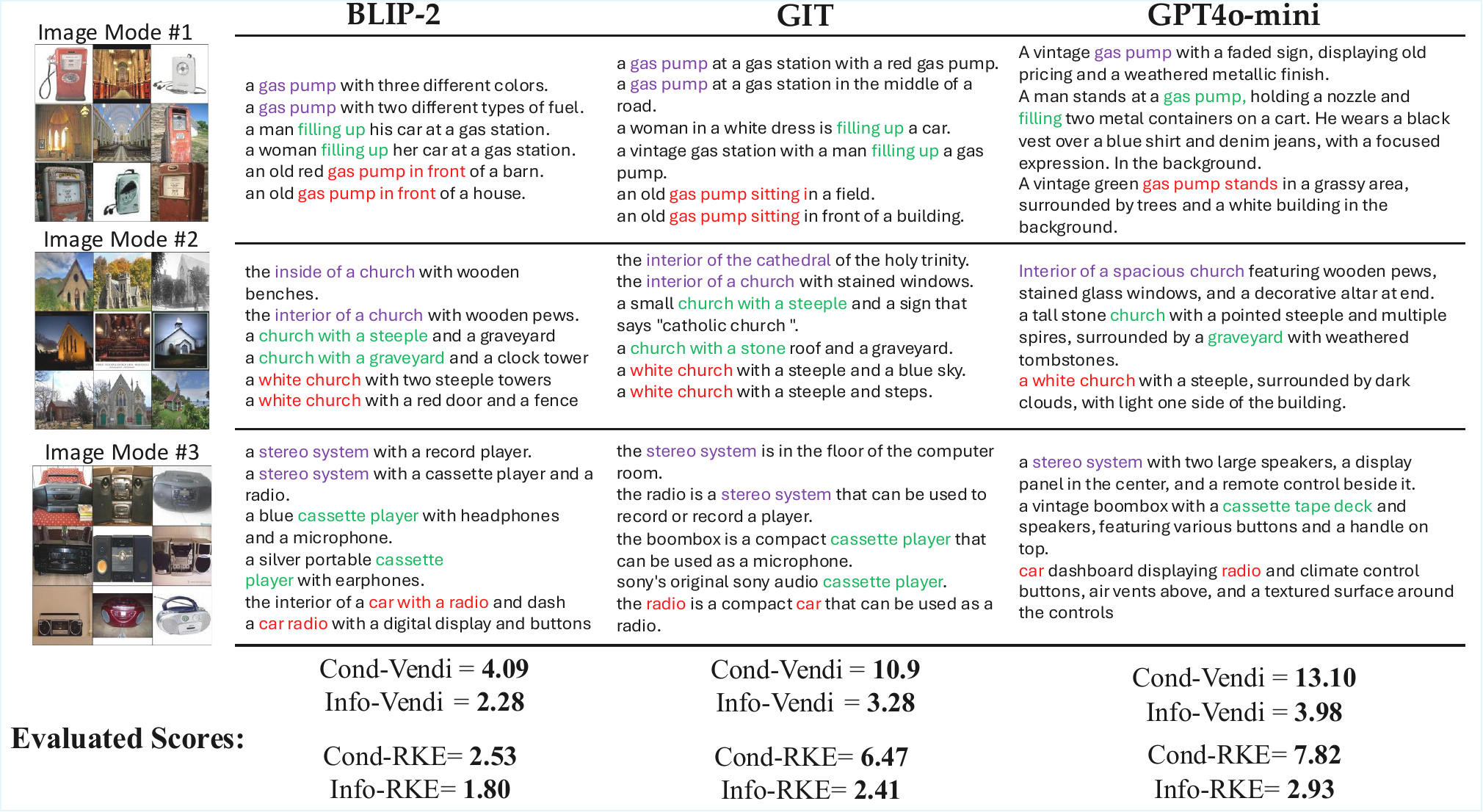}
    \caption{ Conditional-Vendi and Information-Vendi of image-captioning models for 3 image types}
    \label{fig:image-captioning}
\end{figure*}

\textbf{Conditional-Vendi Guidance for Diverse Image Generation.}
To demonstrate the advantages of using prompt-aware metrics over unconditional Vendi score, we investigated how they can improve the sample diversity in latent diffusion models. Specifically, we guided the model with both Truncated-Conditional-Vendi and the standard Vendi score, following the method introduced in \citep{askari2024improving}.

In our experiments, we applied guidance to Stable Diffusion-XL and PixArt-$\Sigma$ in the latent space rather than in the ambient space. We found that applying guidance in the latent space can improve image diversity and quality while also significantly reducing computational costs. The numerical experiments were conducted on 2$\times$NVIDIA GeForce RTX 4090 GPUs, each of which has 22.5 GB of memory.

Figure~\ref{fig:qualitative_sdxl} shows qualitative results using Stable Diffusion-XL, highlighting how prompt-aware guidance can lead to more relevant image generations. We also provide quantitative results of Vendi vs. Conditional-Vendi guidance methods on Stable Diffusion XL in Table~\ref{tab:diversity_guidance}, indicating that Conditional-Vendi guidance improved the sample diversity (Vendi and in-batch similarity) while maintaining text-image alignment in the CLIPScore and KD metrics.
For the detailed experimental setup and additional results, we refer to the Appendix.

\section{Conclusion and Limitations}
We introduced Conditional-Vendi and Conditional-RKE as prompt-aware diversity scores for generative models, extending Vendi and RKE to kernel matrices of prompt–output pairs to separate model-induced diversity from prompt-induced variation. Our theoretical analysis established finite-sample convergence guarantees and a truncated-spectrum approximation for scalability, while experiments on benchmark datasets showed that the proposed scores capture intuitive aspects of diversity that unconditional metrics overlook. A limitation of the approach would be its reliance on high-quality embeddings, which may introduce biases for an arbitrarily selected embedding, as also discussed in \citep{stein2023exposing}. Beyond evaluation, we demonstrated that Conditional-Vendi can also guide generation toward diverse outputs, and we highlight future applications in using these scores as regularizers for training or as fairness measures to assess demographic consistency in generative models.


\clearpage
\clearpage

\section*{Acknowledgments}
The work of Farzan Farnia is partially supported by a grant from the Research Grants Council of the Hong Kong Special Administrative Region, China, Project 14209920, and is partially supported by CUHK Direct Research Grants with CUHK Project No. 4055164 and 4937054.
The work of Amin Gohari is supported by the Hong Kong Research Grants Council (RGC) under grant number 14310025. The work is also supported by a grant under 1+1+1 CUHK-CUHK(SZ)-GDSTC Joint Collaboration Fund. 
Finally, the authors would like to sincerely thank the anonymous reviewers for their insightful feedback and constructive suggestions.

\bibliography{ref}

@inproceedings{
rezaei2024be,
title={Be More Diverse than the Most Diverse: Optimal Mixtures of Generative Models via Mixture-{UCB} Bandit Algorithms},
author={Parham Rezaei and Farzan Farnia and Cheuk Ting Li},
booktitle={The Thirteenth International Conference on Learning Representations},
year={2025},
url={https://openreview.net/forum?id=2Chkk5Ye2s}
}

@inproceedings{hu2025eval,
  title={A Multi-Armed Bandit Approach to Online Selection and
Evaluation of Generative Models},
  author={Hu, Xiaoyan and Leung, Ho-fung and Farnia, Farzan},
booktitle={Proceedings of the 23rd International Conference on Artificial Intelligence and Statistics (AISTATS)},
  year={2025},
}

@article{hu2025promptwise,
  title={PromptWise: Online Learning for Cost-Aware Prompt Assignment in Generative Models},
  author={Hu, Xiaoyan and Pick, Lauren and Leung, Ho-fung and Farnia, Farzan},
  journal={arXiv preprint arXiv:2505.18901},
  year={2025}
}

@inproceedings{hu2025online,
  title={PAK-UCB Contextual Bandit: An Online Learning Approach to Prompt-Aware Selection of Generative Models and LLMs},
  author={Hu, Xiaoyan and Leung, Ho-fung and Farnia, Farzan},
    booktitle={Proceedings of the 26th International Conference on Machine Learning (ICML)},
  year={2025},
}

@inproceedings{wang2023distributedKID,
  title={On the distributed evaluation of generative models},
  author={Wang, Zixiao and Farzan, Farnia and Lin, Zhenghao and Shen, Yunheng and Yu, Bei},
  booktitle={Proceedings of the IEEE/CVF International Conference on Computer Vision Workshops},
  pages={7644--7653},
  year={2025}
}

@InProceedings{Ospanov_2025_ICCV,
    author    = {Ospanov, Azim and Jalali, Mohammad and Farnia, Farzan},
    title     = {Scendi Score: Prompt-Aware Diversity Evaluation via Schur Complement of CLIP Embeddings},
    booktitle = {Proceedings of the IEEE/CVF International Conference on Computer Vision (ICCV)},
    month     = {October},
    year      = {2025},
    pages     = {16927-16937}
}

@inproceedings{zhang2025finc,
  title={Unveiling Differences in Generative Models: A Scalable Differential Clustering Approach},
  author={Zhang, Jingwei and  Jalali, Mohammad and Li, Cheuk Ting and Farnia, Farzan},
  booktitle={{Proceedings of the IEEE/CVF Conference on Computer Vision and Pattern Recognition (CVPR)}},
  year={2025}
}

@inproceedings{
jafari2026diversityaware,
title={Diversity-Aware Online Prompt Assignment to Generative Models},
author={Donya Jafari and Farzan Farnia},
booktitle={The Fourteenth International Conference on Learning Representations},
year={2026},
url={https://openreview.net/forum?id=nnN2TKlS5C}
}

@article{Ho2022ClassifierFreeDG,
  title   = {Classifier-Free Diffusion Guidance},
  author  = {Jonathan Ho and Tim Salimans},
  journal = {arXiv preprint arXiv:2207.12598},
  year    = {2022},
  url     = {https://arxiv.org/abs/2207.12598}
}

@inproceedings{sutherland2018efficient,
  title={Efficient and principled score estimation with nystr{\"o}m kernel exponential families},
  author={Sutherland, Danica J and Strathmann, Heiko and Arbel, Michael and Gretton, Arthur},
  booktitle={International Conference on Artificial Intelligence and Statistics},
  pages={652--660},
  year={2018},
  organization={PMLR}
}

@article{openai2024sora,
  title   = {Sora: A Text-to-Video Model},
  author  = {OpenAI},
  journal = {arXiv preprint arXiv:2402.17177},
  year    = {2024},
}

@article{brown2020language,
  title     = {Language Models are Few-Shot Learners},
  author    = {Brown, Tom B. and Mann, Benjamin and Ryder, Nick and Subbiah, Melanie and Kaplan, Jared and Dhariwal, Prafulla and Neelakantan, Arvind and Shyam, Pranav and Sastry, Girish and Askell, Amanda and others},
  journal   = {Advances in Neural Information Processing Systems},
  volume    = {33},
  pages     = {1877--1901},
  year      = {2020},
}

@inproceedings{Rombach_2022_CVPR,
  author    = {Rombach, Robin and Blattmann, Andreas and Lorenz, Dominik and Esser, Patrick and Ommer, Bj{\"o}rn},
  title     = {High-Resolution Image Synthesis With Latent Diffusion Models},
  booktitle = {Proceedings of the IEEE/CVF Conference on Computer Vision and Pattern Recognition (CVPR)},
  month     = {June},
  year      = {2022},
  pages     = {10684--10695},
}

@article{ramesh2022hierarchical,
  title   = {Hierarchical Text-Conditional Image Generation with CLIP Latents},
  author  = {Ramesh, Aditya and Dhariwal, Prafulla and Nichol, Alex and Chu, Casey and Chen, Mark},
  journal = {arXiv preprint arXiv:2204.06125},
  year    = {2022},
  doi     = {10.48550/arXiv.2204.06125},
}

@article{saharia2022imagen,
  title   = {Photorealistic Text-to-Image Diffusion Models with Deep Language Understanding},
  author  = {Saharia, Chitwan and Chan, William and Saxena, Saurabh and Li, Lala and Whang, Jay and Denton, Emily and Ghasemipour, Seyed Kamyar Seyed and Ayan, Burcu Karagol and Mahdavi, Sara and Lopes, Rapha Gontijo and Salimans, Tim and Ho, Jonathan and Fleet, David J. and Norouzi, Mohammad},
  journal = {arXiv preprint arXiv:2205.11487},
  year    = {2022},
  doi     = {10.48550/arXiv.2205.11487},
}

@inproceedings{ho2022videodiffusion,
  title     = {Video Diffusion Models},
  author    = {Ho, Jonathan and Chan, William and Saharia, Chitwan and Whang, Jay and Gao, Ruiqi and Gritsenko, Alexey and Kingma, Diederik P. and Poole, Ben and Norouzi, Mohammad and Fleet, David J. and Salimans, Tim},
  booktitle = {Advances in Neural Information Processing Systems},
  volume    = {35},
  year      = {2022},
}

@article{ho2022imagenvideo,
  title   = {Imagen Video: High Definition Video Generation with Diffusion Models},
  author  = {Ho, Jonathan and Chan, William and Saharia, Chitwan and Whang, Jay and Gao, Ruiqi and Gritsenko, Alexey and Kingma, Diederik P. and Poole, Ben and Norouzi, Mohammad and Fleet, David J. and Salimans, Tim},
  journal = {arXiv preprint arXiv:2210.02303},
  year    = {2022},
  doi     = {10.48550/arXiv.2210.02303},
}

@inproceedings{goodfellow2014generative,
  title     = {Generative Adversarial Nets},
  author    = {Goodfellow, Ian and Pouget-Abadie, Jean and Mirza, Mehdi and Xu, Bing and Warde-Farley, David and Ozair, Sherjil and Courville, Aaron and Bengio, Yoshua},
  booktitle = {Advances in Neural Information Processing Systems},
  volume    = {27},
  pages     = {2672--2680},
  year      = {2014},
}

@String(CVPR  = {IEEE Conf. Comput. Vis. Pattern Recog.})

@String(ICCV  = {Int. Conf. Comput. Vis.})

@String(NeurIPS = {Adv. Neural Inform. Process. Syst.})

@String(ICML  = {Int. Conf. Mach. Learn.})

@String(AAAI  = {AAAI})

@String(CVPR  = {CVPR})

@String(ICCV  = {ICCV})

@String(NeurIPS = {NeurIPS})

@String(ICML  = {ICML})

@inproceedings{han2023rarity,
title={Rarity Score : A New Metric to Evaluate the Uncommonness of Synthesized Images},
author={Jiyeon Han and Hwanil Choi and Yunjey Choi and Junho Kim and Jung-Woo Ha and Jaesik Choi},
booktitle={The Eleventh International Conference on Learning Representations},
year={2023},
url={https://openreview.net/forum?id=JTGimap_-F}
}

@inproceedings{pasarkar2023cousins,
  title={Cousins Of The Vendi Score: A Family Of Similarity-Based Diversity Metrics For Science And Machine Learning},
  author={Pasarkar, Amey and Dieng, Adji Bousso},
  booktitle={International Conference on Artificial Intelligence and Statistics},
  year={2024},
  organization={PMLR}
}

@article{giraldo2014measures,
  title={Measures of entropy from data using infinitely divisible kernels},
  author={Giraldo, Luis Gonzalo Sanchez and Rao, Murali and Principe, Jose C},
  journal={IEEE Transactions on Information Theory},
  volume={61},
  number={1},
  pages={535--548},
  year={2014},
  publisher={IEEE}
}

@inproceedings{hessel2021clipscore,
  title={{CLIPScore:} A Reference-free Evaluation Metric for Image Captioning},
  author={Hessel, Jack and Holtzman, Ari and Forbes, Maxwell and Bras, Ronan Le and Choi, Yejin},
  booktitle={EMNLP},
  year={2021}
}

@article{kingma2013auto,
  title={Auto-encoding variational bayes},
  author={Kingma, Diederik P and Welling, Max},
  journal={arXiv preprint arXiv:1312.6114},
  year={2013}
}

@article{kynkaanniemi2019improved,
  title={Improved precision and recall metric for assessing generative models},
  author={Kynk{\"a}{\"a}nniemi, Tuomas and Karras, Tero and Laine, Samuli and Lehtinen, Jaakko and Aila, Timo},
  journal={Advances in Neural Information Processing Systems},
  volume={32},
  year={2019}
}

@inproceedings{naeem2020reliable,
  title={Reliable fidelity and diversity metrics for generative models},
  author={Naeem, Muhammad Ferjad and Oh, Seong Joon and Uh, Youngjung and Choi, Yunjey and Yoo, Jaejun},
  booktitle={International Conference on Machine Learning},
  pages={7176--7185},
  year={2020},
  organization={PMLR}
}

@article{sajjadi2018assessing,
  title={Assessing generative models via precision and recall},
  author={Sajjadi, Mehdi SM and Bachem, Olivier and Lucic, Mario and Bousquet, Olivier and Gelly, Sylvain},
  journal={Advances in neural information processing systems},
  volume={31},
  year={2018}
}

@article{binkowski2018demystifying,
  title={Demystifying mmd gans},
  author={Bi{\'n}kowski, Miko{\l}aj and Sutherland, Danica J and Arbel, Michael and Gretton, Arthur},
  journal={arXiv preprint arXiv:1801.01401},
  year={2018}
}

@article{heusel2017gans,
  title={{GANs} trained by a two time-scale update rule converge to a local nash equilibrium},
  author={Heusel, Martin and Ramsauer, Hubert and Unterthiner, Thomas and Nessler, Bernhard and Hochreiter, Sepp},
  journal={Advances in neural information processing systems},
  volume={30},
  year={2017}
}

@article{Borji2022,
title = {Pros and cons of GAN evaluation measures: New developments},
journal = {Computer Vision and Image Understanding},
volume = {215},
pages = {103329},
year = {2022},
issn = {1077-3142},
author = {Ali Borji},
}

@inproceedings{Salimans2016,
 author = {Salimans, Tim and Goodfellow, Ian and Zaremba, Wojciech and Cheung, Vicki and Radford, Alec and Chen, Xi and Chen, Xi},
 booktitle = {Advances in Neural Information Processing Systems},
 editor = {D. Lee and M. Sugiyama and U. Luxburg and I. Guyon and R. Garnett},
 pages = {},
 publisher = {Curran Associates, Inc.},
 title = {Improved Techniques for Training {GAN}s},
 volume = {29},
 year = {2016}
}

@inproceedings{kyn2023,
title={The Role of ImageNet Classes in Fr\'echet Inception Distance},
author={Tuomas Kynk{\"a}{\"a}nniemi and Tero Karras and Miika Aittala and Timo Aila and Jaakko Lehtinen},
booktitle={The Eleventh International Conference on Learning Representations },
year={2023},
url={https://openreview.net/forum?id=4oXTQ6m_ws8}
}

@inproceedings{
jiralerspong2023feature,
title={Feature Likelihood Score: Evaluating the Generalization of Generative Models Using Samples},
author={Marco Jiralerspong and Joey Bose and Ian Gemp and Chongli Qin and Yoram Bachrach and Gauthier Gidel},
booktitle={Thirty-seventh Conference on Neural Information Processing Systems},
year={2023},
url={https://openreview.net/forum?id=l2VKZkolT7}
}

@inproceedings{jalali2023information,
  title={An Information-Theoretic Evaluation of Generative Models in Learning Multi-modal Distributions},
  author={Jalali, Mohammad and Li, Cheuk Ting and Farnia, Farzan},
  booktitle={Thirty-seventh Conference on Neural Information Processing Systems},
  year={2023}
}

@misc{oquab2023dinov2,
  title={DINOv2: Learning Robust Visual Features without Supervision},
  author={Oquab, Maxime and Darcet, Timothée and Moutakanni, Theo and Vo, Huy V. and Szafraniec, Marc and Khalidov, Vasil and Fernandez, Pierre and Haziza, Daniel and Massa, Francisco and El-Nouby, Alaaeldin and Howes, Russell and Huang, Po-Yao and Xu, Hu and Sharma, Vasu and Li, Shang-Wen and Galuba, Wojciech and Rabbat, Mike and Assran, Mido and Ballas, Nicolas and Synnaeve, Gabriel and Misra, Ishan and Jegou, Herve and Mairal, Julien and Labatut, Patrick and Joulin, Armand and Bojanowski, Piotr},
  journal={arXiv:2304.07193},
  year={2023}
}

@misc{rombach2021highresolution,
      title={High-Resolution Image Synthesis with Latent Diffusion Models}, 
      author={Robin Rombach and Andreas Blattmann and Dominik Lorenz and Patrick Esser and Björn Ommer},
      year={2021},
      eprint={2112.10752},
      archivePrefix={arXiv},
      primaryClass={cs.CV}
}

@misc{stein2023exposing,
      title={Exposing flaws of generative model evaluation metrics and their unfair treatment of diffusion models}, 
      author={George Stein and Jesse C. Cresswell and Rasa Hosseinzadeh and Yi Sui and Brendan Leigh Ross and Valentin Villecroze and Zhaoyan Liu and Anthony L. Caterini and J. Eric T. Taylor and Gabriel Loaiza-Ganem},
      year={2023},
      eprint={2306.04675},
      archivePrefix={arXiv},
      primaryClass={cs.LG}
}

@inproceedings{alaa2022faithful,
  title={How faithful is your synthetic data? sample-level metrics for evaluating and auditing generative models},
  author={Alaa, Ahmed and Van Breugel, Boris and Saveliev, Evgeny S and van der Schaar, Mihaela},
  booktitle={International Conference on Machine Learning},
  pages={290--306},
  year={2022},
  organization={PMLR}
}

@inproceedings{kang2023gigagan,
  author    = {Kang, Minguk and Zhu, Jun-Yan and Zhang, Richard and Park, Jaesik and Shechtman, Eli and Paris, Sylvain and Park, Taesung},
  title     = {Scaling up GANs for Text-to-Image Synthesis},
  booktitle = {Proceedings of the IEEE Conference on Computer Vision and Pattern Recognition (CVPR)},
  year      = {2023},
}

@misc{gemini,
      title={Gemini 1.5: Unlocking multimodal understanding across millions of tokens of context}, 
      author={Gemini Team},
      year={2024},
      eprint={2403.05530},
      archivePrefix={arXiv},
      primaryClass={cs.CL},
      url={https://arxiv.org/abs/2403.05530}, 
}

@inproceedings{
kim2024stream,
title={{STREAM}: Spatio-TempoRal Evaluation and  Analysis Metric for Video Generative Models},
author={Pum Jun Kim and Seojun Kim and Jaejun Yoo},
booktitle={The Twelfth International Conference on Learning Representations},
year={2024},
url={https://openreview.net/forum?id=7JfKCZQPxJ}
}

@article{friedman2022vendi,
  title={The vendi score: A diversity evaluation metric for machine learning},
  author={Friedman, Dan and Dieng, Adji Bousso},
  journal={Transactions on machine learning research},
  year={2023}
}

@misc{
unterthiner2019fvd,
title={{FVD}: A new Metric for Video Generation},
author={Thomas Unterthiner and Sjoerd van Steenkiste and Karol Kurach and Rapha{\"e}l Marinier and Marcin Michalski and Sylvain Gelly},
year={2019},
url={https://openreview.net/forum?id=rylgEULtdN}
}

@INPROCEEDINGS{i3d,
  author={Carreira, João and Zisserman, Andrew},
  booktitle={2017 IEEE Conference on Computer Vision and Pattern Recognition (CVPR)}, 
  title={Quo Vadis, Action Recognition? A New Model and the Kinetics Dataset}, 
  year={2017},
  volume={},
  number={},
  pages={4724-4733},
}

@Article{Saito2020VIS,
author={Saito, Masaki
and Saito, Shunta
and Koyama, Masanori
and Kobayashi, Sosuke},
title={Train Sparsely, Generate Densely: Memory-Efficient Unsupervised Training of High-Resolution Temporal GAN},
journal={International Journal of Computer Vision},
year={2020},
day={01},
volume={128},
number={10},
pages={2586-2606},
doi={10.1007/s11263-020-01333-y},
url={https://doi.org/10.1007/s11263-020-01333-y}
}

@misc{radford2021learning,
      title={Learning Transferable Visual Models From Natural Language Supervision}, 
      author={Alec Radford and Jong Wook Kim and Chris Hallacy and Aditya Ramesh and Gabriel Goh and Sandhini Agarwal and Girish Sastry and Amanda Askell and Pamela Mishkin and Jack Clark and Gretchen Krueger and Ilya Sutskever},
      year={2021},
      eprint={2103.00020},
      archivePrefix={arXiv},
      primaryClass={cs.CV}
}

@InProceedings{zhang24_KEN,
  title = 	 {An Interpretable Evaluation of Entropy-based Novelty of Generative Models},
  author =       {Zhang, Jingwei and Li, Cheuk Ting and Farnia, Farzan},
  booktitle = 	 {Proceedings of the 41st International Conference on Machine Learning},
  pages = 	 {59148--59172},
  year = 	 {2024},
  volume = 	 {235},
  publisher =    {PMLR}
}

@InProceedings{huang2023vbench,
     title={{VBench}: Comprehensive Benchmark Suite for Video Generative Models},
     author={Huang, Ziqi and He, Yinan and Yu, Jiashuo and Zhang, Fan and Si, Chenyang and Jiang, Yuming and Zhang, Yuanhan and Wu, Tianxing and Jin, Qingyang and Chanpaisit, Nattapol and Wang, Yaohui and Chen, Xinyuan and Wang, Limin and Lin, Dahua and Qiao, Yu and Liu, Ziwei},
     booktitle={Proceedings of the IEEE/CVF Conference on Computer Vision and Pattern Recognition},
     year={2024}
 }

@misc{kannen2024aesthetics,
      title={Beyond Aesthetics: Cultural Competence in Text-to-Image Models}, 
      author={Nithish Kannen and Arif Ahmad and Marco Andreetto and Vinodkumar Prabhakaran and Utsav Prabhu and Adji Bousso Dieng and Pushpak Bhattacharyya and Shachi Dave},
      year={2024},
      eprint={2407.06863},
      archivePrefix={arXiv},
      primaryClass={cs.CV},
      url={https://arxiv.org/abs/2407.06863}, 
}

@inproceedings{kim2022mutual,
title={Mutual Information Divergence: A Unified Metric for Multimodal Generative Models},
author={Jin-Hwa Kim and Yunji Kim and Jiyoung Lee and Kang Min Yoo and Sang-Woo Lee},
booktitle={Advances in Neural Information Processing Systems},
editor={Alice H. Oh and Alekh Agarwal and Danielle Belgrave and Kyunghyun Cho},
year={2022},
url={https://openreview.net/forum?id=wKd2XtSRsjl}
}

@misc{chen2023videocrafter1,
      title={VideoCrafter1: Open Diffusion Models for High-Quality Video Generation}, 
      author={Haoxin Chen and Menghan Xia and Yingqing He and Yong Zhang and Xiaodong Cun and Shaoshu Yang and Jinbo Xing and Yaofang Liu and Qifeng Chen and Xintao Wang and Chao Weng and Ying Shan},
      year={2023},
      eprint={2310.19512},
      archivePrefix={arXiv},
      primaryClass={cs.CV}
}

@article{zhang2023show,
  title={Show-1: Marrying Pixel and Latent Diffusion Models for Text-to-Video Generation},
  author={Zhang, David Junhao and Wu, Jay Zhangjie and Liu, Jia-Wei and Zhao, Rui and Ran, Lingmin and Gu, Yuchao and Gao, Difei and Shou, Mike Zheng},
  journal={arXiv preprint arXiv:2309.15818},
  year={2023}
}

@software{opensora,
  author = {Zangwei Zheng and Xiangyu Peng and Tianji Yang and Chenhui Shen and Shenggui Li and Hongxin Liu and Yukun Zhou and Tianyi Li and Yang You},
  title = {{Open-Sora}: Democratizing Efficient Video Production for All},
  month = {March},
  year = {2024},
  url = {https://github.com/hpcaitech/Open-Sora}
}

@inproceedings{
podell2024sdxl,
title={{SDXL}: Improving Latent Diffusion Models for High-Resolution Image Synthesis},
author={Dustin Podell and Zion English and Kyle Lacey and Andreas Blattmann and Tim Dockhorn and Jonas M{\"u}ller and Joe Penna and Robin Rombach},
booktitle={The Twelfth International Conference on Learning Representations},
year={2024},
url={https://openreview.net/forum?id=di52zR8xgf}
}

@inproceedings{corso2024particleguidance,
title={{Particle Guidance}: non-{I.I.D}. Diverse Sampling with Diffusion Models},
author={Gabriele Corso and Yilun Xu and Valentin De Bortoli and Regina Barzilay and Tommi S. Jaakkola},
booktitle={The Twelfth International Conference on Learning Representations},
year={2024},
url={https://openreview.net/forum?id=KqbCvIFBY7}
}

@misc{razzhigaev2023kandinsky,
      title={Kandinsky: an Improved Text-to-Image Synthesis with Image Prior and Latent Diffusion}, 
      author={Anton Razzhigaev and Arseniy Shakhmatov and Anastasia Maltseva and Vladimir Arkhipkin and Igor Pavlov and Ilya Ryabov and Angelina Kuts and Alexander Panchenko and Andrey Kuznetsov and Denis Dimitrov},
      year={2023},
      eprint={2310.03502},
      archivePrefix={arXiv},
      primaryClass={cs.CV},
      url={https://arxiv.org/abs/2310.03502}, 
}

@misc{chen2023pixartalpha,
      title={PixArt-\(\alpha\): Fast Training of Diffusion Transformer for Photorealistic Text-to-Image Synthesis}, 
      author={Junsong Chen and Jincheng Yu and Chongjian Ge and Lewei Yao and Enze Xie and Yue Wu and Zhongdao Wang and James Kwok and Ping Luo and Huchuan Lu and Zhenguo Li},
      year={2023},
      eprint={2310.00426},
      archivePrefix={arXiv},
      primaryClass={cs.CV}
}

@misc{chen2024pixartdelta,
      title={PIXART-\(\delta\): Fast and Controllable Image Generation with Latent Consistency Models}, 
      author={Junsong Chen and Yue Wu and Simian Luo and Enze Xie and Sayak Paul and Ping Luo and Hang Zhao and Zhenguo Li},
      year={2024},
      eprint={2401.05252},
      archivePrefix={arXiv},
      primaryClass={cs.CV}
}

@misc{astolfi2024consistencydiversity,
      title={Consistency-diversity-realism Pareto fronts of conditional image generative models}, 
      author={Pietro Astolfi and Marlene Careil and Melissa Hall and Oscar Mañas and Matthew Muckley and Jakob Verbeek and Adriana Romero Soriano and Michal Drozdzal},
      year={2024},
      eprint={2406.10429},
      archivePrefix={arXiv},
      primaryClass={cs.CV},
      url={https://arxiv.org/abs/2406.10429}, 
}

@misc{openai_gpt-4_2024,
	title = {{GPT}-4o mini: advancing cost-efficient intelligence},
	author = {OpenAI},
	month = {July},
	year = {2024},
}

@misc{flux_2024,
  author       = {Black Forest Lab},
  title        = {{FLUX}: A Diffusion-based Text-to-Image ({T2I}) Model},
  year         = {2024},
  howpublished = {\url{https://github.com/blackforestlab/flux}},
  note         = {Accessed: 2024-09}
}

@inproceedings{lee2023holistic,
title={Holistic Evaluation of Text-to-Image Models},
author={Tony Lee and Michihiro Yasunaga and Chenlin Meng and Yifan Mai and Joon Sung Park and Agrim Gupta and Yunzhi Zhang and Deepak Narayanan and Hannah Benita Teufel and Marco Bellagente and Minguk Kang and Taesung Park and Jure Leskovec and Jun-Yan Zhu and Li Fei-Fei and Jiajun Wu and Stefano Ermon and Percy Liang},
booktitle={Thirty-seventh Conference on Neural Information Processing Systems Datasets and Benchmarks Track},
year={2023},
url={https://openreview.net/forum?id=qY9LR74O3Z}
}

@inproceedings{li2022blip,
      title={BLIP: Bootstrapping Language-Image Pre-training for Unified Vision-Language Understanding and Generation}, 
      author={Junnan Li and Dongxu Li and Caiming Xiong and Steven Hoi},
      year={2022},
      booktitle={ICML},
}

@misc{wang2022git,
      title={{GIT}: A Generative Image-to-text Transformer for Vision and Language}, 
      author={Jianfeng Wang and Zhengyuan Yang and Xiaowei Hu and Linjie Li and Kevin Lin and Zhe Gan and Zicheng Liu and Ce Liu and Lijuan Wang},
      year={2022},
      eprint={2205.14100},
      archivePrefix={arXiv},
      primaryClass={cs.CV},
      url={https://arxiv.org/abs/2205.14100}, 
}

@misc{pixart-sigma,
      title={PixArt-\(\Sigma\): Weak-to-Strong Training of Diffusion Transformer for 4K Text-to-Image Generation}, 
      author={Junsong Chen and Chongjian Ge and Enze Xie and Yue Wu and Lewei Yao and Xiaozhe Ren and Zhongdao Wang and Ping Luo and Huchuan Lu and Zhenguo Li},
      year={2024},
      eprint={2403.04692},
      archivePrefix={arXiv},
      primaryClass={cs.CV},
      url={https://arxiv.org/abs/2403.04692}
}

@inproceedings{
ospanov2024towards,
title={Towards a Scalable Reference-Free Evaluation of Generative Models},
author={Azim Ospanov and Jingwei Zhang and Mohammad Jalali and Xuenan Cao and Andrej Bogdanov and Farzan Farnia},
booktitle={The Thirty-eighth Annual Conference on Neural Information Processing Systems},
year={2024},
url={https://openreview.net/forum?id=Ex3rPvEct8}
}

@inproceedings{ospanov2025do,
title={Do Vendi Scores Converge with Finite Samples? Truncated Vendi Score for Finite-Sample Convergence Guarantees},
author={Azim Ospanov and Farzan Farnia},
booktitle={The 41st Conference on Uncertainty in Artificial Intelligence},
year={2025},
url={https://openreview.net/forum?id=Vb5sG3ZQjE}
}

@article{askari2024improving,
  title={Improving Geo-diversity of Generated Images with Contextualized Vendi Score Guidance},
  author={Askari Hemmat, Reyhane and Hall, Melissa and Sun, Alicia and Ross, Candace and Drozdzal, Michal and Romero-Soriano, Adriana},
  journal={arXiv e-prints},
  pages={arXiv--2406},
  year={2024}
}

@inproceedings{Nguyen2024,
author = {Nguyen, Quan and Dieng, Adji Bousso},
title = {Quality-weighted vendi scores and their application to diverse experimental design},
year = {2024},
publisher = {JMLR.org},
booktitle = {Proceedings of the 41st International Conference on Machine Learning},
articleno = {1530},
numpages = {16},
location = {Vienna, Austria},
series = {ICML'24}
}

@article{touvron2023llama2,
  title={Llama 2: Open Foundation and Fine-Tuned Chat Models},
  author={Touvron, Hugo and Martin, Louis and Stone, Kevin R. and Albert, Peter and Almahairi, Amjad and Babaei, Yasmine and Bashlykov, Nikolay and Batra, Soumya and Bhargava, Prajjwal and Bhosale, Shruti and others},
  journal={arXiv preprint arXiv:2307.09288},
  year={2023},
  url={https://arxiv.org/abs/2307.09288}
}

@misc{gemmateam2024,
  title={Gemma: Open Models Based on Gemini Research and Technology},
  author={Gemma Team and Thomas Mesnard and Cassidy Hardin and Robert Dadashi and Surya Bhupatiraju and Shreya Pathak and Laurent Sifre and Morgane Rivière and Mihir Sanjay Kale and Juliette Love and Pouya Tafti and Léonard Hussenot and Pier Giuseppe Sessa and Aakanksha Chowdhery and Adam Roberts and Aditya Barua and Alex Botev and Alex Castro-Ros and Ambrose Slone and Amélie Héliou and Andrea Tacchetti and Anna Bulanova and Antonia Paterson and Beth Tsai and Bobak Shahriari and Charline Le Lan and Christopher A. and others},
  year={2024},
  eprint={2403.08295},
  archivePrefix={arXiv},
  primaryClass={cs.CL},
  url={https://arxiv.org/abs/2403.08295}
}

@inproceedings{kim2022diffusionclip,
  title={Diffusionclip: Text-guided diffusion models for robust image manipulation},
  author={Kim, Gwanghyun and Kwon, Taesung and Ye, Jong Chul},
  booktitle={Proceedings of the IEEE/CVF conference on computer vision and pattern recognition},
  pages={2426--2435},
  year={2022}
}

@InProceedings{nichol2021glide,
  title = 	 {{GLIDE}: Towards Photorealistic Image Generation and Editing with Text-Guided Diffusion Models},
  author =       {Nichol, Alexander Quinn and Dhariwal, Prafulla and Ramesh, Aditya and Shyam, Pranav and Mishkin, Pamela and Mcgrew, Bob and Sutskever, Ilya and Chen, Mark},
  booktitle = 	 {Proceedings of the 39th International Conference on Machine Learning},
  pages = 	 {16784--16804},
  year = 	 {2022},
  volume = 	 {162},
  series = 	 {Proceedings of Machine Learning Research},
  publisher =    {PMLR},
  url = 	 {https://proceedings.mlr.press/v162/nichol22a.html},
}

@inproceedings{liu2023more,
  title={More control for free! image synthesis with semantic diffusion guidance},
  author={Liu, Xihui and Park, Dong Huk and Azadi, Samaneh and Zhang, Gong and Chopikyan, Arman and Hu, Yuxiao and Shi, Humphrey and Rohrbach, Anna and Darrell, Trevor},
  booktitle={Proceedings of the IEEE/CVF winter conference on applications of computer vision},
  pages={289--299},
  year={2023}
}

@article{dhariwal2021diffusion,
  title={Diffusion models beat gans on image synthesis},
  author={Dhariwal, Prafulla and Nichol, Alexander},
  journal={Advances in neural information processing systems},
  volume={34},
  pages={8780--8794},
  year={2021}
}

@inproceedings{mou2024t2i,
  title={T2i-adapter: Learning adapters to dig out more controllable ability for text-to-image diffusion models},
  author={Mou, Chong and Wang, Xintao and Xie, Liangbin and Wu, Yanze and Zhang, Jian and Qi, Zhongang and Shan, Ying},
  booktitle={Proceedings of the AAAI conference on artificial intelligence},
  volume={38},
  number={5},
  pages={4296--4304},
  year={2024}
}

@inproceedings{zhang2023adding,
  title={Adding conditional control to text-to-image diffusion models},
  author={Zhang, Lvmin and Rao, Anyi and Agrawala, Maneesh},
  booktitle={Proceedings of the IEEE/CVF international conference on computer vision},
  pages={3836--3847},
  year={2023}
}

@inproceedings{
    tevet2022human,
    title={Human Motion Diffusion Model},
    author={Guy Tevet and Sigal Raab and Brian Gordon and Yoni Shafir and Daniel Cohen-or and Amit Haim Bermano},
    booktitle={The Eleventh International Conference on Learning Representations },
    year={2023},
    url={https://openreview.net/forum?id=SJ1kSyO2jwu}
}

@article{zhao2022egsde,
  title={Egsde: Unpaired image-to-image translation via energy-guided stochastic differential equations},
  author={Zhao, Min and Bao, Fan and Li, Chongxuan and Zhu, Jun},
  journal={Advances in Neural Information Processing Systems},
  volume={35},
  pages={3609--3623},
  year={2022}
}

@inproceedings{ruiz2023dreambooth,
  title={Dreambooth: Fine tuning text-to-image diffusion models for subject-driven generation},
  author={Ruiz, Nataniel and Li, Yuanzhen and Jampani, Varun and Pritch, Yael and Rubinstein, Michael and Aberman, Kfir},
  booktitle={Proceedings of the IEEE/CVF conference on computer vision and pattern recognition},
  pages={22500--22510},
  year={2023}
}

@inproceedings{
  he2023manifold,
  title={Manifold Preserving Guided Diffusion},
  author={Yutong He and Naoki Murata and Chieh-Hsin Lai and Yuhta Takida and Toshimitsu Uesaka and Dongjun Kim and Wei-Hsiang Liao and Yuki Mitsufuji and J Zico Kolter and Ruslan Salakhutdinov and Stefano Ermon},
  booktitle={The Twelfth International Conference on Learning Representations},
  year={2024},
  url={https://openreview.net/forum?id=o3BxOLoxm1}
}

@inproceedings{bansal2023universal,
  title={Universal guidance for diffusion models},
  author={Bansal, Arpit and Chu, Hong-Min and Schwarzschild, Avi and Sengupta, Soumyadip and Goldblum, Micah and Geiping, Jonas and Goldstein, Tom},
  booktitle={Proceedings of the IEEE/CVF Conference on Computer Vision and Pattern Recognition},
  pages={843--852},
  year={2023}
}

@inproceedings{yu2023freedom,
  title={Freedom: Training-free energy-guided conditional diffusion model},
  author={Yu, Jiwen and Wang, Yinhuai and Zhao, Chen and Ghanem, Bernard and Zhang, Jian},
  booktitle={Proceedings of the IEEE/CVF International Conference on Computer Vision},
  pages={23174--23184},
  year={2023}
}

@article{ye2024tfg,
  title={Tfg: Unified training-free guidance for diffusion models},
  author={Ye, Haotian and Lin, Haowei and Han, Jiaqi and Xu, Minkai and Liu, Sheng and Liang, Yitao and Ma, Jianzhu and Zou, James Y and Ermon, Stefano},
  journal={Advances in Neural Information Processing Systems},
  volume={37},
  pages={22370--22417},
  year={2024}
}

@inproceedings{sadat2025no,
    title={No Training, No Problem: Rethinking Classifier-Free Guidance for Diffusion Models},
    author={Seyedmorteza Sadat and Manuel Kansy and Otmar Hilliges and Romann M. Weber},
    booktitle={The Thirteenth International Conference on Learning Representations},
    year={2025},
    url={https://openreview.net/forum?id=b3CzCCCILJ}
}

@article{ho2022classifier,
  title={Classifier-free diffusion guidance},
  author={Ho, Jonathan and Salimans, Tim},
  journal={arXiv preprint arXiv:2207.12598},
  year={2022}
}

@INPROCEEDINGS{Sehwag_low_density,
  author={Sehwag, Vikash and Hazirbas, Caner and Gordo, Albert and Ozgenel, Firat and Ferrer, Cristian Canton},
  booktitle={2022 IEEE/CVF Conference on Computer Vision and Pattern Recognition (CVPR)}, 
  title={Generating High Fidelity Data from Low-density Regions using Diffusion Models}, 
  year={2022},
  pages={11482-11491},
  doi={10.1109/CVPR52688.2022.01120}
}

@inproceedings{dollar_street,
 author = {Gaviria Rojas, William and Diamos, Sudnya and Kini, Keertan and Kanter, David and Janapa Reddi, Vijay and Coleman, Cody},
 booktitle = {Advances in Neural Information Processing Systems},
 pages = {12979--12990},
 publisher = {Curran Associates, Inc.},
 title = {{The Dollar Street Dataset}: Images Representing the Geographic and Socioeconomic Diversity of the World},
 url = {https://proceedings.neurips.cc/paper_files/paper/2022/file/5474d9d43c0519aa176276ff2c1ca528-Paper-Datasets_and_Benchmarks.pdf},
 volume = {35},
 year = {2022}
}

@inproceedings{ramaswamy2022geode,
    author = {Vikram V. Ramaswamy and Sing Yu Lin and Dora Zhao and Aaron B. Adcock and Laurens van der Maaten and Deepti Ghadiyaram and 
              Olga Russakovsky},
    title = {{GeoDE}: a Geographically Diverse Evaluation Dataset for Object Recognition},
    booktitle = {NeurIPS Datasets and Benchmarks},
    year = {2023}
}

@InProceedings{Miao_2024_CVPR,
    author    = {Miao, Zichen and Wang, Jiang and Wang, Ze and Yang, Zhengyuan and Wang, Lijuan and Qiu, Qiang and Liu, Zicheng},
    title     = {Training Diffusion Models Towards Diverse Image Generation with Reinforcement Learning},
    booktitle = {Proceedings of the IEEE/CVF Conference on Computer Vision and Pattern Recognition (CVPR)},
    month     = {June},
    year      = {2024},
    pages     = {10844-10853}
}

@misc{microsoft2025phi4mini,
    title={Phi-4-Mini Technical Report: Compact yet Powerful Multimodal Language Models via Mixture-of-LoRAs},
    author={Microsoft and : and Abdelrahman Abouelenin and Atabak Ashfaq and Adam Atkinson and Hany Awadalla and Nguyen Bach and Jianmin Bao and Alon Benhaim and Others},
    year={2025},
    eprint={2503.01743},
    archivePrefix={arXiv},
    primaryClass={cs.CL}
}

@misc{team2025gemma,
    title={Gemma 3 Technical Report},
    author={Gemma Team and Aishwarya Kamath and Johan Ferret and Shreya Pathak and Nino Vieillard and Ramona Merhej and Sarah Perrin and Others},
    year={2025},
    eprint={2503.19786},
    archivePrefix={arXiv},
    primaryClass={cs.CL}
}
\bibliographystyle{apalike}





\section*{Checklist}

\begin{enumerate}

  \item For all models and algorithms presented, check if you include:
  \begin{enumerate}
    \item A clear description of the mathematical setting, assumptions, algorithm, and/or model. [Yes]
    \item An analysis of the properties and complexity (time, space, sample size) of any algorithm. [Yes]
    \item (Optional) Anonymized source code, with specification of all dependencies, including external libraries. [Yes]
  \end{enumerate}

  \item For any theoretical claim, check if you include:
  \begin{enumerate}
    \item Statements of the full set of assumptions of all theoretical results. [Yes]
    \item Complete proofs of all theoretical results. [Yes]
    \item Clear explanations of any assumptions. [Yes]     
  \end{enumerate}

  \item For all figures and tables that present empirical results, check if you include:
  \begin{enumerate}
    \item The code, data, and instructions needed to reproduce the main experimental results (either in the supplemental material or as a URL). [Yes]
    \item All the training details (e.g., data splits, hyperparameters, how they were chosen). [Yes]
    \item A clear definition of the specific measure or statistics and error bars (e.g., with respect to the random seed after running experiments multiple times). [Yes]
    \item A description of the computing infrastructure used. (e.g., type of GPUs, internal cluster, or cloud provider). [Yes]
  \end{enumerate}

  \item If you are using existing assets (e.g., code, data, models) or curating/releasing new assets, check if you include:
  \begin{enumerate}
    \item Citations of the creator If your work uses existing assets. [Not Applicable]
    \item The license information of the assets, if applicable. [Yes]
    \item New assets either in the supplemental material or as a URL, if applicable. [Not Applicable]
    \item Information about consent from data providers/curators. [Not Applicable]
    \item Discussion of sensible content if applicable, e.g., personally identifiable information or offensive content. [Not Applicable]
  \end{enumerate}

  \item If you used crowdsourcing or conducted research with human subjects, check if you include:
  \begin{enumerate}
    \item The full text of instructions given to participants and screenshots. [Not Applicable]
    \item Descriptions of potential participant risks, with links to Institutional Review Board (IRB) approvals if applicable. [Not Applicable]
    \item The estimated hourly wage paid to participants and the total amount spent on participant compensation. [Not Applicable]
  \end{enumerate}

\end{enumerate}

\clearpage

\onecolumn
\aistatstitle{Supplementary Materials: Conditional Vendi Score: Prompt-Aware Diversity Evaluation for Text-Guided Generative AI Models}
\appendix

\section{Proofs of Theorems in Section~\ref{sec: concentration}}
\subsection{Formal Statement of Assumptions and Theorems}

\begin{assumption}[Normalized kernels]\label{ass:kernels}
$k_{\mathcal X}$ and $k_{\mathcal T}$ are normalized  kernel functions satisfying
$k_{\mathcal X}(x,x)=k_{\mathcal T}(t,t)=1$ for all $x\in\mathcal X$ and $t\in\mathcal T$.
\end{assumption}

\begin{assumption}[Population kernel covariance norm bounds]\label{ass:pop-bounds}
There exist positive constants $m_{\mathcal{X},\mathcal{T}},M_{\mathcal{X},\mathcal{T}},m_{\mathcal{T}},M_{\mathcal{T}}$ such that, for the Hilbert-Schmidt norm of population kernel covaraince matrices $C_{\mathcal T} = \mathbb{E}_{t\sim P_T}[\phi_\mathcal{T}(t)\phi_\mathcal{T}(t)^\top]$ and $C_{\mathcal X,\mathcal T} = \mathbb{E}_{t\sim P_T,x\sim P_{X|Tt}}[(\phi_\mathcal{T}(t)\otimes \phi_\mathcal{X}(x))(\phi_\mathcal{T}(t)\otimes \phi_\mathcal{X}(x))^\top]$, 
\[
0<m_{\mathcal X,\mathcal T}\le \|C_{\mathcal X,\mathcal T}\|_{\mathrm{HS}}\le M_{\mathcal X,\mathcal T},
\qquad
0<m_{\mathcal T}\le \|C_{\mathcal T}\|_{\mathrm{HS}}\le M_{\mathcal T}.
\]
Note that the above assumption is identical to the following inequalities for the vector of the eigenvalues of the kernel covariance matrices,
denoted by $\widetilde{\boldsymbol\lambda}_{\mathcal T}$ and $\widetilde{\boldsymbol\lambda}_{\mathcal X,\mathcal T}$,
\[
0<m_{\mathcal X,\mathcal T}\le \|\widetilde{\boldsymbol\lambda}_{\mathcal X,\mathcal T}\|_2\le M_{\mathcal X,\mathcal T},
\qquad
0<m_{\mathcal T}\le \|\widetilde{\boldsymbol\lambda}_{\mathcal T}\|_2\le M_{\mathcal T}.
\]
Also, we define and use the condition number $L:=\frac{m_{\mathcal X,\mathcal T}}{M_{\mathcal T}}$ in our analysis.
\end{assumption}

\setcounter{theorem}{0}
\begin{theorem}[Conditional-Vendi convergence]\label{thm:main-vendi}
Suppose Assumptions~\ref{ass:kernels}-\ref{ass:pop-bounds} hold with $d_{\mathcal X}<\infty$ and $d_{\mathcal T}<\infty$. 
For every $\delta\in(0,1)$ such that $
n \ge 4e^2\bigl(1+\sqrt{2\log\tfrac{4}{\delta}}\bigr)^2$,
then the following holds with probability at least $1-\delta$,
\begin{align*}
&\Bigl|\log \mathrm{Conditional\text{-}Vendi}(x_{1:n}\!\vert t_{1:n})
- \log \mathrm{Conditional\text{-}Vendi}(P)\Bigr| \\
\;\le\;
&\frac{1}{\sqrt{n}}\Bigl(1+\sqrt{2\log\tfrac{4}{\delta}}\Bigr)\!\left[
\sqrt{d_{\mathcal X}d_{\mathcal T}} \log\!\bigl({n d_{\mathcal X} d_{\mathcal T}}\bigr)
+
\sqrt{d_{\mathcal T}}\log\!\bigl({n d_{\mathcal T}}\bigr)
\right] \\
\le\: &
\sqrt{\tfrac{20\, d_{\mathcal X}d_{\mathcal T}\log(4/\delta)}{n}}
\log\!\bigl(n d_{\mathcal X}d_{\mathcal T}\bigr)
\end{align*}
The final inequality holds as $d_\mathcal{X} \ge 1$ and for every $0<\delta < 1$, we have $(\sqrt{5}-\sqrt{2}) \log(4/\delta)\ge 1$.
\end{theorem}

\begin{theorem}[Truncated Conditional-Vendi convergence]\label{thm:truncated-vendi}
Suppose Assumptions~\ref{ass:kernels}-\ref{ass:pop-bounds} hold. Fix a truncation level $t\in\mathbb N$ and truncate both the joint and prompt spectra to their top-$t$ components with renormalization. For every $\delta\in(0,1)$ satisfying 
$ n \ge 4e^2\bigl(1+\sqrt{2\log\tfrac{4}{\delta}}\bigr)^2$,
we have the following with probability at least $1-\delta$,
\[
\Bigl|\log \mathrm{Conditional\text{-}Vendi}^{(t)}(x_{1:n}\!\vert t_{1:n})
- \log \mathrm{Conditional\text{-}Vendi}^{(t)}(P)\Bigr|
\le \sqrt{\frac{t}{n}}\Bigl(2+\sqrt{8\log\tfrac{4}{\delta}}\Bigr)\log\!\bigl({nt}\bigr)\le \sqrt{\frac{20t\log(4/\delta)}{n}}\log(nt)
\]
The final inequality holds as for every $0<\delta < 1$, we have $(\sqrt{20}-\sqrt{8}) \log(4/\delta)\ge 2$.
\end{theorem}

\begin{theorem}[Conditional-RKE convergence]\label{thm:main-rke}
Suppose Assumptions~\ref{ass:kernels}-\ref{ass:pop-bounds} hold. Let $C_0 = \frac{2}{m_{\mathcal T}} + \frac{2M_{\mathcal X,\mathcal T}}{m_{\mathcal T}^2}$.
For every $\delta>0$ that satisfies
$n \ge 16\bigl(1+\sqrt{2\log\tfrac{4}{\delta}}\bigr)^2 \max\!\bigl\{\frac{1}{m_{\mathcal T}^2}, \frac{4C_0^2}{L^2}\bigr\}$,
then  the following will hold with probability at least $1-\delta$ 
\[
\Bigl|\,\mathrm{Conditional\text{-}RKE}(x_{1:n}\!\vert t_{1:n})
- \mathrm{Conditional\text{-}RKE}(P)\,\Bigr|
\;\le\;
\frac{32C_0}{L^{\,3}\sqrt{n}}\Bigl(1+\sqrt{2\log\tfrac{4}{\delta}}\Bigr).
\]
\end{theorem}

\subsection{Auxiliary Lemmas and Propositions}

The following are the propositions and lemmas we utilize to prove the theorems.
\begin{proposition}\label{prop:hadamard-feature}
Under Assumption~\ref{ass:kernels}, let $K_{\mathcal X},K_{\mathcal T}\in\mathbb R^{n\times n}$ denote kernel matrices on $\{x_i\}$ and $\{t_i\}$. The normalized Hadamard product $\frac{1}{n}(K_{\mathcal X}\odot K_{\mathcal T})$ and the operator $\widehat{C}_{\mathcal X,\mathcal T}$ share the same multiset of nonzero eigenvalues.
\end{proposition}

\begin{proof}
Let $\phi_{\mathcal X}$ and $\phi_{\mathcal T}$ denote feature maps with $\|\phi_{\mathcal X}(x)\|=\|\phi_{\mathcal T}(t)\|=1$. We define $\phi_{\mathcal X,\mathcal T}([x,t]) := \phi_{\mathcal X}(x)\otimes \phi_{\mathcal T}(t)$. It follows that
\[
\langle\phi_{\mathcal X,\mathcal T}([x,t]), \phi_{\mathcal X,\mathcal T}([x',t'])\rangle
= k_{\mathcal X}(x,x')\,k_{\mathcal T}(t,t').
\]
Let $\Phi_{\mathcal X,\mathcal T}\in\mathbb R^{n\times D}$ be the matrix that stacks the row vectors $\phi_{\mathcal X,\mathcal T}([x_i,t_i])^\top$. We then have $\frac{1}{n}(K_{\mathcal X}\odot K_{\mathcal T}) = \frac{1}{n}\Phi_{\mathcal X,\mathcal T}\Phi_{\mathcal X,\mathcal T}^\top$.
The nonzero eigenvalues of $\frac{1}{n}\Phi_{\mathcal X,\mathcal T}\Phi_{\mathcal X,\mathcal T}^\top$ and
$\frac{1}{n}\Phi_{\mathcal X,\mathcal T}^\top\Phi_{\mathcal X,\mathcal T}=\widehat{C}_{\mathcal X,\mathcal T}$ coincide, which completes the proof.
\end{proof}

\begin{corollary}\label{cor:cond-entropy}
Under Assumption~\ref{ass:kernels}, the conditional von Neumann entropy satisfies
\[
H(X\vert T)=H(\widehat{C}_{\mathcal X,\mathcal T})-H(\widehat{C}_{\mathcal T}),
\]
where $H(\cdot)$ denotes the von Neumann entropy of the nonzero spectrum. For the RKE case, we consider the order-2 Rényi entropy $H_2(M) = \log(1/\|M\|_F^2)$, resulting in $
H_2(X\vert T)=H_2(\widehat{C}_{\mathcal X,\mathcal T})-H_2(\widehat{C}_{\mathcal T})$.
\end{corollary}

\begin{lemma}\label{lem:power-transfer}
Let $\beta<0$ and define $f(u)=u^{\beta}$ on $[u_0,\infty)$ with $u_0>0$. 
If $z,w\ge u_0$ and $|z-w|\le \epsilon$, then
\[
\bigl|f(z)-f(w)\bigr| \le |\beta|\,u_0^{\,\beta-1}\,\epsilon.
\]
\end{lemma}
\begin{proof}
By the mean value theorem there exists $\Gamma$ between $z$ and $w$ such that 
$|f(z)-f(w)|=|f'(\Gamma)|\,|z-w|=|\beta|\,\Gamma^{\beta-1}|z-w|$. 
Since $\beta<0$, we have $\Gamma^{\beta-1}\le u_0^{\beta-1}$, which completes the proof of the lemma.
\end{proof}

\begin{lemma}[Lemma 2 in \citep{ospanov2025do}]\label{lem:logdiff}
If $a,b\in[0,1]$ satisfy $|b-a|\le \frac{1}{e}$, then
$$\Bigl|a\log\bigl(\frac{1}{a}\bigr)-b\log\bigl(\frac{1}{b}\bigr)\Bigr| \:\le\: |b-a|\log\bigl(\frac{1}{|b-a|}\bigr)$$
\end{lemma}

\begin{lemma}[Lemma 3 in \citep{ospanov2025do}]\label{lem:schur}
If $\mathbf u\in\mathbb R_+^{d}$ satisfies $\|\mathbf u\|_2\le \epsilon\le \frac{1}{e}$ for some $\epsilon>0$, then
$$\sum_{i=1}^d u_i\log\bigl(\frac{1}{u_i}\bigr)\: \le\: \epsilon\sqrt{d}\log\bigl(\frac{\sqrt{d}}{\epsilon}\bigr)$$
\end{lemma}

\subsection{Auxiliary Matrix-based Concentration Bounds}

\begin{lemma}\label{lem:HS}
Under Assumption~\ref{ass:kernels}, for every $\delta\in(0,1)$, with probability at least $1-\delta$, the following hold simultaneously
\[
\bigl\|\widehat{C}_{\mathcal X,\mathcal T}-C_{\mathcal X,\mathcal T}\bigr\|_{\mathrm{HS}}
\;\le\; \frac{2}{\sqrt{n}}\Bigl(1+\sqrt{2\log\tfrac{4}{\delta}}\Bigr),
\qquad
\bigl\|\widehat{C}_{\mathcal T}-C_{\mathcal T}\bigr\|_{\mathrm{HS}}
\;\le\; \frac{2}{\sqrt{n}}\Bigl(1+\sqrt{2\log\tfrac{4}{\delta}}\Bigr).
\]
\end{lemma}

\begin{proof}
Given i.i.d. pairs $(t_i,x_i)_{i=1}^n \sim P_T\times P_{X\vert T}$, we define the empirical covariance operator
\[
\widehat{C}_{\mathcal X,\mathcal T} = \frac{1}{n}\sum_{i=1}^n \phi_{\mathcal X,\mathcal T}([x_i,t_i])\phi_{\mathcal X,\mathcal T}([x_i,t_i])^\top
\]
and its population counterpart $C_{\mathcal X,\mathcal T} = \mathbb{E}[\phi_{\mathcal X,\mathcal T}([X,T])\phi_{\mathcal X,\mathcal T}([X,T])^\top]$.
We then define the centered random operator $Z_i^{(\mathcal X,\mathcal T)} := \phi_{\mathcal X,\mathcal T}([x_i,t_i])\phi_{\mathcal X,\mathcal T}([x_i,t_i])^\top - C_{\mathcal X,\mathcal T}$.
We observe that $\mathbb{E}[Z_i^{(\mathcal X,\mathcal T)}]=0$ and, by the normalization property, we have $\|Z_i^{(\mathcal X,\mathcal T)}\|_{\mathrm{HS}}\le 2$.
Applying the Hoeffding inequality for Hilbert-Schmidt operators \citep[Lemma~11]{sutherland2018efficient} with $L=2$ yields the following bound with probability at least $1-\frac{\delta}{2}$: 
\begin{equation*}    \bigl\|\widehat{C}_{\mathcal X,\mathcal T}-C_{\mathcal X,\mathcal T}\bigr\|_{\mathrm{HS}}
\;\le\; \frac{2}{\sqrt{n}}\Bigl(1+\sqrt{2\log\tfrac{4}{\delta}}\Bigr)
\end{equation*}
The same argument applies to the prompt operator $C_{\mathcal{T}}$ and its empirical ${\widehat{C}}_{\mathcal{T}}$ to show with probability at least $1-\frac{\delta}{2}$:
\begin{equation*}    \bigl\|\widehat{C}_{\mathcal T}-C_{\mathcal T}\bigr\|_{\mathrm{HS}}
\;\le\; \frac{2}{\sqrt{n}}\Bigl(1+\sqrt{2\log\tfrac{4}{\delta}}\Bigr)
\end{equation*}
Using a union bound shows that the above inequalities will simultaneously hold with probability at least $1-\delta$.
\end{proof}

\begin{theorem}\label{thm:eig}
Under Assumptions~\ref{ass:kernels}-\ref{ass:pop-bounds}, for any $\delta\in(0,1)$, with probability at least $1-\delta$,
\[
\bigl\|\widehat{\boldsymbol\lambda}_{\mathcal X,\mathcal T}-\widetilde{\boldsymbol\lambda}_{\mathcal X,\mathcal T}\bigr\|_2
\;\le\; \frac{2}{\sqrt{n}}\Bigl(1+\sqrt{2\log\tfrac{4}{\delta}}\Bigr),
\qquad
\bigl\|\widehat{\boldsymbol\lambda}_{\mathcal T}-\widetilde{\boldsymbol\lambda}_{\mathcal T}\bigr\|_2
\;\le\; \frac{2}{\sqrt{n}}\Bigl(1+\sqrt{2\log\tfrac{4}{\delta}}\Bigr).
\]
\end{theorem}

\begin{proof}
The application of Hoffman–Wielandt inequality shows that $\|\widehat{\boldsymbol\lambda}_{\mathcal X , \mathcal T}-\widetilde{\boldsymbol\lambda}_{\mathcal X , \mathcal T}\|_2 \le \|\widehat{C}_{\mathcal X , \mathcal T}-C_{\mathcal X , \mathcal T}\|_{\mathrm{HS}}$and $\|\widehat{\boldsymbol\lambda}_{ \mathcal T}-\widetilde{\boldsymbol\lambda}_{ \mathcal T}\|_2 \le \|\widehat{C}_{ \mathcal T}-C_{ \mathcal T}\|_{\mathrm{HS}}$. The result follows immediately from applying Lemma~\ref{lem:HS}.
\end{proof}

\subsection{Theorem Proofs}

\subsubsection{Proof of Theorem~\ref{thm:main-vendi}}

We define $\varepsilon'_\delta = \frac{2}{\sqrt{n}}(1+\sqrt{2\log\tfrac{4}{\delta}})$. The logarithm of the Conditional-Vendi score is given by
\[
\log \mathrm{Conditional\text{-}Vendi}(x_{1:n}\vert t_{1:n}) = H(\widehat{C}_{\mathcal X,\mathcal T})-H(\widehat{C}_{\mathcal T}),
\]
where $H(\cdot)$ denotes the von Neumann entropy. We note that our sample size condition ensures $\varepsilon'_\delta \le 1/e$.

By Theorem~\ref{thm:eig}, we have $\|\widehat{\boldsymbol\lambda}_{\mathcal X,\mathcal T}-\widetilde{\boldsymbol\lambda}_{\mathcal X,\mathcal T}\|_2\le \varepsilon'_\delta$ and $\|\widehat{\boldsymbol\lambda}_{\mathcal T}-\widetilde{\boldsymbol\lambda}_{\mathcal T}\|_2\le \varepsilon'_\delta$ with probability at least $1-\delta$. We now apply Lemma~\ref{lem:logdiff} coordinatewise to each eigenvalue difference, followed by Lemma~\ref{lem:schur} to bound the sum. This yields
\[
|H(\widehat{C}_{\mathcal X,\mathcal T})-H(C_{\mathcal X,\mathcal T})| \le \varepsilon'_\delta\sqrt{d_{\mathcal X} d_{\mathcal T}}\log\!\Bigl(\tfrac{\sqrt{d_{\mathcal X} d_{\mathcal T}}}{\varepsilon'_\delta}\Bigr),
\]
and similarly,
\[
|H(\widehat{C}_{\mathcal T})-H(C_{\mathcal T})| \le \varepsilon'_\delta\sqrt{d_{\mathcal T}}\log\!\Bigl(\tfrac{\sqrt{d_{\mathcal T}}}{\varepsilon'_\delta}\Bigr).
\]
The triangle inequality then yields the desired bound, noting that by definition $\epsilon'_\delta\ge \frac{2}{\sqrt{n}}$ will automatically hold and therefore $\log(\frac{C}{\epsilon'_\delta})\le \log(\frac{C\sqrt{n}}{2})$ for every $C>0$.

\subsubsection{Proof of Theorem~\ref{thm:truncated-vendi}}

We use the notation $\widehat{\boldsymbol\lambda}_{\mathcal X,\mathcal T}$ and $\widetilde{\boldsymbol\lambda}_{\mathcal X,\mathcal T}$ (resp.\ $\widehat{\boldsymbol\lambda}_{\mathcal T}$ and $\widetilde{\boldsymbol\lambda}_{\mathcal T}$) for the empirical and population eigenvalue vectors of the joint (resp.\ prompt) operator, each sorted in nonincreasing order and summing to $1$. Fix a truncation level $t\in\mathbb N$ and define the $t$-truncated \emph{and renormalized} vectors by keeping the top-$t$ coordinates and then rescaling to unit $\ell_1$ mass.

\medskip
\noindent\textbf{Lemma A.}
Let $\mathbf v\in[0,1]^d$ with $\mathbf 1^\top\mathbf v=1$ and let $S_t=\sum_{i=1}^t v_i$. Define $\mathbf v^{(t)}\in[0,1]^d$ by
\[
v^{(t)}_i=
\begin{cases}
v_i+\frac{1-S_t}{t}, & i\le t,\\
0,& i>t.
\end{cases}
\]
Then $\mathbf v^{(t)}$ is the (Euclidean) projection of $\mathbf v$ onto the convex set
\[
\Delta_t:=\Bigl\{\,\mathbf u\in[0,1]^d:\; u_i=0\ (i>t),\;\sum_{i=1}^t u_i=1\,\Bigr\}.
\]

\emph{Proof of Lemma A.}
Consider the convex program
\[
\min_{\mathbf u\in\mathbb R^t}\ \sum_{i=1}^t (u_i-v_i)^2\quad
\text{s.t.}\quad u_i\ge0\ (i\le t),\ \sum_{i=1}^t u_i=1.
\]
With Lagrangian $L(\mathbf u,\lambda,\boldsymbol\mu)=\sum_{i=1}^t (u_i-v_i)^2+\lambda\bigl(\sum_{i=1}^t u_i-1\bigr)-\sum_{i=1}^t \mu_i u_i$, the KKT conditions are satisfied by $u_i^*=v_i+\frac{1-S_t}{t}$, $\lambda^*=\frac{1-S_t}{t}$, and $\mu_i^*=0$ (primal feasibility: $u_i^*\ge0$ and $\sum_i u_i^*=1$; dual feasibility: $\mu_i^*\ge0$; complementary slackness: $\mu_i^*u_i^*=0$; stationarity: $2(u_i^*-v_i)+\lambda^*-\mu_i^*=0$). Since the problem is convex with affine constraints, KKT optimality is sufficient. Extending to $d$ coordinates by padding zeros yields $\mathbf v^{(t)}\in\Delta_t$ as the Euclidean projection of $\mathbf v$. \hfill$\square$

\medskip
\noindent\textbf{Lemma B.}
For any $\mathbf u,\mathbf v\in[0,1]^d$ with $\mathbf 1^\top\mathbf u=\mathbf 1^\top\mathbf v=1$,
\[
\bigl\|\mathbf u^{(t)}-\mathbf v^{(t)}\bigr\|_2 \;\le\; \|\mathbf u-\mathbf v\|_2,
\qquad\text{where}\ \mathbf u^{(t)},\mathbf v^{(t)}\ \text{are as in Lemma 1.}
\]

\emph{Proof of Lemma B.}
Euclidean projection onto a closed convex set is a nonexpansive map in $\ell_2$, i.e., $\|\Pi_C(a)-\Pi_C(b)\|_2\le\|a-b\|_2$ for all $a,b$ and any closed convex $C$. Applying this with $C=\Delta_t$ and $\Pi_C(\cdot)= (\cdot)^{(t)}$ gives the claim. \hfill$\square$

\medskip
To prove the theorem, we first analyze the eigenvalue concentration through truncation.
By Theorem~\ref{thm:eig}, with probability at least $1-\delta$,
\[
\bigl\|\widehat{\boldsymbol\lambda}_{\mathcal X,\mathcal T}-\widetilde{\boldsymbol\lambda}_{\mathcal X,\mathcal T}\bigr\|_2\le
\varepsilon'_\delta,\qquad
\bigl\|\widehat{\boldsymbol\lambda}_{\mathcal T}-\widetilde{\boldsymbol\lambda}_{\mathcal T}\bigr\|_2\le
\varepsilon'_\delta,
\]
where $\varepsilon'_\delta=\frac{2}{\sqrt{n}}\bigl(1+\sqrt{2\log\tfrac{4}{\delta}}\bigr)$. Applying Lemma~B separately to the joint and prompt spectra,
\begin{equation}\label{eq:trunc_eig_diff}
\bigl\|\widehat{\boldsymbol\lambda}^{(t)}_{\mathcal X,\mathcal T}-\widetilde{\boldsymbol\lambda}^{(t)}_{\mathcal X,\mathcal T}\bigr\|_2\le
\varepsilon'_\delta,\qquad
\bigl\|\widehat{\boldsymbol\lambda}^{(t)}_{\mathcal T}-\widetilde{\boldsymbol\lambda}^{(t)}_{\mathcal T}\bigr\|_2\le
\varepsilon'_\delta.
\end{equation}

\medskip
Then, we analyze entropy perturbation in the truncated size-$t$ dimensions. Assume $n\ge 4e^2\bigl(1+\sqrt{2\log\tfrac{4}{\delta}}\bigr)^2$ so that $\varepsilon'_\delta\le 1/e$. Since the truncated (renormalized) vectors each have at most $t$ nonzero entries and sum to $1$, we can apply the coordinatewise log-difference bound (Lemma~\ref{lem:logdiff}) followed by the $\ell_2$–to–entropy control (Lemma~\ref{lem:schur}) to obtain
\[
\bigl|H(\widehat{\boldsymbol\lambda}^{(t)}_{\mathcal X,\mathcal T})-H(\widetilde{\boldsymbol\lambda}^{(t)}_{\mathcal X,\mathcal T})\bigr|
\le
\bigl\|\widehat{\boldsymbol\lambda}^{(t)}_{\mathcal X,\mathcal T}-\widetilde{\boldsymbol\lambda}^{(t)}_{\mathcal X,\mathcal T}\bigr\|_2\,
\sqrt{t}\,\log\!\Bigl(\frac{\sqrt{t}}{\bigl\|\widehat{\boldsymbol\lambda}^{(t)}_{\mathcal X,\mathcal T}-\widetilde{\boldsymbol\lambda}^{(t)}_{\mathcal X,\mathcal T}\bigr\|_2}\Bigr),
\]
and the analogous bound with $\mathcal T$ in place of $(\mathcal X,\mathcal T)$. Using \eqref{eq:trunc_eig_diff} and the monotonicity of $\log(\cdot)$,
\[
\bigl|H(\widehat{\boldsymbol\lambda}^{(t)}_{\mathcal X,\mathcal T})-H(\widetilde{\boldsymbol\lambda}^{(t)}_{\mathcal X,\mathcal T})\bigr|
\le \varepsilon'_\delta\,\sqrt{t}\,\log\!\Bigl(\frac{\sqrt{t}}{\varepsilon'_\delta}\Bigr),\qquad
\bigl|H(\widehat{\boldsymbol\lambda}^{(t)}_{\mathcal T})-H(\widetilde{\boldsymbol\lambda}^{(t)}_{\mathcal T})\bigr|
\le \varepsilon'_\delta\,\sqrt{t}\,\log\!\Bigl(\frac{\sqrt{t}}{\varepsilon'_\delta}\Bigr).
\]

\medskip
Next, note that
by definition:
\[
\log \mathrm{Conditional\text{-}Vendi}^{(t)}(x_{1:n}\vert t_{1:n})
= H(\widehat{\boldsymbol\lambda}^{(t)}_{\mathcal X,\mathcal T})-H(\widehat{\boldsymbol\lambda}^{(t)}_{\mathcal T}),
\]
and the same for the population quantity. The triangle inequality therefore gives
\[
\Bigl|\log \mathrm{Conditional\text{-}Vendi}^{(t)}(x_{1:n}\vert t_{1:n})
- \log \mathrm{Conditional\text{-}Vendi}^{(t)}(P)\Bigr|
\le
2\,\varepsilon'_\delta\,\sqrt{t}\,\log\!\Bigl(\frac{\sqrt{t}}{\varepsilon'_\delta}\Bigr).
\]
Finally, $\log\!\bigl(\frac{\sqrt{t}}{\varepsilon'_\delta}\bigr)\le \log(nt)$ for $n,t\ge 2$, and substituting $\varepsilon'_\delta=\frac{2}{\sqrt{n}}\bigl(1+\sqrt{2\log\tfrac{4}{\delta}}\bigr)$ yields the stated bound.

\subsubsection{Proof of Theorem~\ref{thm:main-rke}}

We define $\varepsilon_\delta = \frac{2}{\sqrt{n}}(1+\sqrt{2\log\tfrac{4}{\delta}})$. The Conditional-RKE score is defined as 
\[
\mathrm{Conditional\text{-}RKE}(x_{1:n}\vert t_{1:n}) = \Bigl(\frac{\|\widehat{\boldsymbol\lambda}_{\mathcal X,\mathcal T}\|_2}{\|\widehat{\boldsymbol\lambda}_{\mathcal T}\|_2}\Bigr)^{-2}.
\]
We denote by $R$ the ratio $\|\widehat{\boldsymbol\lambda}_{\mathcal X,\mathcal T}\|_2/\|\widehat{\boldsymbol\lambda}_{\mathcal T}\|_2$ and by $R_*$ its population counterpart $\|\widetilde{\boldsymbol\lambda}_{\mathcal X,\mathcal T}\|_2/\|\widetilde{\boldsymbol\lambda}_{\mathcal T}\|_2$. We observe that the assumption on the lower-bound on sample size $n$ can be rewritten to show that $\varepsilon_\delta \le \min\{\tfrac{1}{2}m_{\mathcal T}, \tfrac{L}{2C_0}\}$.

To prove the theorem, we first bound the ratio deviation. To do this, we denote $A=\|\widehat{\boldsymbol\lambda}_{\mathcal X,\mathcal T}\|_2$, $B=\|\widehat{\boldsymbol\lambda}_{\mathcal T}\|_2$, and their population counterparts $A_*=\|\widetilde{\boldsymbol\lambda}_{\mathcal X,\mathcal T}\|_2$, $B_*=\|\widetilde{\boldsymbol\lambda}_{\mathcal T}\|_2$.
From Theorem~\ref{thm:eig}, we have $|A-A_*|\le \varepsilon_\delta$ and $|B-B_*|\le \varepsilon_\delta$.
We can then write
\[
|R-R_*| = \Bigl|\frac{AB_* - A_*B}{BB_*}\Bigr| \le \frac{|A-A_*|}{B} + \frac{A_*|B_*-B|}{BB_*}.
\]
Using the fact that $B\ge m_{\mathcal T}-\varepsilon_\delta\ge \tfrac{1}{2}m_{\mathcal T}$ (which holds under our sample size condition) and $A_*\le M_{\mathcal X,\mathcal T}$ by assumption, we obtain
\[
|R-R_*| \le \frac{2\varepsilon_\delta}{m_{\mathcal T}} + \frac{2M_{\mathcal X,\mathcal T}\varepsilon_\delta}{m_{\mathcal T}^2} = C_0\varepsilon_\delta.
\]

Then, we analyze the deviation when we change the power. We note that $R_*\ge L$ by Assumption~\ref{ass:pop-bounds}. Since we have shown $|R-R_*|\le C_0\varepsilon_\delta$ and our sample size condition ensures $C_0\varepsilon_\delta\le \tfrac{1}{2}L$, we conclude that $R\ge \tfrac{1}{2}L$.
Applying Lemma~\ref{lem:power-transfer} with $\beta=-2$, we obtain
\[
|R^{-2}-R_*^{-2}| \le 2(\tfrac{1}{2}L)^{-3}|R-R_*| = \frac{16}{L^3}C_0\varepsilon_\delta.
\]
This completes the proof.

\section{Theoretical Interpretation of the Conditional-Entropy Score}

\begin{theorem}\label{thm:cond-rke}
Consider the Gaussian kernel with bandwidth $\sigma$. 
Suppose $T$ follows a mixture distribution $\sum_{i=1}^m \omega_i P_{T,i}$ 
with component means $\mu_i$ and within-component variances 
$\sigma_i^2=\mathbb{E}_{T\sim P_{T,i}}[\|T-\mu_i\|_2^2]$. 
Define the error quantity term $\Gamma$ as
\begin{equation}\label{Eq: Gamma Definition}
\Gamma \;=\; 32\sum_{i=1}^m \omega_i \Biggl[ \frac{\sigma_i^2}{\sigma^2} 
\;+\; (i-1)\sum_{j=1}^{i-1} 
     \exp\!\Bigl(-\tfrac{\|\mu_i-\mu_j\|_2^2}{\sigma^2}\Bigr) \Biggr].
\end{equation}
Then, the matrix-based order-$2$ conditional entropy satisfies the following for $g(z)=2\log\bigl(\tfrac{1}{1-z/\|\omega\|_2}\bigr)$
\[
\Bigl|\,\widetilde{H}_2(X\vert T)\:-\:\log\Bigl(1\big/
    \mathbb E_{I\sim\omega^2}\bigl[\exp(-\widetilde{H}_2(X\vert G=I))\bigr]\Bigr)\Bigr|
\;\le\; 2\,g(\Gamma).
\]
In the above, $\omega^2$ represents the probability model $p_i = \omega^2_i/\bigl(\sum_{j=1}^m \omega_j^2\bigr)$, whose probability values are proportional to the square of the probability weight $\omega_i$'s. Also, note that the above is equivalent to what follows in terms of the Conditional-RKE score:
$\mathrm{Conditional\text{-}RKE}(X\vert T)=\exp(\widetilde{H}_2(X\vert T))$,
\begin{equation*}\label{eq:rke-bound}
\exp\!\bigl(-2g(\Gamma)\bigr)\cdot 
\Biggl(\mathbb E_{I\sim\omega^2}\Bigl[\tfrac{1}{\mathrm{RKE}(X\vert G=I)}\Bigr]\Biggr)^{-1}
\le\mathrm{Conditional\text{-}RKE}(X\vert T)
\le\exp\!\bigl(2g(\Gamma)\bigr)\cdot 
\Biggl(\mathbb E_{I\sim\omega^2}\Bigl[\tfrac{1}{\mathrm{RKE}(X\vert G=I)}\Bigr]\Biggr)^{-1}.
\end{equation*}
\end{theorem}

To prove the above theorem, we prove a more general result that applies to every matrix-based 
order-$\alpha$ Rényi conditional entropy of a unit-trace PSD matrix $M\in\mathbb{R}^{n\times n}$, defined as $H_{\alpha}(M)= \frac{1}{1-\alpha}\log\bigl(\sum_{i=1}^n\lambda_i^\alpha\bigr)$, on how it relates to the aggregation of text-instance entropy values for every $\alpha\ge 2$. Note that Theorem~\ref{thm:cond-rke} is the direct corollary of the next theorem with $\alpha=2$.

\begin{theorem}\label{Theorem: 1}
Consider the Gaussian kernel with bandwidth $\sigma$. Suppose $T$ follows a mixture distribution $\sum_{i=1}^m \omega_i P_{T,i}$ where $\omega_i$ denotes the weight of the $i$th component $P_{T,i}$ with mean vector $\mu_i$ and total variance $\mathbb{E}_{T\sim P_{T,i}}[\Vert  T - \mu_i\Vert^2_2]=\sigma_i^2$. Given the aggregation map $f(z)=\exp((1-\alpha)z)$, for every order $\alpha\ge 2$, the matrix-based order-$\alpha$ conditional entropy satisfies the following inequality with $\Gamma$ defined in \eqref{Eq: Gamma Definition} where $g(z) = \frac{\alpha}{\alpha-1}\log\bigl(\frac{1}{1-z/\Vert \boldsymbol{\omega}\Vert_\alpha}\bigr)$ is an increasing scalar function with $g(0)=0$:
\begin{align*}
    \biggl\vert \widetilde{H}_\alpha(X|T)  - f^{-1}\Bigl( \mathbb{E}_{I\sim \boldsymbol{\omega}^\alpha}\Bigl[f\bigl(\widetilde{H}_\alpha(X|G=I)\bigr)\Bigr]\Bigr) \biggr\vert \,\le\,  2 g\bigl(\Gamma\bigr) 
\end{align*}

\end{theorem}

\begin{proof}
To prove Theorem \ref{Theorem: 1}, we begin by showing the following lemma.
\begin{lemma}
Suppose that the kernel function $k$ and variable $T$ satisfy the assumptions in Theorem~\ref{Theorem: 1}. Then, the following Frobenius norm bound holds for $C_i= \mathbb{E}\bigl[\phi_X(x)\phi_X(x)^\top \vert G = i \bigr]$ where $G\in\{1,\ldots , m\}$ is the cluster random variable for text $T$: 
\begin{align*}
    \Bigl\Vert C_{\mathcal X,\mathcal T} - \sum_{i=1}^m \omega_i C_i\otimes \phi_T(\mu_i)\phi_T(\mu_i)^\top  \Bigr\Vert^2_F \le  \frac{\sum_{i=1}^m 2\omega_i\sigma_i^2}{ \sigma^2}.
\end{align*}
\end{lemma}
\begin{proof}
To show this lemma, we define $T_i$ as a variable distributed as $P_{T|G=i}$. Then,
\begin{align*}
    &\Bigl\Vert C_{\mathcal X,\mathcal T} - \sum_{i=1}^m \omega_i C_i \otimes \phi(\mu_i)\phi(\mu_i)^\top \Bigr\Vert^2_F  \\
    = \, &\Bigl\Vert \mathbb{E}\bigl[\phi_X(x)\phi_X(x)^\top \otimes \phi_T(t)\phi_T(t)^\top\bigr] - \sum_{i=1}^m \omega_i C_i \otimes \phi(\mu_i)\phi(\mu_i)^\top \Bigr\Vert^2_F  \\
    = \, &\Bigl\Vert \sum_{i=1}^m \omega_i \mathbb{E}\bigl[\phi_X(x)\phi_X(x)^\top \otimes \phi_T(t)\phi_T(t)^\top \big\vert G=i\bigr] - \sum_{i=1}^m \omega_i C_i \otimes \phi(\mu_i)\phi(\mu_i)^\top \Bigr\Vert^2_F  \\
    = \, &\Bigl\Vert \sum_{i=1}^m \omega_i \mathbb{E}\bigl[\phi_X(x)\phi_X(x)^\top \otimes \phi_T(t)\phi_T(t)^\top \big\vert G=i\bigr] - \sum_{i=1}^m \omega_i \mathbb{E}\bigl[\phi_X(x)\phi_X(x)^\top \otimes \phi_T(\mu_i)\phi_T(\mu_i)^\top \big\vert G=i\bigr]  \Bigr\Vert^2_F  \\
    = \, &\Bigl\Vert \sum_{i=1}^m \omega_i \mathbb{E}\bigl[\phi_X(x)\phi_X(x)^\top \otimes \Bigl(\phi_T(t)\phi_T(t)^\top - \phi_T(\mu_i)\phi_T(\mu_i)^\top \Bigr) \big\vert G=i\bigr]  \Bigr\Vert^2_F  \\
     \stackrel{(a)}{\le} \, & \sum_{i=1}^m \omega_i \mathbb{E}\Bigl[\Bigl\Vert\phi_X(x)\phi_X(x)^\top \otimes \Bigl(\phi_T(t)\phi_T(t)^\top - \phi_T(\mu_i)\phi_T(\mu_i)^\top \Bigr)\Bigr\Vert^2_F  \big\vert G=i\Bigr]   \\
    \stackrel{(b)}{=} \, & \sum_{i=1}^m \omega_i \mathbb{E}\Bigl[\Bigl\Vert\phi_X(x)\phi_X(x)^\top\Bigr\Vert^2_F \Bigl\Vert \phi_T(t)\phi_T(t)^\top - \phi_T(\mu_i)\phi_T(\mu_i)^\top \Bigr\Vert^2_F  \big\vert G=i\Bigr]   \\
    \stackrel{(c)}{=} \, & \sum_{i=1}^m \omega_i \mathbb{E}\Bigl[ \Bigl\Vert \phi_T(t)\phi_T(t)^\top - \phi_T(\mu_i)\phi_T(\mu_i)^\top \Bigr\Vert^2_F  \big\vert G=i\Bigr] \\
    \stackrel{(d)}{=} \, & \sum_{i=1}^m \omega_i \mathbb{E}\Bigl[ 2- 2\exp\Bigl(\frac{-\Vert t - \mu_i\Vert_2^2}{\sigma^2}\Bigr)  \big\vert G=i\Bigr] \\
    \stackrel{(e)}{\le} \, & \sum_{i=1}^m \omega_i \Bigl[ 2- 2\exp\Bigl(\frac{-\mathbb{E}\bigl[\Vert t - \mu_i\Vert_2^2\big\vert G=i\bigr]}{\sigma^2}\Bigr)  \Bigr] \\
    \stackrel{(f)}{\le} \, & \sum_{i=1}^m \omega_i \Bigl[ 2- 2\exp\Bigl(\frac{-\sigma_i^2}{\sigma^2}\Bigr)  \Bigr] \\
    \stackrel{(g)}{\le} \, & \sum_{i=1}^m 2\omega_i \frac{\sigma_i^2}{\sigma^2}
\end{align*}
In the above, (a) follows from Jensen's inequality for the convex Frobenius-norm-squared function. (b) holds because $\Vert A\otimes B \Vert^2_F =\Vert A \Vert^2_F\Vert B \Vert^2_F $ for every matrices $A,\, B$. (c) comes from the normalized Gaussian kernel satisfying $\langle \phi_T(t) , \phi_T(t)  \rangle = k(t,t) = 1$, resulting in $\Vert \phi_T(t) \phi_T(t)^\top \Vert^2_F = \mathrm{Tr}\bigl(\phi_T(t) \phi_T(t)^\top\phi_T(t) \phi_T(t)^\top\bigr) = \mathrm{Tr}\bigl(\phi_T(t)  \phi_T(t)^\top\bigr) =1$. (d) follows from the Gaussian kernel definition, proving that $\phi_T(t)^\top\phi_T(\mu_i) = \exp\bigl( -\Vert t -\mu_i \Vert^2_2/2\sigma^2\bigr)$. (e) shows the application of Jensen's inequality to the concave $s(z) = 1-\exp(-z)$. (f) holds because $s(z) = 1-\exp(-z)$ is a monotonically increasing function. Finally, (g) follows from the inequality $1-\exp(-z) \le z$ for every scalar $z$. Therefore, the proof is complete.

\end{proof}

Next, we apply the Gram–Schmidt process to $\phi_T(\mu_1),\ldots ,\phi_T(\mu_m)$ to find orthogonal vectors $u_1,\ldots , u_m$. We let $u_1=\phi_T(\mu_1)$. Then, for every $2\le i\le m$, we define
\begin{align*}
   v_i\, :=\, &\phi({\mu}_i) - \sum_{j=1}^{i-1} \langle \phi(\mu_i) ,u_j\rangle u_j, \quad u_i = v_i/\Vert v_i \Vert_2
\end{align*}
As a result, the following holds
\begin{align*}
    &\Bigl\Vert \sum_{i=1}^m \omega_i C_i\otimes \phi(\mu_i)\phi(\mu_i)^\top - \sum_{i=1}^m \omega_i C_i\otimes u_iu_i^\top \Bigr\Vert^2_F\\
    = \, &\Bigl\Vert \sum_{i=1}^m \omega_i C_i\otimes \Bigl(\phi(\mu_i)\phi(\mu_i)^\top - u_iu_i^\top \Bigr)\Bigr\Vert^2_F\\
    \stackrel{(h)}{\le} \, & \sum_{i=1}^m \omega_i \Bigl\Vert C_i\otimes \Bigl(\phi(\mu_i)\phi(\mu_i)^\top - u_iu_i^\top \Bigr)\Bigr\Vert^2_F\\
    \stackrel{}{=} \, & \sum_{i=1}^m \omega_i \Bigl\Vert C_i\Bigr\Vert^2_F \Bigl\Vert\phi(\mu_i)\phi(\mu_i)^\top - u_iu_i^\top \Bigr\Vert^2_F\\
    \stackrel{(i)}{\le} \, & \sum_{i=1}^m \omega_i  \Bigl\Vert\phi(\mu_i)\phi(\mu_i)^\top - u_iu_i^\top \Bigr\Vert^2_F\\
    \stackrel{(j)}{=} \,  &\sum_{i=1}^m \omega_i\Bigl( 2-2\bigl(u_i^\top\phi_T(\mu_i)\bigr)^2 \Bigr) \\
    \stackrel{(k)}{=} \,  &2\sum_{i=1}^m \omega_i\sum_{j=1}^{i-1}\bigl(u_j^\top\phi_T(\mu_i)\bigr)^2      
\end{align*}
Here, (h) follows from the application of Jensen's inequality for the convex Frobenius-norm-squared. (i) holds since the text kernel is normalized and $\langle \phi_X(x),\phi_X(x) \rangle = k_X(x,x) = 1$, and therefore $\Vert C_i\Vert_F\le \mathbb{E}[\Vert \phi_X(x) \Vert^2_2] =1$. (j) follows from the expansion $\Vert uu^\top - vv^\top\Vert^2_F = \Vert u\Vert^4_2+ \Vert v\Vert^4_2 - 2\langle u ,v\rangle^2 $. \noindent
Next, we bound the inner products $\langle u_j , \phi_T(\mu_i)\rangle$ that appear in step~(k).
Recall that each $u_j$ is obtained from the Gram--Schmidt process applied to
$\{\phi_T(\mu_\ell)\}_{\ell< i}$, so we can write
\[
u_j \;=\;\sum_{\ell \le j} r_{j\ell}\,\phi_T(\mu_\ell), \qquad 
\|r_j\|_2 \le 1,
\]
where $r_j=(r_{j1},\ldots,r_{jj})^\top$ is the coefficient vector.  
Therefore,
\[
\bigl(u_j^\top \phi_T(\mu_i)\bigr)^2
=\Bigl(\sum_{\ell\le j} r_{j\ell}\,k_T(\mu_\ell,\mu_i)\Bigr)^2
\;\le\;\sum_{\ell\le j} k_T(\mu_\ell,\mu_i)^2,
\]
where we used $\|r_j\|_2\le 1$ and Cauchy--Schwarz.  
Summing over $j<i$ yields
\[
\sum_{j<i} \bigl(u_j^\top \phi_T(\mu_i)\bigr)^2
\;\le\;\sum_{j<i}\sum_{\ell\le j} k_T(\mu_\ell,\mu_i)^2
\;=\;\sum_{\ell<i} (i-\ell)\,k_T(\mu_\ell,\mu_i)^2
\;\le\;(i-1)\sum_{\ell<i} k_T(\mu_\ell,\mu_i)^2.
\]
For the Gaussian kernel $k_T(t,t')=\exp(-\|t-t'\|^2/(2\sigma^2))$, we have
$k_T(\mu_\ell,\mu_i)^2 = \exp(-\|\mu_\ell-\mu_i\|^2/\sigma^2)$.  
Hence, inequality~(k) becomes
\begin{align}
\sum_{i=1}^m \omega_i \Bigl\Vert \phi_T(\mu_i)\phi_T(\mu_i)^\top - u_i u_i^\top \Bigr\Vert_F^2
&\le 2\sum_{i=1}^m \omega_i (i-1)\sum_{\ell<i} \exp\!\Bigl(-\tfrac{\|\mu_i-\mu_\ell\|^2}{\sigma^2}\Bigr).
\label{eq:ortho-bound}
\end{align}

\noindent
Combining~\eqref{eq:ortho-bound} with the variance bound from Lemma~1, and applying 
$\|A+B\|_F^2\le 2\|A\|_F^2+2\|B\|_F^2$, we obtain
\begin{align}
\Bigl\| C_{\mathcal X,\mathcal T} - \sum_{i=1}^m \omega_i C_i \otimes u_i u_i^\top \Bigr\|_F^2
&\le 4\sum_{i=1}^m \omega_i \frac{\sigma_i^2}{\sigma^2}
   + 8\sum_{i=1}^m \omega_i (i-1)\sum_{\ell<i} 
       \exp\!\Bigl(-\tfrac{\|\mu_i-\mu_\ell\|^2}{\sigma^2}\Bigr).
\label{eq:frobenius-final}
\end{align}

Since $u_1,\ldots , u_m$ are orthogonal vectors, the definition of Kronecker product implies that the eigenvalues of $\sum_{i=1}^m \omega_i C_i\otimes u_iu_i^\top $ will be the union of the eigenvalues of $\omega_i C_i\otimes u_iu_i^\top$ over $i\in\{1,\ldots ,m\}$. On the other hand, we know that the non-zero eigenvalues of $\omega_i C_i\otimes u_iu_i^\top$ will be equal to the factor $\omega_i \Vert u_i\Vert^2_2 =\omega_i$ times the eigenvalues of $C_i$. Consequently, we can show that for vector $\widehat{\lambda}_{x\otimes t} = \mathrm{Union}\bigl(\omega_i\,\mathrm{Eigs}(C_i):\, i\in\{1,\ldots ,m\}\bigr)$, we have the following for every $\alpha\ge 2$ and defined increasing function $g$ in Theorem~\ref{Theorem: 1}
\begin{align*}
    \Bigl\vert \widetilde{H}_\alpha(X,T) - \frac{1}{1-\alpha}\log\bigl(\Vert \widehat{\lambda}_{x\otimes t}\Vert^\alpha_\alpha \bigr) \Bigr\vert \, \le\, & g\bigl( \Vert \widetilde{\lambda}_{x\otimes t}\Vert_\alpha - \Vert \widehat{\lambda}_{x\otimes t}\Vert_\alpha    \bigr) \\
    \le\, & g\bigl( \Vert \mathrm{sort}\bigl(\widetilde{\lambda}_{x\otimes t}\bigr) -  \mathrm{sort}\bigl(\widehat{\lambda}_{x\otimes t} \bigr)\Vert_\alpha   \bigr) \\
    \le\, & g\bigl( \Vert \mathrm{sort}\bigl(\widetilde{\lambda}_{x\otimes t}\bigr) -  \mathrm{sort}\bigl(\widehat{\lambda}_{x\otimes t} \bigr) \Vert_2  \bigr) \\
    \le\, & g\Bigl( \sum_{i=1}^m 4\omega_i \frac{\sigma_i^2}{ \sigma^2} + \sum_{i=2}^m\sum_{j=1}^{i-1} 8\omega_i(i-1)\exp\Bigl(\frac{-\Vert\mu_i - \mu_j \Vert^2_2}{\sigma^2} \Bigr) \Bigr).
\end{align*}

Note that the above proof holds for every marginal distribution on $X$, and we choose a deterministic constant $X=\mathbf{0}$, then the joint entropy reduces to the marginal entropy and the above inequality also shows the following: 
\begin{align*}
    \Bigl\vert \widetilde{H}_\alpha(T) - \frac{1}{1-\alpha}\log\bigl(\Vert [\omega_1,\ldots ,\omega_m]\Vert^\alpha_\alpha \bigr) \Bigr\vert \, \le\, & g\Bigl( \sum_{i=1}^m 4\omega_i \frac{\sigma_i^2}{ \sigma^2} + \sum_{i=2}^m\sum_{j=1}^{i-1} 8(i-1)\omega_i\exp\Bigl(\frac{-\Vert\mu_i - \mu_j \Vert^2_2}{\sigma^2} \Bigr) \Bigr).
\end{align*}
Therefore, following the Triangle inequality and the definition $\widetilde{H}_\alpha(X|T) = \widetilde{H}_\alpha(X,T)-\widetilde{H}_\alpha(T)$, the previous two inequalities prove that
\begin{align*}
    &\Bigl\vert \widetilde{H}_\alpha(X|T) - \Bigl( \frac{1}{1-\alpha}\log\bigl(\Vert \widehat{\lambda}_{x\otimes t}\Vert^\alpha_\alpha \bigr)- \frac{1}{1-\alpha}\log\bigl(\Vert [\omega_1,\ldots ,\omega_m]\Vert^\alpha_\alpha \bigr)\Bigr) \Bigr\vert \\
    \, \le\, & 2g\Bigl( \sum_{i=1}^m 4\omega_i \frac{\sigma_i^2}{ \sigma^2} + \sum_{i=2}^m\sum_{j=1}^{i-1} 8(i-1)\omega_i\exp\Bigl(\frac{-\Vert\mu_i - \mu_j \Vert^2_2}{\sigma^2} \Bigr) \Bigr) \,=\, 2g(\Gamma)
\end{align*}
On the other hand, we can simplify the above expression as
\begin{align*}
   &\frac{1}{1-\alpha}\log\bigl(\Vert \widehat{\lambda}_{x\otimes t}\Vert^\alpha_\alpha \bigr)- \frac{1}{1-\alpha}\log\bigl(\Vert [\omega_1,\ldots ,\omega_m]\Vert^\alpha_\alpha \bigr) \\
   =\, & \frac{1}{1-\alpha}\log\bigl(\sum_{i=1}^m \omega_i^\alpha\Vert {\lambda}_{C_i}\Vert^\alpha_\alpha \bigr) - \frac{1}{1-\alpha}\log\bigl(\sum_{i=1}^m \omega_i^\alpha \bigr) 
   \\
   =\, & \frac{1}{1-\alpha}\log\bigl(\sum_{i=1}^m \frac{\omega_i^\alpha}{\sum_{j=1}^m \omega_j^\alpha}\Vert {\lambda}_{C_i}\Vert^\alpha_\alpha \bigr)
\end{align*}
Note that the definition $f_{\alpha}(t) = \exp((1-\alpha)t)$ implies that $f^{-1}_{\alpha}(z) = \frac{1}{1-\alpha}\log(z)$, which connects to the entropy definition as $H(X|G=i) = f^{-1}_{\alpha}(\Vert {\lambda}_{C_i}\Vert^\alpha_\alpha)$. Hence, we combine the previous two equations to complete the proof:
\begin{align*}
    &\Bigl\vert \widetilde{H}_\alpha(X|T) -  f^{-1}_{\alpha}\Bigl( \sum_{i=1}^m \frac{\omega_i^\alpha}{\sum_{j=1}^m \omega_j^\alpha}f_{\alpha}\bigl(\widetilde{H}_\alpha(X|G=i)\bigr)\Bigr) \Bigr\vert \, \le\, 2g(\Gamma)
\end{align*}
\end{proof}

\begin{theorem}[Truncated Conditional-Vendi Interpretation]\label{thm:trunc-vendi-agg}
Let $T\sim\sum_{i=1}^m \omega_i P_{T,i}$ with means $\mu_i$ and within-component variances 
$\sigma_i^2=\mathbb E\!\left[\|T-\mu_i\|_2^2 \vert G=i\right]$. 
Assume Gaussian kernel bandwidth $\sigma$ and normalized feature maps.  
Fix a truncation level $t\in\mathbb N$ and consider $\Gamma$ defined in \eqref{Eq: Gamma Definition}. Let $\widetilde{\boldsymbol\lambda}^{(t)}_{\mathcal X,\mathcal T}$ and 
$\widetilde{\boldsymbol\lambda}^{(t)}_{\mathcal T}$ be the top-$t$ truncated, renormalized eigenvalue vectors of the joint and prompt operators, and define
\[
H^{(t)}(X\vert T)
:= H\!\big(\widetilde{\boldsymbol\lambda}^{(t)}_{\mathcal X,\mathcal T}\big)
   - H\!\big(\widetilde{\boldsymbol\lambda}^{(t)}_{\mathcal T}\big).
\]
Then, we have the following:
\begin{equation}\label{eq:vendi-bound}
\Bigl|\,H^{(t)}(X\vert T)\;-\;\sum_{i=1}^m \omega_i\,H^{(t)}(X\vert G=i)\Bigr|
\;\le\; \sqrt{t\Gamma}\,\log(\tfrac{{t}}{\Gamma}).
\end{equation}
Equivalently, for the truncated Conditional-Vendi score 
$\mathrm{Vendi}^{(t)}=\exp(H^{(t)})$,
\begin{align}\label{eq:vendi-mult}
\bigl(\frac{\Gamma}{t}\bigr)^{2\sqrt{t\Gamma}}\,
\prod_{i=1}^m \Big(\mathrm{Vendi}^{(t)}(X\vert G=i)\Big)^{\omega_i}
\;&\le\;\mathrm{Conditional\text{-}Vendi}^{(t)}(X\vert T) \nonumber\\
\;&\le\;
\bigl(\frac{t}{\Gamma}\bigr)^{2\sqrt{t\Gamma}}\,
\prod_{i=1}^m \Big(\mathrm{Vendi}^{(t)}(X\vert G=i)\Big)^{\omega_i}.
\end{align}
\end{theorem}
\begin{proof}
By the Frobenius bound in~\eqref{eq:frobenius-final} and the Hoffman--Wielandt inequality,
\[
\bigl\|\boldsymbol\lambda(C_{\mathcal X,\mathcal T}) - \boldsymbol\lambda(\textstyle\sum_i \omega_i C_i\otimes u_i u_i^\top)\bigr\|^2_2
\;\le\; {\Gamma}.
\]
Since $\{u_i\}$ are orthonormal, the nonzero eigenvalues of $\sum_i \omega_i C_i\otimes u_i u_i^\top$
are the union of the eigenvalues of $\omega_i C_i$, so
$\boldsymbol\lambda(\sum_i \omega_i C_i\otimes u_i u_i^\top)$ is obtained by stacking
$\omega_i\,\mathrm{Eigs}(C_i)$ over $i$. The same argument with $X$ fixed (or $C_X=0$) yields the prompt-side spectral approximation.
Projecting both joint and prompt spectra onto the $t$-truncated, renormalized simplex $\Delta_t$ (Lemma~A) and using nonexpansiveness of Euclidean projection (Lemma~B) gives
\[
\bigl\|\widetilde{\boldsymbol\lambda}^{(t)}_{\mathcal X,\mathcal T} - \widehat{\boldsymbol\lambda}^{(t)}_{\mathcal X,\mathcal T}\bigr\|_2
\;\le\; \sqrt{\Gamma},
\qquad
\bigl\|\widetilde{\boldsymbol\lambda}^{(t)}_{\mathcal T} - \widehat{\boldsymbol\lambda}^{(t)}_{\mathcal T}\bigr\|_2
\;\le\; \sqrt{\Gamma},
\]
where $\widehat{\boldsymbol\lambda}^{(t)}_{\mathcal X,\mathcal T}$ (resp.\ $\widehat{\boldsymbol\lambda}^{(t)}_{\mathcal T}$) denotes the $t$-truncated, renormalized vector formed from the stacked union of $\omega_i\,\mathrm{Eigs}(C_i)$ (resp.\ of the mixture weights $(\omega_i)$). On the truncated simplex space, we apply Lemma~\ref{lem:logdiff} coordinatewise and then Lemma~\ref{lem:schur} to obtain,
\[
\bigl| H(\widetilde{\boldsymbol\lambda}^{(t)}_{\mathcal X,\mathcal T}) - H(\widehat{\boldsymbol\lambda}^{(t)}_{\mathcal X,\mathcal T}) \bigr|
\;\le\; \,\sqrt{t\Gamma}\,\log\!\Big(\sqrt{\frac{t}{\Gamma}}\Big),
\qquad
\bigl| H(\widetilde{\boldsymbol\lambda}^{(t)}_{\mathcal T}) - H(\widehat{\boldsymbol\lambda}^{(t)}_{\mathcal T}) \bigr|
\;\le\; \,\sqrt{t\Gamma}\,\log\!\Big(\sqrt{\frac{t}{\Gamma}}\Big).
\]
Subtracting the two displays and using the triangle inequality yields
\[
\Bigl|\,H^{(t)}(X\vert T) - \Bigl(H(\widehat{\boldsymbol\lambda}^{(t)}_{\mathcal X,\mathcal T}) - H(\widehat{\boldsymbol\lambda}^{(t)}_{\mathcal T})\Bigr)\Bigr|
\;\le\; 2\,\,\sqrt{t\Gamma}\,\log\!\Big(\sqrt{\tfrac{t}{\Gamma}}\Big).
\]
Finally, by construction of the stacked union and the truncation on each block, 
$H(\widehat{\boldsymbol\lambda}^{(t)}_{\mathcal X,\mathcal T}) - H(\widehat{\boldsymbol\lambda}^{(t)}_{\mathcal T})
= \sum_{i=1}^m \omega_i\, H^{(t)}_i$,
which proves \eqref{eq:vendi-bound}. Exponentiating both sides results in
\eqref{eq:vendi-mult} since $\mathrm{Vendi}^{(t)}=\exp(H^{(t)})$ and $\exp(\sum_i \omega_i H^{(t)}_i)=\prod_i (\mathrm{Vendi}^{(t)}(X\vert G=i))^{\omega_i}$.
\end{proof}

\section{ Details of Truncated Conditional-Vendi Guidance in Section~\ref{sec: maintext guidance}}

In the guidance setting, let $K_Z=[k_{\mathcal Z}(\boldsymbol{z}^{(i)},\boldsymbol{z}^{(j)})]_{i,j=1}^n$ and
$K_T=[k_{\mathcal T}(t_i,t_j)]_{i,j=1}^n$ be normalized latent variable and prompt kernel matrices with unit diagonal entries.
Define the joint unit-trace PSD matrix
\begin{equation}\label{eq:A-def-app}
A \;=\; \tfrac{1}{n}\,(K_Z \odot K_T),\qquad \mathrm{Tr}(A)=1.
\end{equation}
We use the truncated Conditional-Vendi score introduced in the main text:
\begin{equation}\label{eq:cv-trunc-def}
\mathrm{Conditional\text{-}Vendi}^{(t)}(x_{1:n}\mid t_{1:n})
\;=\;
\exp\!\Bigl(H^{(t)}(A) - H^{(t)}\!\bigl(\tfrac{1}{n}K_T\bigr)\Bigr).
\end{equation}
Given a unit-trace PSD $M$ with eigenvalues $\lambda_1\ge\cdots\ge\lambda_n\ge 0$
and $S_t=\sum_{i=1}^{t}\lambda_i$ providing $c(\boldsymbol{\lambda}) = \frac{1-S_t}{t}$, we recall the following definition of the $t$-truncated entropy of the matrix $M$:
\begin{equation}\label{eq:shift-c}
H^{(t)}(M) \;=\; -\sum_{i=1}^{t} (\lambda_i + c(\boldsymbol{\lambda}))\,\log(\lambda_i + c(\boldsymbol{\lambda})).
\end{equation}
Here we describe how to compute the gradient of $H^{(t)}(M)$ in the guidance process.
Consider the eigendecomposition of symmteric matrix $M$ as $M=V\Lambda V^\top$ where $V_t=[v_1,\ldots,v_t]$, $\Lambda_t=\mathrm{diag}(\lambda_1,\ldots,\lambda_t)$ are the matrix of top-$t$ eigenvectors and vector of top $t$ eigenvalues.
Using the notation $\bar{\ell}_t = \frac{1}{t}\sum_{i=1}^{t}\log(\lambda_i+c(\boldsymbol{\lambda}))$, the gradient is
\begin{equation}\label{eq:Ht-grad}
\nabla_M H^{(t)}(M)
\;=\;
-\,V_t \log(\Lambda_t + c I_t)\,V_t^\top
\;+\;
\bar{\ell}_t\, V_t V_t^\top.
\end{equation}
Since $H^{(t)}\!\bigl(\tfrac{1}{n}K_T\bigr)$ does not depend on variable $Z$, we will obtain
\begin{equation}\label{eq:cv-grad}
\nabla_{\boldsymbol{z}^{(n)}} \mathrm{Conditional\text{-}Vendi}^{(t)}(z_{1:n};t_{1:n})
\;=\;
\mathrm{Conditional\text{-}Vendi}^{(t)}(z_{1:n};t_{1:n}) \cdot\nabla_{\boldsymbol{z}^{(n)}} H^{(t)}(A).
\end{equation}
Using \eqref{eq:A-def-app}, the Hadamard chain rule gives the following where
$\nabla_A H^{(t)}(A)$ is given by \eqref{eq:Ht-grad}:
\begin{equation}\label{eq:chain-KZ}
\nabla_{K_Z} H^{(t)}(A)
\;=\;
\tfrac{1}{n}\,\bigl(\nabla_A H^{(t)}(A)\bigr)\odot K_T.
\end{equation}
Note that only the $n$-th row and column of $K_Z$ depends on $\boldsymbol{z}^{(n)}$, and thus we have
\begin{equation}\label{eq:latent-grad}
\nabla_{\boldsymbol{z}^{(n)}} H^{(t)}(A)
\;=\;
\sum_{i=1}^{n-1}
\Bigl( (\nabla_{K_Z} H^{(t)}(A))_{in} + (\nabla_{K_Z} H^{(t)}(A))_{ni} \Bigr)\,
\nabla_{\boldsymbol{z}^{(n)}} k_{\mathcal Z}(\boldsymbol{z}^{(i)},\boldsymbol{z}^{(n)}).
\end{equation}
Combining \eqref{eq:cv-grad} and \eqref{eq:latent-grad} results in the guidance direction as follows:
\[
\boldsymbol{z}^{(n)}_{\tau-1}
\leftarrow
\mathrm{Sampler}\!\bigl(\boldsymbol{z}^{(n)}_\tau,\hat{\epsilon}_\theta(\boldsymbol{z}^{(n)}_\tau,\tau,t_n)\bigr)
+\eta\;\nabla_{\boldsymbol{z}^{(n)}} \mathrm{Conditional\text{-}Vendi}^{(t)}(z_{1:n};t_{1:n}).
\]
\textbf{Extension to Conditional-RKE Guidance.}
Given the unit value diagonal entries of the kernel matrix, the Conditional-RKE score takes the following form:
\[
\mathrm{Conditional\text{-}RKE}(x_{1:n}\mid t_{1:n})
= \frac{\|K_T\|_F^2}{\|K_Z\odot K_T\|_F^2}.
\]
Let $S=K_Z\odot K_T$. Taking the derivative with respect to $S$ gives the following
\[
\nabla_S \,\mathrm{Conditional\text{-}RKE}(z_{1:n};t_{1:n})
= -\,2\,\|K_T\|_F^2\,\frac{S}{\|S\|_F^4},
\;
\nabla_{K_Z} \,\mathrm{Conditional\text{-}RKE}(z_{1:n};t_{1:n})
= \nabla_S \,\mathrm{Conditional\text{-}RKE} \odot K_T,
\]
and the latent gradient follows via the same accumulation as \eqref{eq:latent-grad}. This variant
avoids a full spectral decomposition and can be implemented using $O(n^2)$ computations per iteration.
\section{Additional Related Works}

\textbf{Evaluation of deep generative models}.
Metrics for evaluating generative models are generally divided into reference-dependent and reference-free categories \citep{Borji2022}. Reference-dependent metrics compare generated and real data distributions, with common examples including FID \citep{heusel2017gans} and KID \citep{binkowski2018demystifying}. Other reference-based measures, such as the Inception Score \citep{Salimans2016}, Precision/Recall \citep{sajjadi2018assessing,kynkaanniemi2019improved}, and Density/Coverage \citep{naeem2020reliable}, jointly evaluate fidelity and diversity with respect to a reference dataset.  

Beyond fidelity, several works examine memorization and novelty. These include the authenticity score \citep{alaa2022faithful} and Feature Likelihood Divergence \citep{jiralerspong2023feature} for assessing generalization, as well as the rarity score \citep{han2023rarity} and KEN \citep{zhang24_KEN} for quantifying novelty. The memorization metrics are reference-based. In contrast, reference-free evaluations assess quality and diversity directly from the generated data. Notable examples include the Vendi score \citep{friedman2022vendi,pasarkar2023cousins} and RKE score \citep{jalali2023information} for diversity, and \citep{Nguyen2024} for evaluating the quality of generated data.

\textbf{Evaluation of conditional generative models}. 
The evaluation of prompt-based generative models, such as text-to-image and text-to-video systems, has been explored in several recent works. Most metrics focus on measuring alignment between prompts and outputs. A widely used example is CLIPScore \citep{hessel2021clipscore}, which computes cosine similarity in the CLIP embedding space. Other efforts have introduced benchmarks and curated prompt sets to evaluate broader aspects. For instance, HEIM \citep{lee2023holistic} assesses twelve criteria, including text–image alignment, image quality, and bias.  

However, alignment- and quality-focused metrics may overlook output diversity. \citet{astolfi2024consistencydiversity} emphasize that metrics centered on style or aesthetics can fail to capture variability across outputs for the same prompt. They propose computing per-prompt diversity using similarity functions and then averaging across prompts. Similarly, \citet{kannen2024aesthetics} extend the Vendi score to the per-prompt setting. Both approaches require generating multiple outputs for each prompt with different seeds. In contrast, our proposed Conditional-Vendi does not require repeated generations; instead, it quantifies model-induced diversity by analyzing variability across prompt types. Our theoretical results interpret Conditional-Vendi as an aggregation of diversity scores across prompt categories.  

\textbf{Information measures for evaluating conditional generative models}:  
\citet{kim2022mutual} propose the Mutual Information Divergence (MID) score, which fits multivariate Gaussian distributions to text and image representations and estimates their mutual information to quantify relevance in conditional generative models. In contrast, our score builds on the standard PSD matrix-based entropy measures applied to kernel matrices. Unlike MID, which relies on mutual information between Gaussian-fitted embeddings, the proposed diversity operates on kernel similarity values. 

\textbf{Conditional Generation with Guidance.} 
Controlling generative processes using specific conditions is increasingly important for practical applications, relying on inputs such as text prompts~\citep{kim2022diffusionclip, nichol2021glide, liu2023more}, class labels~\citep{dhariwal2021diffusion}, style images~\citep{mou2024t2i, zhang2023adding}, or human motions~\citep{tevet2022human}, among others. Approaches to conditional generation with guidance can be divided into training-based and training-free methods. Training-based strategies either learn a time-dependent classifier to steer the noisy sample $\boldsymbol{x}_t$ toward the target condition $\boldsymbol{y}$~\citep{dhariwal2021diffusion, nichol2021glide, zhao2022egsde, liu2023more}, or directly train the conditional denoising model $\boldsymbol{\epsilon}_\theta(\boldsymbol{x}_t,t,\boldsymbol{y})$ through few-shot adaptation~\citep{mou2024t2i, ruiz2023dreambooth}. In contrast, training-free guidance enables zero-shot conditional generation by using a pre-trained differentiable predictor, such as a classifier, loss function, or energy function, which measures how well a generated sample aligns with the target condition~\citep{he2023manifold, bansal2023universal, yu2023freedom, ye2024tfg}. Our Conditional-Vendi and Conditional-RKE guidance methods fall into this category, leveraging conditional entropy score guidance to improve the diversity of generated samples.

\textbf{Guidance for Improving Diversity.}

A common approach in diffusion-based generative models involves using guidance mechanisms to manage the trade-off between sample quality and diversity \citep{sadat2025no, ho2022classifier}. For instance, classifier-free guidance \citep{ho2022classifier} significantly improves alignment with prompts and overall image quality, but can reduce diversity due to its strongly deterministic conditioning. Recent studies have sought to mitigate this diversity issue. \citep{Sehwag_low_density} proposed a method that promotes diversity by explicitly sampling from low-density regions of the data manifold, though their approach works directly in pixel space, making it challenging to adapt to latent diffusion models. Another approach involves fine-tuning: \citep{Miao_2024_CVPR} present a reinforcement learning-based finetuning strategy that enhances diversity by optimizing an image-set-based diversity reward function.

\citep{askari2024improving} proposes contextualized Vendi Score Guidance (c-VSG) to boost generative diversity via the Vendi Score \citep{friedman2022vendi,ospanov2024towards}, but it relies on identical prompts, limiting its generality. In contrast, our approach uses Conditional-Vendi and Conditional-RKE Scores for prompt-aware guidance, enabling adaptive soft-clustering and effective conditioning on diverse prompts. Unlike c-VSG, which operates on latent features of reference images, our method directly guides the diffusion model’s latent space, reducing computational cost while enhancing prompt-aware diversity.

\section{Implementation Details and Hyperparameters}

\subsection{Conditional-Vendi and Conditional-RKE Evaluation Hyperparameters}
To select the bandwidth parameter $\sigma$, 
Similar to \citep{jalali2023information, ospanov2024towards}, we chose the Gaussian kernel bandwidth for each type of data as the smallest $\sigma$ that ensures a variance below 0.01 in the evaluated score over independent evaluations. We observed that for image data, $\sigma \in [20, 30]$; for text data, $\sigma \in [0.1, 0.8]$; and for video data, $\sigma \in [10, 20]$ can satisfy this requirement. Note that by selecting an overly large $\sigma$ value for text embeddings, Conditional-Vendi simplifies to the expectation of unconditional Vendi per prompt. For Truncated Conditional-Vendi and Information-Vendi, we set the truncation parameter $t$ to $10,000$ as suggested in \citep{ospanov2025do} in all the experiments.


\subsection{Conditional-Vendi and RKE Guidance Details}
In the kernel-based guidance experiments of Conditional-Vendi and Conditional-RKE, we considered a Gaussian kernel, which consistently led to higher output scores in comparison to the other standard cosine similarity kernel. We used the same Gaussian kernel bandwidth $\sigma$ in the RKE and Vendi experiments, and the bandwidth parameter choice matches the selected value in \citep
{friedman2022vendi, jalali2023information}. The numerical experiments were conducted on 4$\times$NVIDIA GeForce RTX 4090 GPUs, each of which has 22.5 GB of memory.

\subsection{Experimental Configuration for Table~\ref{tab:diversity_guidance}}
\label{configuration_sdxl_table}

We used Stable Diffusion XL with a resolution of 1024$\times$1024, a fixed classifier-free guidance scale of $W_{\text{CFG}} = 7.5$, and 50 inference steps using the DPM solver. We used 10,000 randomly selected prompts of the MS-COCO 2014 validation set and fixed the generation seed to be able to compare the effect of the methods. We used the following configuration to generate the results reported in Table~\ref{tab:diversity_guidance}.

The hyperparameter tuning was performed by performing cross-validation on the in-batch similarity score, selecting the hyperparameter values that optimized this alignment-based metric. Note that the in-batch similarity score accounts for both text-image consistency and inter-sample diversity as discussed in \citep{corso2024particleguidance}.

\textbf{c-VSG.}
We note that the reference \citep{askari2024improving} considered GeoDE \citep{ramaswamy2022geode} and DollarStreet \citep{dollar_street} datasets, in which multiple samples exist per input prompt. On the other hand, in our experiments, we considered the standard MSCOCO prompt set where for each prompt corresponds we access a single image, making the contextualized Vendi guidance baseline in \citep{askari2024improving} not directly applicable. Therefore, we simulated the non-contextualized version of VSG. For selecting the Vendi score guidance scale, we performed validation over the set $\{0, 0.04, 0.05, 0.06, 0.07\}$, following the procedure in \citep{askari2024improving}. A guidance frequency of 5 was used, consistent with the original implementation. To maintain stable gradient computation for the Vendi score, we implemented a sliding window of 300 most recently generated samples, as gradient calculations became numerically unstable for some steps beyond this threshold.

\textbf{latent Vendi Guidance.} We used a Gaussian kernel with bandwidth \(\sigma_{img} = 0.8\) and used $\eta=0.03$ as the weight of RKE guidance. To balance the effects of the diversity guidance in sample generation, the Vendi guidance update was applied every 10 reverse-diffusion steps in the diffusion process, which is similar to the implementation of Vendi score guidance in \citep{askari2024improving}.

\textbf{latent RKE Guidance.} We used a Gaussian kernel with bandwidth \(\sigma_{img} = 0.8\) and used $\eta=0.03$ as the weight of RKE guidance. To balance the effects of the diversity guidance in sample generation, the RKE guidance update was applied every 10 reverse-diffusion steps in the diffusion process, which is similar to the implementation of Vendi score guidance in \citep{askari2024improving}.

\textbf{latent Conditional-Vendi Guidance.} We used a Gaussian kernel with bandwidth \(\sigma_{img} = 0.8\) and used $\eta=0.03$ as the weight of Conditional-Vendi guidance. To balance the effects of the diversity guidance in sample generation, the RKE guidance update was applied every 10 reverse-diffusion steps in the diffusion process, which is similar to the implementation of Vendi score guidance in \citep{askari2024improving}.

\textbf{latent Conditional-RKE Guidance.} We considered the same Gaussian kernel for the image generation with bandwidth \(\sigma_{img} = 0.8\) and used bandwidth parameter \(\sigma_{text} = 0.3\) for the text kernel. The guidance hyperparameter was set to $\eta=0.03$, as in RKE guidance. Similar to the RKE and Vendi guidance, the Conditional-RKE diversity guidance was applied every 10 reverse-diffusion steps.
3

\section{Additional Numerical Results on Conditional-Vendi and Conditional-RKE Guidance}

To demonstrate the advantages of prompt-aware metrics over unconditional Vendi score, we explored their potential to enhance sample diversity in PixArt. We guided the model using both Truncated-Conditional-Vendi and standard Vendi score, following the methodology from \citep{askari2024improving}.

In our implementation, we applied guidance to PixArt in the latent space rather than the ambient space. This approach substantially reduced memory requirements from over 50 GB to approximately 20 GB while maintaining performance. We observed that latent-space guidance not only improves image diversity and quality but also offers significant computational efficiency gains. A comprehensive comparison between latent and ambient-space guidance is included in the Appendix.

Figure~\ref{fig:qualitative_pixart_complete} presents qualitative results using PixArt, demonstrating how prompt-aware guidance generates more relevant and contextually diverse images. Quantitative comparisons between Vendi and Conditional-Vendi guidance methods on PixArt are provided in Table~\ref{tab:diversity_guidance_pixart}, showing that Conditional-Vendi guidance enhances sample diversity (as measured by Vendi score and in-batch similarity) while preserving text-image alignment through competitive CLIPScore and KD metrics.

\begin{figure*}[t]
    \centering
    \includegraphics[width=\linewidth]{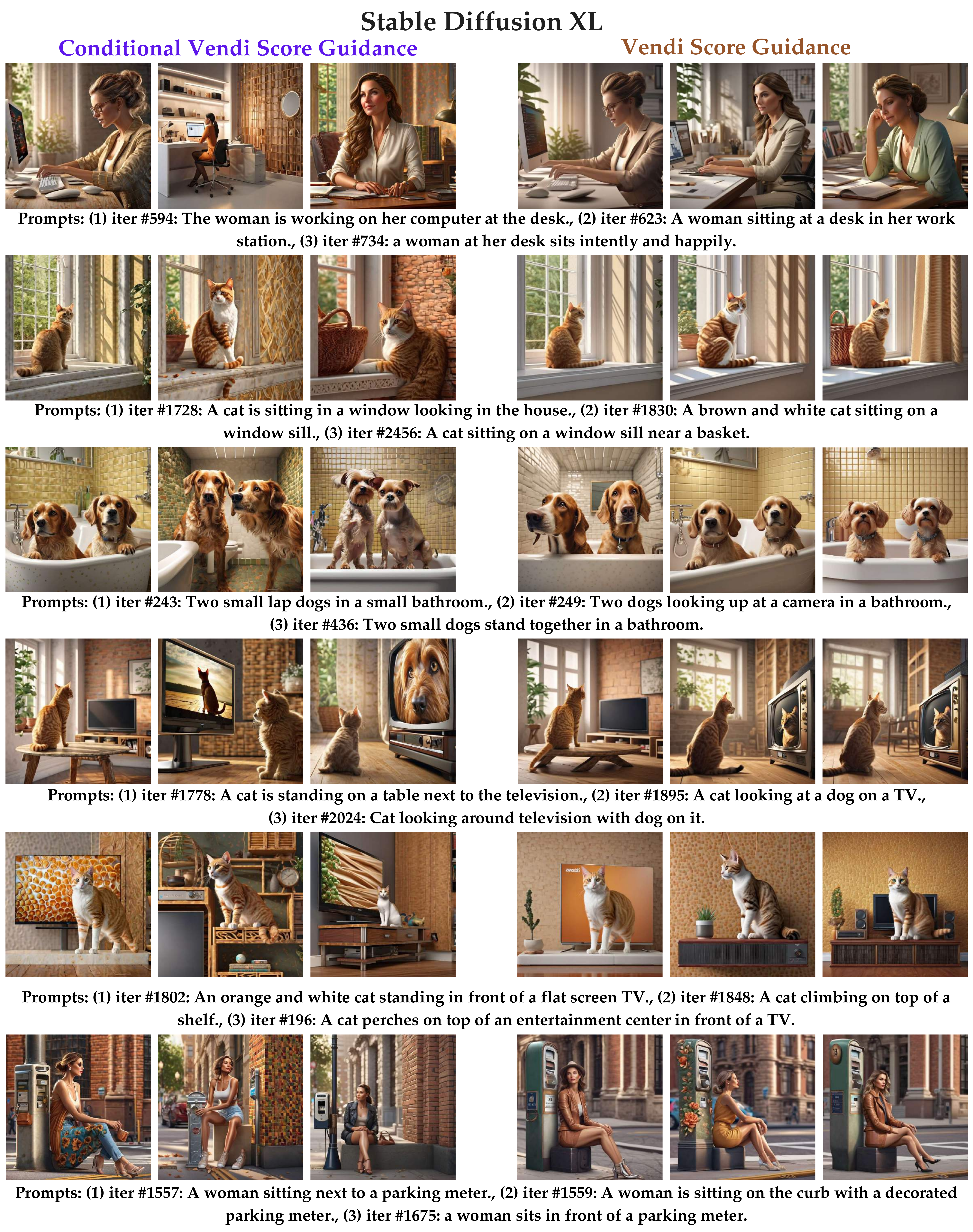}
    \caption{Qualitative comparison of Conditional-Vendi score guidance vs. Vendi score guidance using SD-XL.}
    \label{fig:qualitative_sdxl_complete}
\end{figure*}

\begin{figure*}[t]
    \centering
    \includegraphics[width=\linewidth]{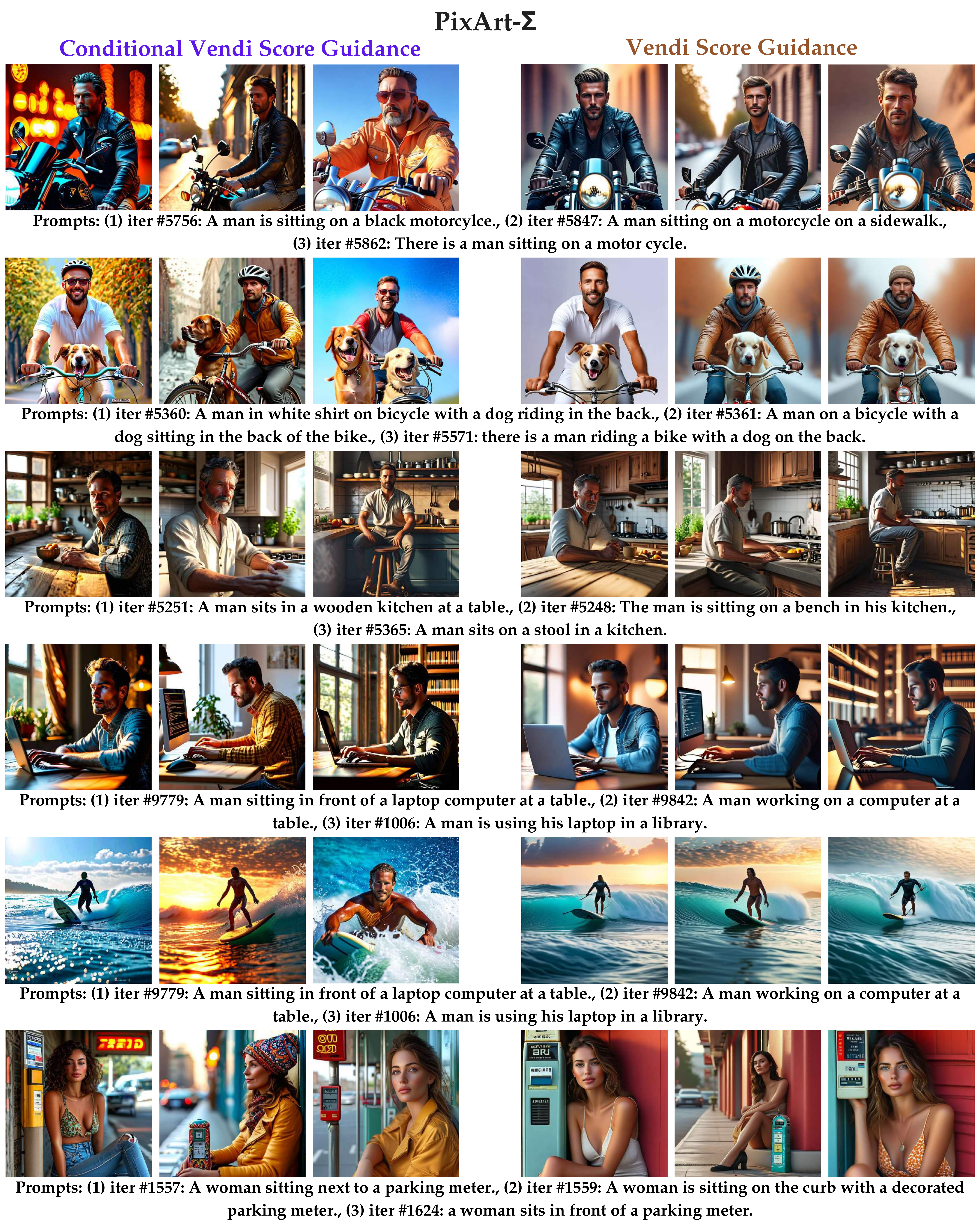}
    \caption{Qualitative comparison of Conditional-Vendi score guidance vs. Vendi score guidance using PixArt-$\Sigma$.}
    \label{fig:qualitative_pixart_complete}
\end{figure*}

\begin{table*}[h]
\centering
\caption{Quantitative comparison of guidance methods on PixArt}
\resizebox{0.95\linewidth}{!}{%
\begin{tabular}{lcccccc}
\toprule
\textbf{Guidance Method} & \textbf{CLIPScore $\uparrow$} & \textbf{KD$\times 10^2\downarrow$} & \textbf{Cond-Vendi$_\text{DINOv2
}$ $\uparrow$} & \textbf{Vendi$_\text{DINOv2
}$ $\uparrow$} & \textbf{In-batch Sim.$\times 10^2\downarrow$} \\
\midrule
Vendi$_\text{CLIP
}$& 29.63 & 38.14 & 26.28 & 261.73 & 83.26 \\
Vendi$_\text{Latent
}$& 30.39 & 36.20 & 28.95 & 298.15 & 81.50 \\
Conditional-Vendi$_\text{Latent
}$& 30.44 & 29.80 & 31.50 & 312.80 & 79.45 \\
\bottomrule
\end{tabular}
} 
\label{tab:diversity_guidance_pixart}
\end{table*}

\section{Additional Numerical Results on the Evaluation of Generative Models}

\subsection{Ablation Studies}

\textbf{Toy example on Gaussian Mixture Models.}
To validate that Conditional-Vendi and Information-Vendi accurately quantify model-induced diversity and prompt correlation, we evaluate these metrics on multiple Gaussian Mixture datasets. As illustrated in Figure~\ref{fig:gaussian_vendi_rke}, we generated separate 2D Gaussian distributions to represent text and image modalities, which we then paired through minimum weight bipartite graph matching.
In the first row, we fix the number of image Gaussian distributions (X) while increasing the number of text modes from 1 (less correlated) to 4 (highly correlated). As shown in the figure, Information-Vendi increases from 2.97 to 4.61, whereas Conditional-Vendi decreases from 2.21 to 1.34. These results indicate that conditional sample diversity is higher when paired with a single text mode compared to scenarios where images are fully aligned with prompts. The correlation between text and images is maximized when there is one cluster of images for each group of prompts.
In the second row, we used two Gaussian distributions for text (T) while varying the number of image modes (X) from 2 to 6. The results show that Conditional-Vendi score increases from 1.55 to 2.46, while Information-Vendi decreases from 4.01 to 3.42. This suggests that when the model generates more modes for a group of prompts, it produces greater model-induced diversity.

\begin{figure}
    \centering
    \includegraphics[width=\linewidth]{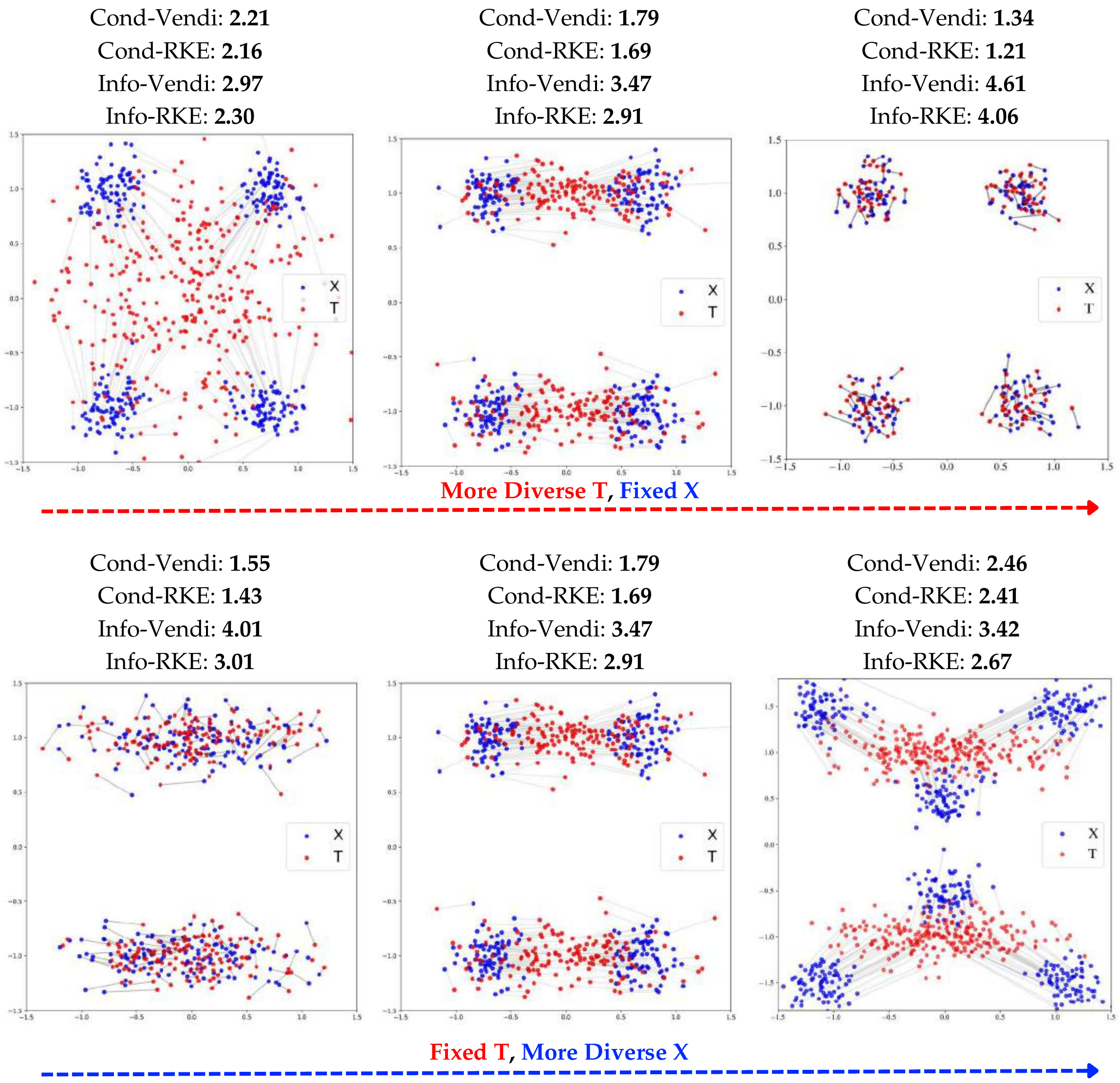}
    \caption{Comparing Conditional-Vendi and Information-Vendi on 2-D Gaussian Distribution. We used 1000 pair of points and used a Gaussian Kernel with bandwidth $\sigma = 0.6$.}
    \label{fig:gaussian_vendi_rke}
\end{figure}


\textbf{Correlation between prompts and generated output.}
To measure the correlation between text and image using Information-Vendi, we used MS-COCO captions to generate images with Stable Diffusion XL and Flux. We gradually substituted the generated images with random ones for the same prompts at different substitution rates. As the substitution rate increased, the correlation between the text and image pairs decreased. In Figure~\ref{fig:substitution_coco}, we measured Information-Vendi at various substitution rates and observed that as the substitution rate increased, Information-Vendi decreased, demonstrating that our score can successfully measure the correlation between text and image.
Unlike other correlation metrics, such as CLIPScore, which require the same embedding for both text and image, our method places no such restriction. This allows for the use of different embeddings for text and image. Furthermore, our approach can be easily generalized to other conditional models, such as text-to-text or text-to-video generation.

\begin{figure}[t]
    \centering
    \includegraphics[width=\linewidth]{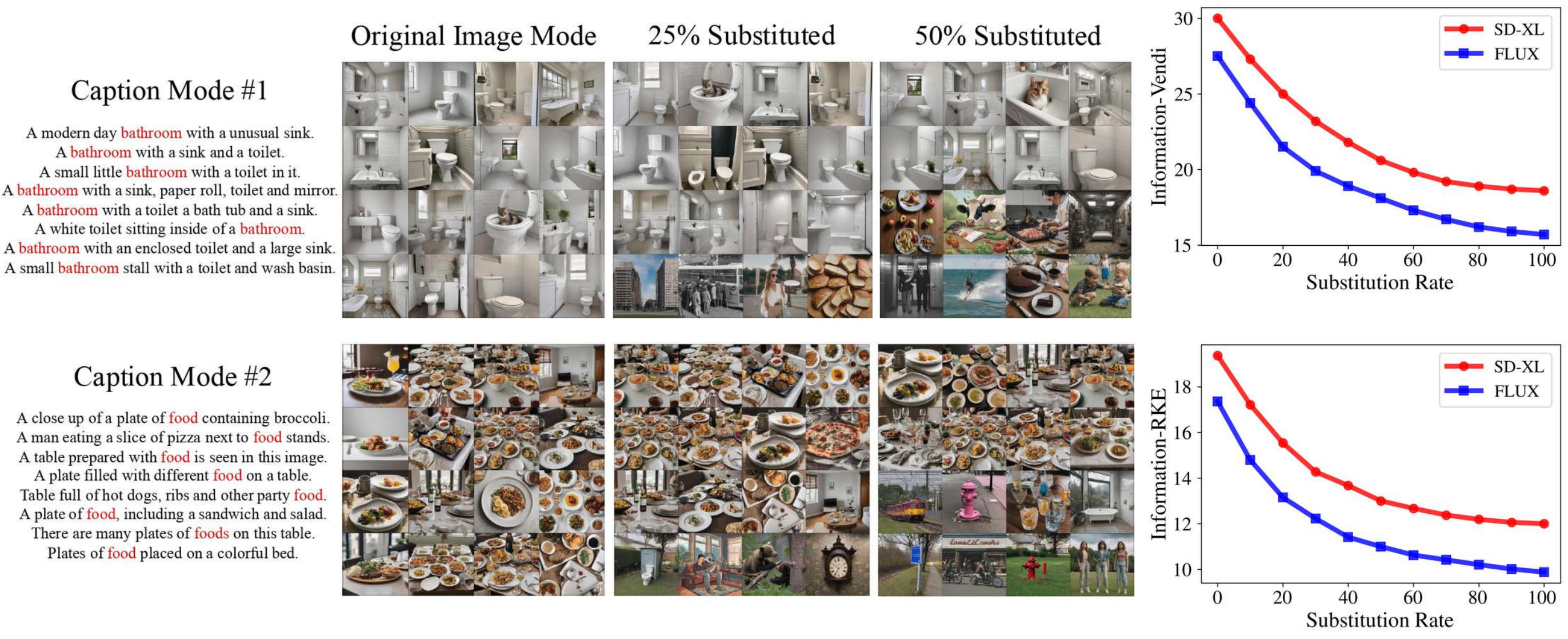}
     \caption{Substituting images generated from models trained on MS-  dataset.}
    \label{fig:substitution_coco}
\end{figure}

\textbf{Correlation between GroundTruth-Cluster-Vendi and Conditional-Vendi Scores.}
To validate the theoretical connection between the Vendi and Conditional-Vendi scores, we performed an experiment and evaluated a baseline metric called GroundTruth-Cluster-Vendi score. To measure the GroundTruth-Cluster-Vendi score, we utilize the side knowledge of the ground-truth clusters of the input prompts and then compute and average the regular Vendi scores for the data generated within each cluster. Mathematically, given $t$ sample cluster sets in $\mathcal{S}=\{S_1,\ldots, S_t\}$, which partition the input text indices $\{1,\ldots, n\}$, we define the Cluster-Vendi score as follows, where $|S_j|$ denotes the cardinality of subset $S_j$:
\begin{align*}
\text{Cluster-Vendi}\bigl(x_1,\ldots,x_n \,|\, \mathcal{S}\bigr) &:= \sum_{i=1}^t \frac{|S_i|}{n} \cdot \mathrm{Vendi}\Bigl(\{x_j : j \in S_i\}\Bigr), \\
\text{Cluster-RKE}\bigl(x_1,\ldots,x_n \,|\, \mathcal{S}\bigr) &:= \sum_{i=1}^t \frac{|S_i|}{n} \cdot \mathrm{RKE}\Bigl(\{x_j : j \in S_i\}\Bigr).
\end{align*}
Note that the above definition requires the knowledge of the clusters, which could be given by an oracle in the case of the GroundTruth-Cluster-Vendi score, or computed by a clustering algorithm such as K-Means to obtain the KMeans-Cluster-Vendi score. Observe that given the knowledge of the clusters revealed by an oracle, the GroundTruth-Cluster-Vendi score is a sensible definition of internal model diversity, which, as shown in Theorem~\ref{Theorem: 1}, is expected to correlate with our defined Conditional-Vendi score.

In the numerical settings of Section~\ref{specified_vs_unspecified}, where we know the ground-truth clusters based on the type of animal or fruit in the texts, we computed the value of the GroundTruth-Cluster-Vendi and GroundTruth-Cluster-RKE score and compared it with the evaluated Conditional-Vendi and RKE scores. As demonstrated in Figures~\ref{fig:animals-sdxl},\ref{fig:animals-pixart-complete}, \ref{fig:fruits-sdxl-complete}, \ref{fig:fruits-pixart-complete}, the two diversity scores, Conditional-Vendi, Cluster-Vendi and Conditional-RK,E and Cluster-RKE, highly correlate for the 4 simulated generative models in the experiments.

However, note that in a real-world scenario, we do not have access to the ground-truth clusters. To estimate the score, we should use a clustering algorithm such as K-Means to find the clusters and compute the Cluster-Vendi score. We note that the optimization problem addressed by standard clustering algorithms represents a challenging non-convex optimization, which, depending on the algorithm's initial point, could converge to different solutions.

\textbf{Measuring Conditional-Vendi across prompt types}
In this section, we conducted additional experiments similar to those in Figure~\ref{fig:inequities_pixart}. We created 10,000 prompts across different categories using GPT-4o and generated corresponding images with text-to-image models. 
We reported Conditional-Vendi and RKE for the top 3 groups in the text data on PixArt-$\alpha$, Stable Diffusion XL
and FLUX 
text-to-image generative models. 

As shown in Figure~\ref{fig:inequities_sdxl}, Figure~\ref{fig:inequities_pixart_appendix}, and Figure~\ref{fig:inequities_flux}, we observed the same behavior during these experiments: the Conditional-Vendi score for "dog" prompts was significantly higher than for the "airplane" and "sofa" categories. This observation suggests that the outputs of generative models are unbalanced when presented with different groups of text prompts.

\begin{figure}[h]
    \centering
    \includegraphics[width=\linewidth]{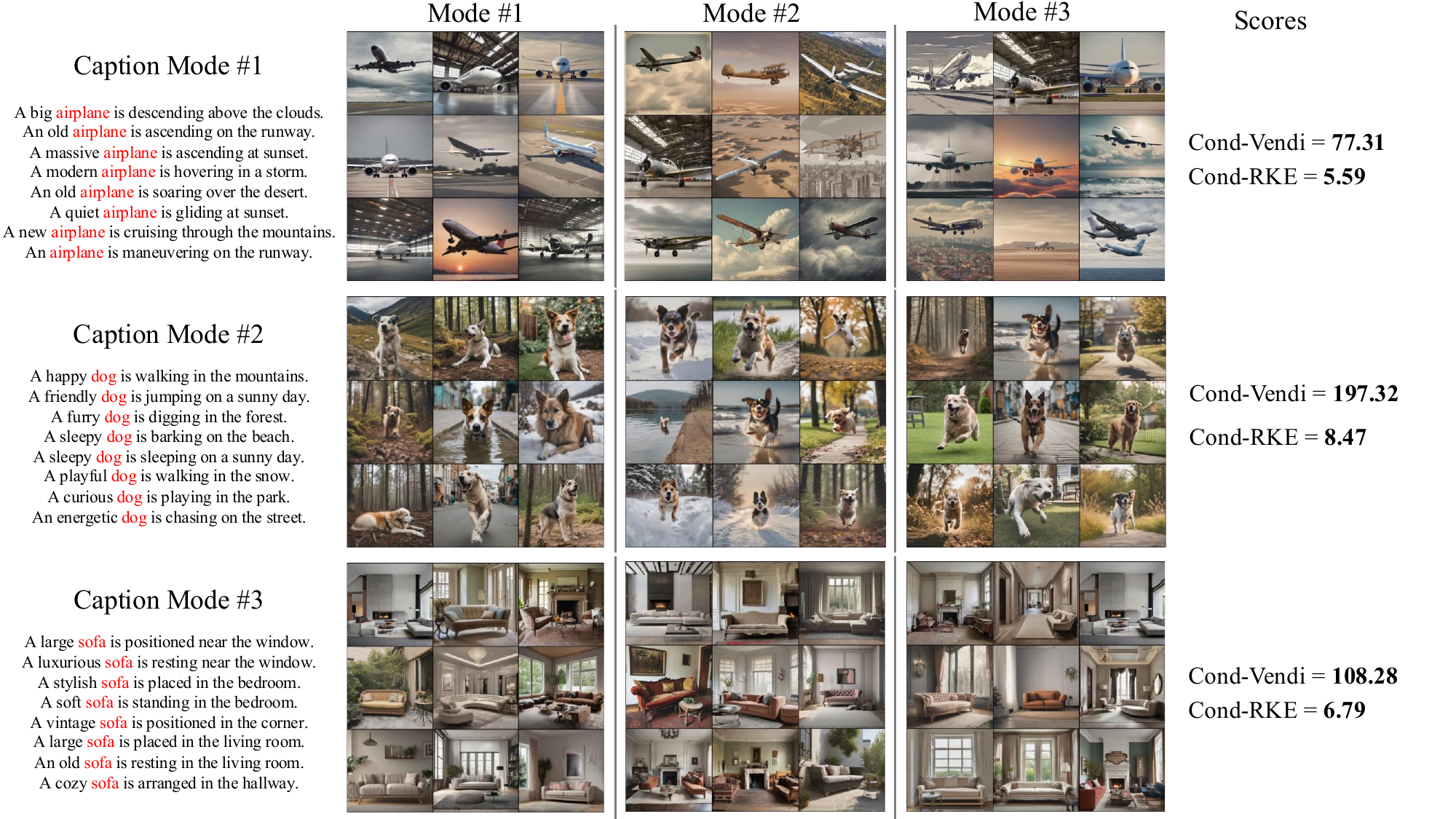}
    \caption{Quantifying image diversity for different clusters of text prompts. Images are generated using the Stable Diffusion XL model.}
    \label{fig:inequities_sdxl}
\end{figure}

\begin{figure}[t]
    \centering
    \includegraphics[width=\linewidth]{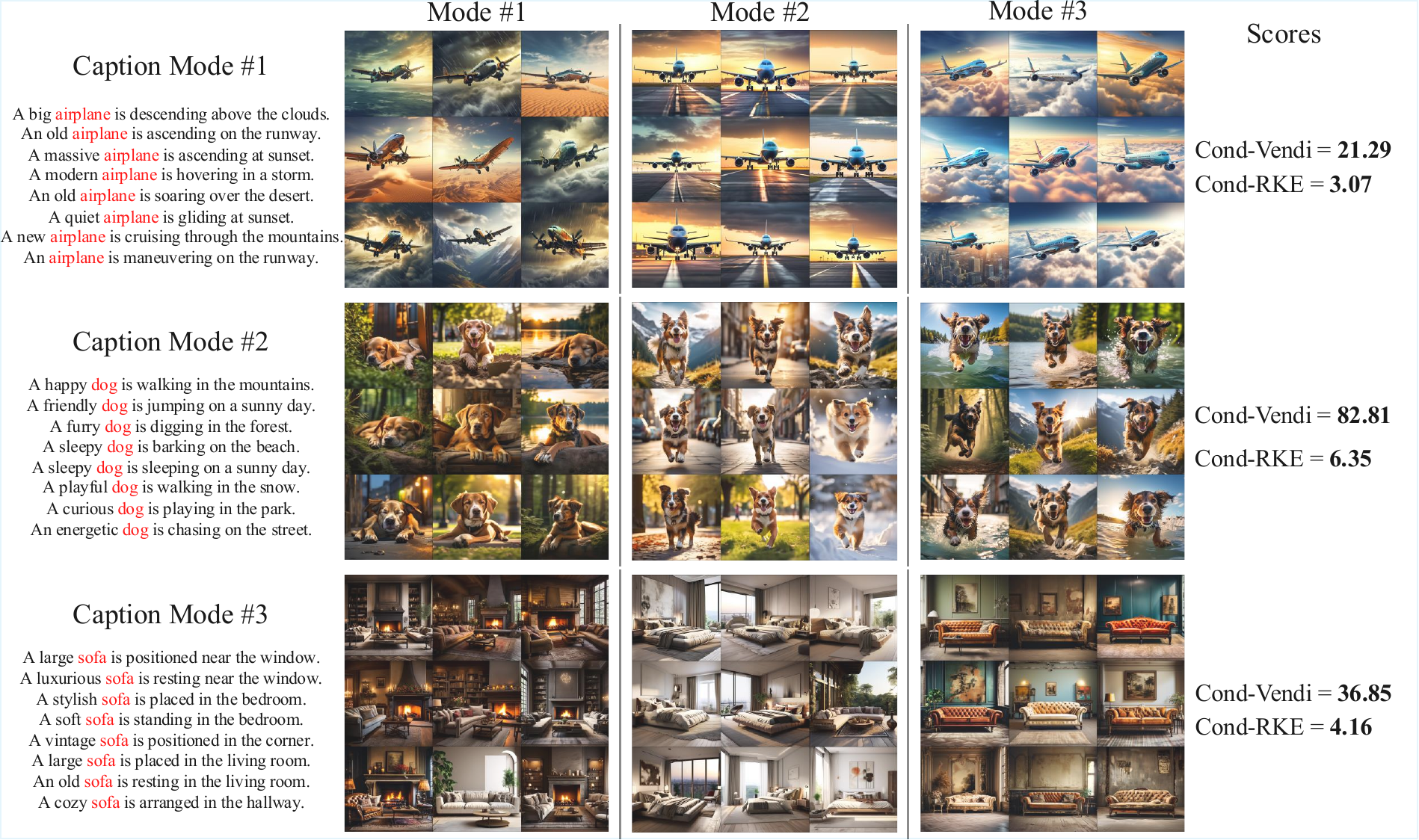}
    \caption{Quantifying image diversity for different clusters of text prompts. Images are generated using the PixArt-$\alpha$ model.}
    \label{fig:inequities_pixart_appendix}
\end{figure}

\begin{figure}[t]
    \centering
    \includegraphics[width=\linewidth]{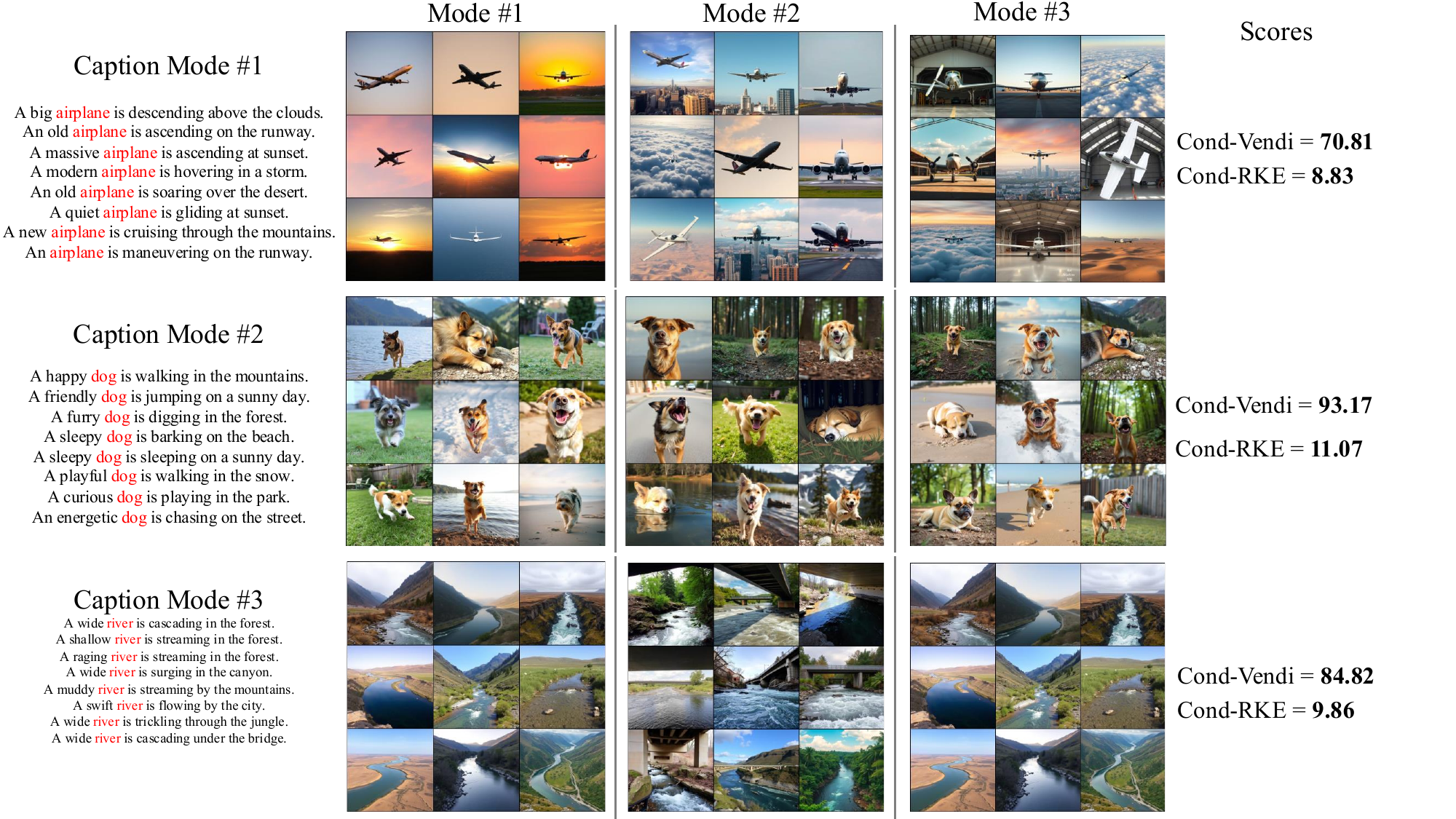}
    \caption{Quantifying image diversity for different clusters of text prompts. Images are generated using the Flux model.}
    \label{fig:inequities_flux}
\end{figure}

\subsection{Quantifying model-induced diversity via Conditional-Vendi and RKE.}
\label{specified_vs_unspecified}

To examine Conditional-Vendi in text-to-image models, we considered 10 types of animals generated by Stable Diffusion XL, as shown in Figure~\ref{fig:animals-sdxl}. We found that Conditional-Vendi and RKE increased more rapidly when the prompts did not specify the type of animal, indicating that the model-generated diversity was driven by its internal variability. In contrast, when the animal types were specified in the prompts, the increase in Conditional-Vendi and RKE was minimal, suggesting that the diversity in the outputs largely followed the constraints imposed by the text prompts. This demonstrates that Conditional-Vendi and Conditional-RKE effectively captures the difference between intrinsic model diversity and prompt-driven diversity. We further extended this experiment to different types of fruits and using a different generative model, PixArt-$\Sigma$ (Figures~\ref{fig:animals-pixart-complete}, \ref{fig:fruits-sdxl-complete}, \ref{fig:fruits-pixart-complete}), and observed the same trend. Additionally, we evaluated Cluster-Vendi and Cluster-RKE as ground-truth measures and observed the same pattern of Conditional-Vendi and RKE, confirming that Conditional-Vendi effectively captures model-intrinsic versus prompt-driven diversity.

To examine Conditional-Vendi in text-to-image models, we performed experiments quantifying diversity scores for unspecified and type-specified prompts across multiple categories and models. We considered nine experimental combinations consisting of two category types: animals and fruits and two state-of-the-art text-to-image models: Stable Diffusion XL (SDXL), and PixArt-$\Sigma$ \citep{pixart-sigma}. For each combination, we generated prompts for 10 different types within the category and created image samples by inputting the prompts into the respective model. In each experiment, we simulated 10 prompt-based generative models by considering image samples from $j$ types for $j \in {1,\ldots, 10}$.

For animals generated by SDXL, as shown in Figure~\ref{fig:animals-sdxl}, Conditional-Vendi and RKE increased more rapidly when the prompts did not specify the type of animal, indicating that model-generated diversity was driven by intrinsic variability. In contrast, when the animal types were specified in the prompts, the increase in Conditional-Vendi and RKE was minimal, suggesting that diversity largely followed the constraints imposed by the text prompts.

We further extended this analysis to fruits and objects and to the PixArt-$\Sigma$ model (Figures~\ref{fig:animals-pixart-complete}, \ref{fig:fruits-sdxl-complete}, \ref{fig:fruits-pixart-complete}). Across all categories and models, we observed the same trend: Conditional-Vendi and RKE increased rapidly for unspecified prompts but grew slowly when the type was specified, validating the correlation between Conditional-Vendi and intrinsic model diversity.

To further validate these results, we evaluated Cluster-Vendi and Cluster-RKE as ground-truth measures of non-prompt-induced diversity. The observed patterns mirrored those of Conditional-Vendi and RKE, confirming that Conditional-Vendi effectively captures the difference between intrinsic model diversity and prompt-driven diversity.


\begin{figure}[t]
    \centering
    \includegraphics[width=0.9\linewidth]{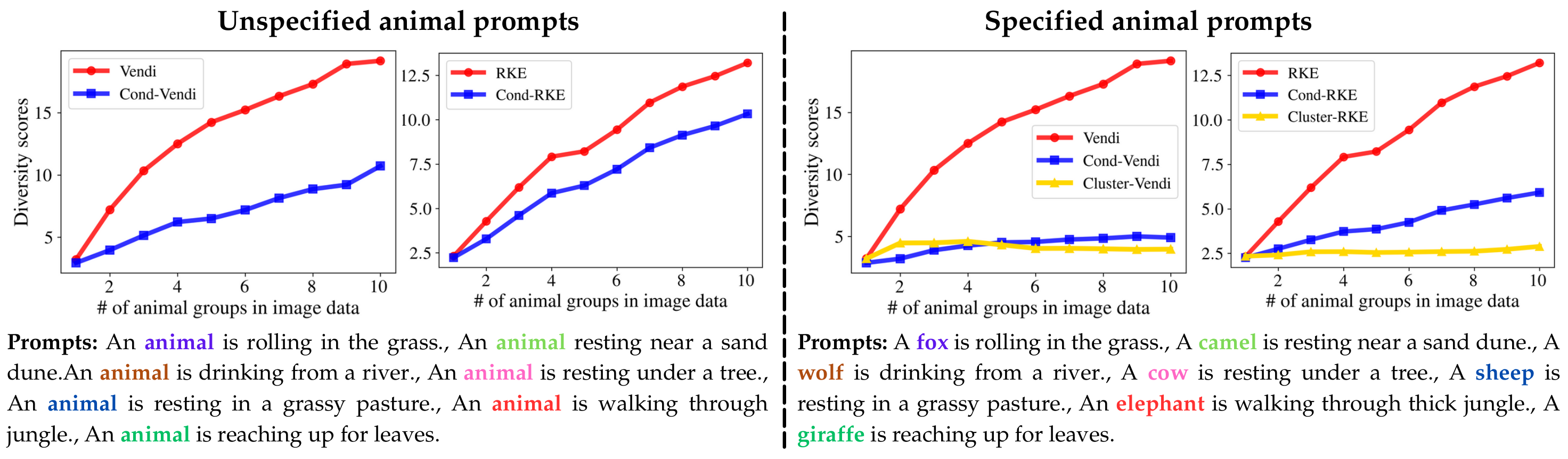}    \includegraphics[width=\linewidth]{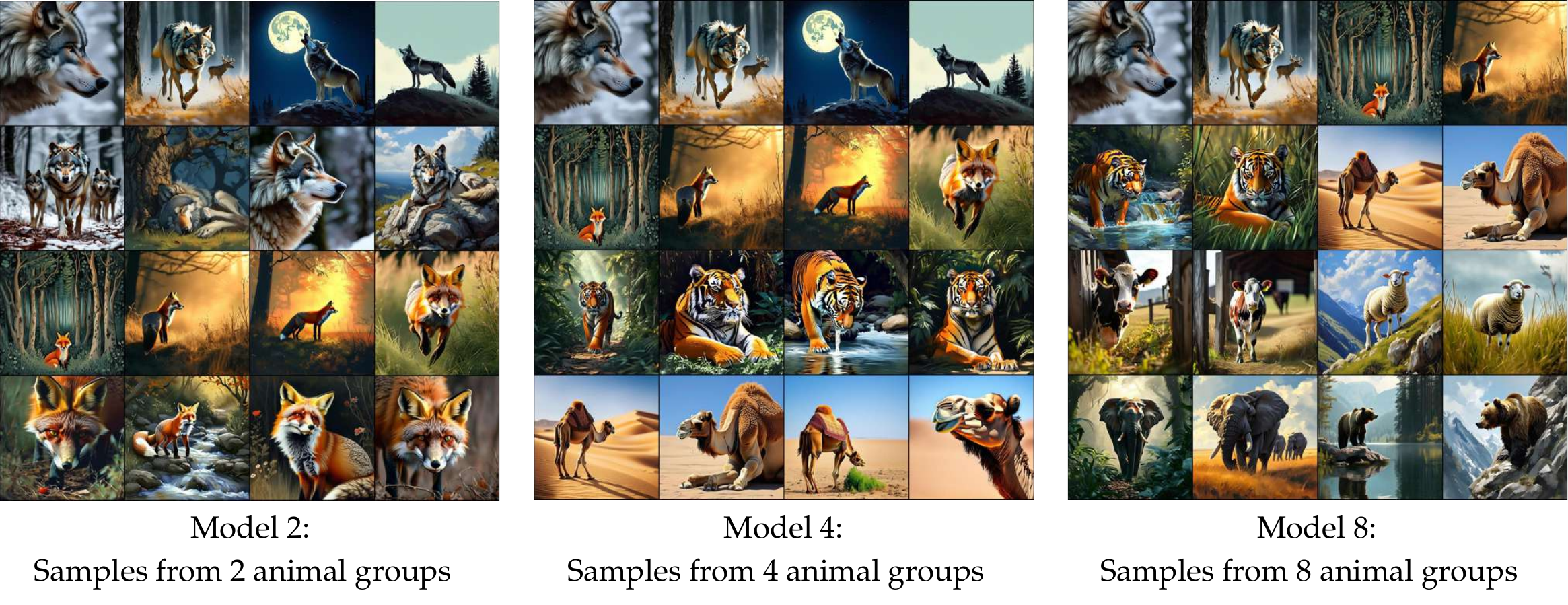}
    \caption{Evaluated Conditional-Vendi, Vendi, Conditional-RKE, and RKE scores on animal samples generated by PixArt$\Sigma$. (Left Plot) We do not specify the animal types in the prompt (Right Plot) we specify the animal types in the prompt.}
    \label{fig:animals-pixart-complete}
    \vspace*{-3mm}
\end{figure}


\begin{figure}[t]
    \centering
    \includegraphics[width=\linewidth]{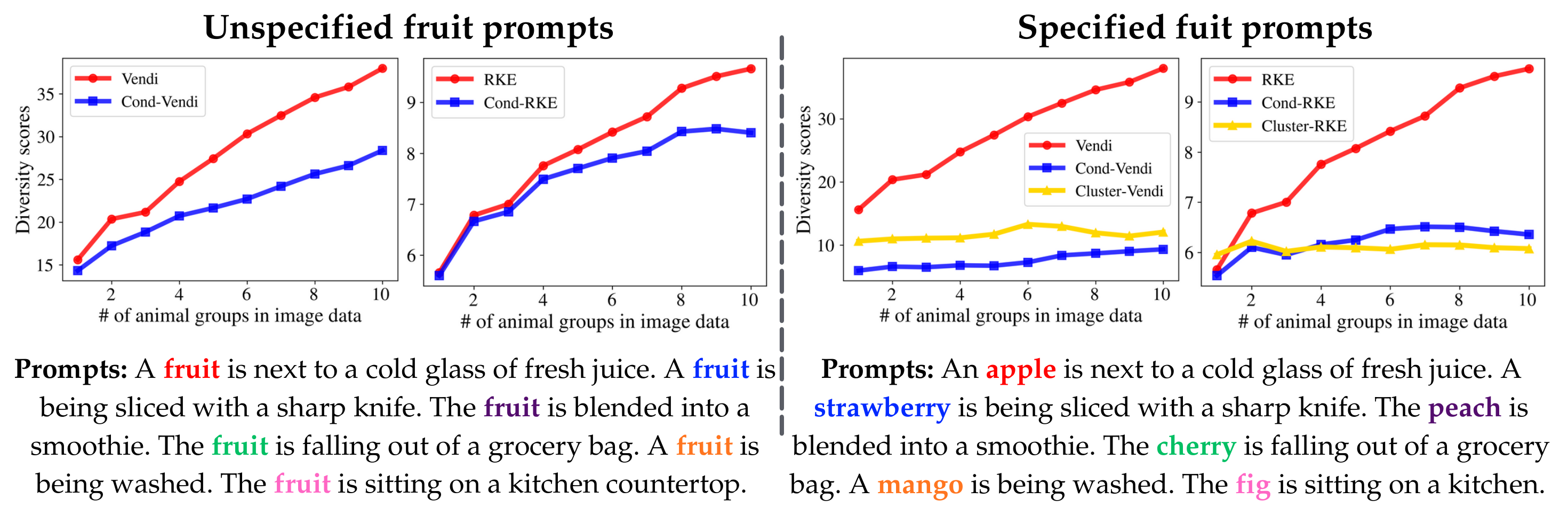}
    \includegraphics[width=\linewidth]{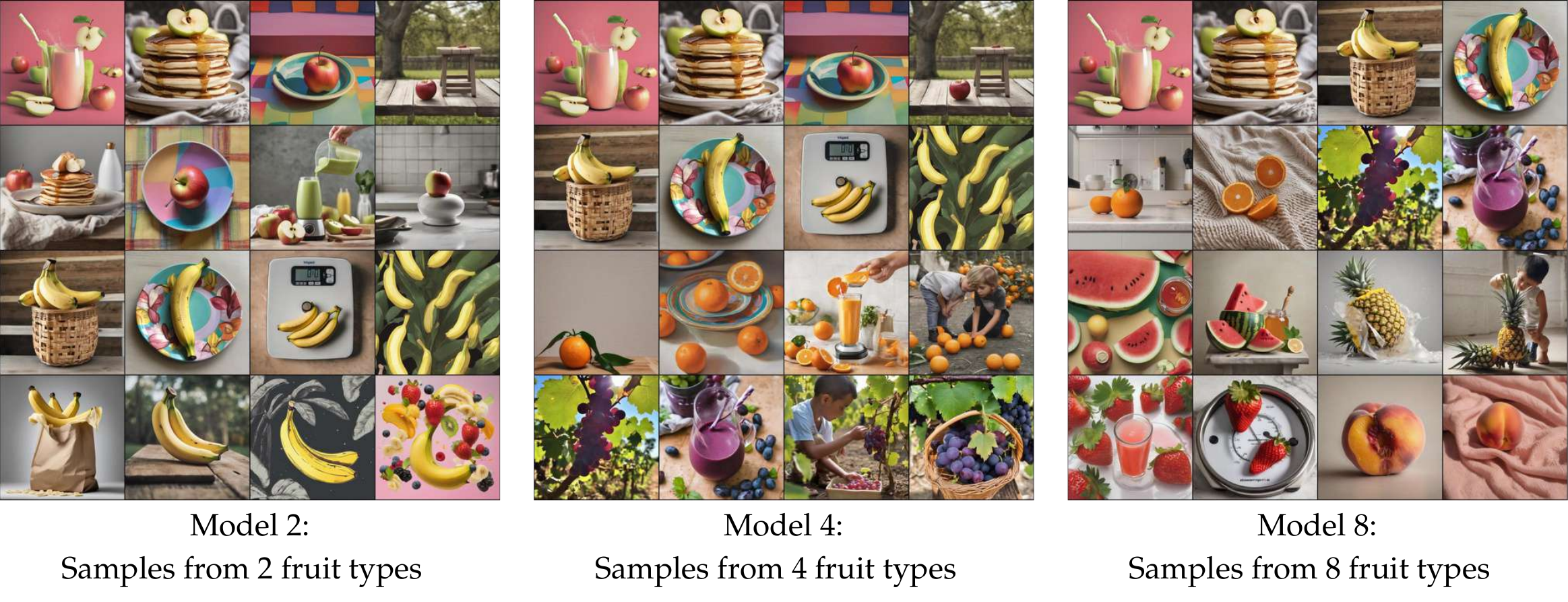}
    \caption{Comparing Conditional-Vendi with Vendi on different fruit types generated by Stable Diffusion-XL.}
    \label{fig:fruits-sdxl-complete}
    \vspace*{-3mm}
\end{figure}

\begin{figure}
    \centering
    \includegraphics[width=\linewidth]{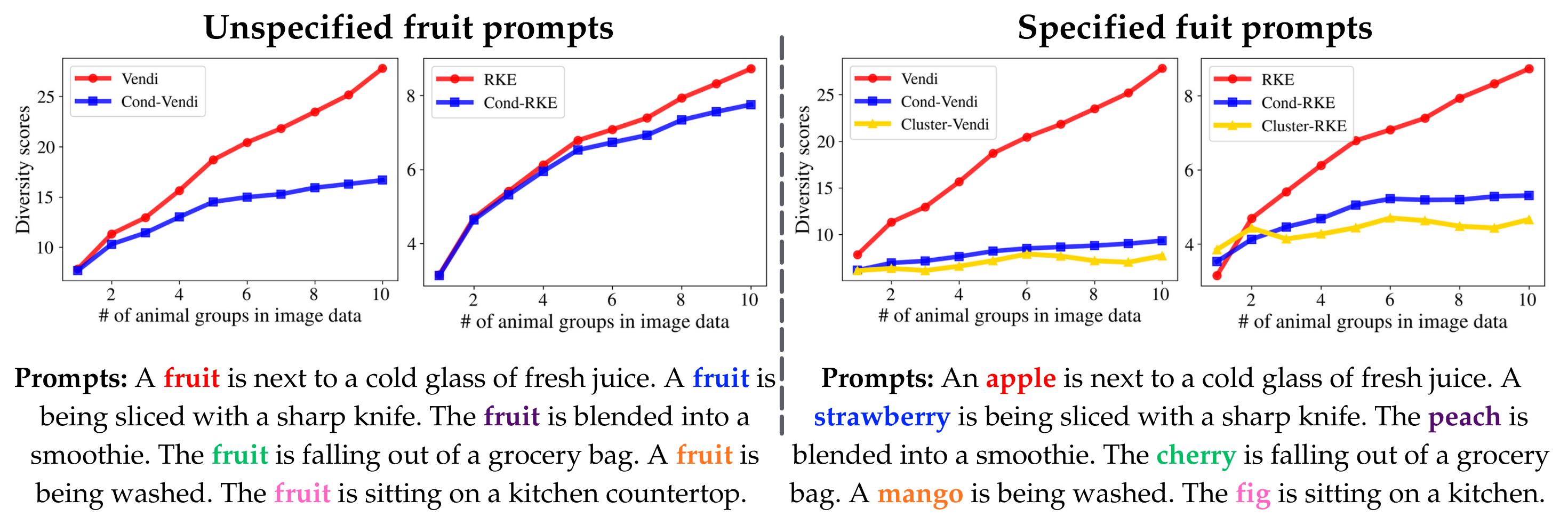}
    \includegraphics[width=\linewidth]{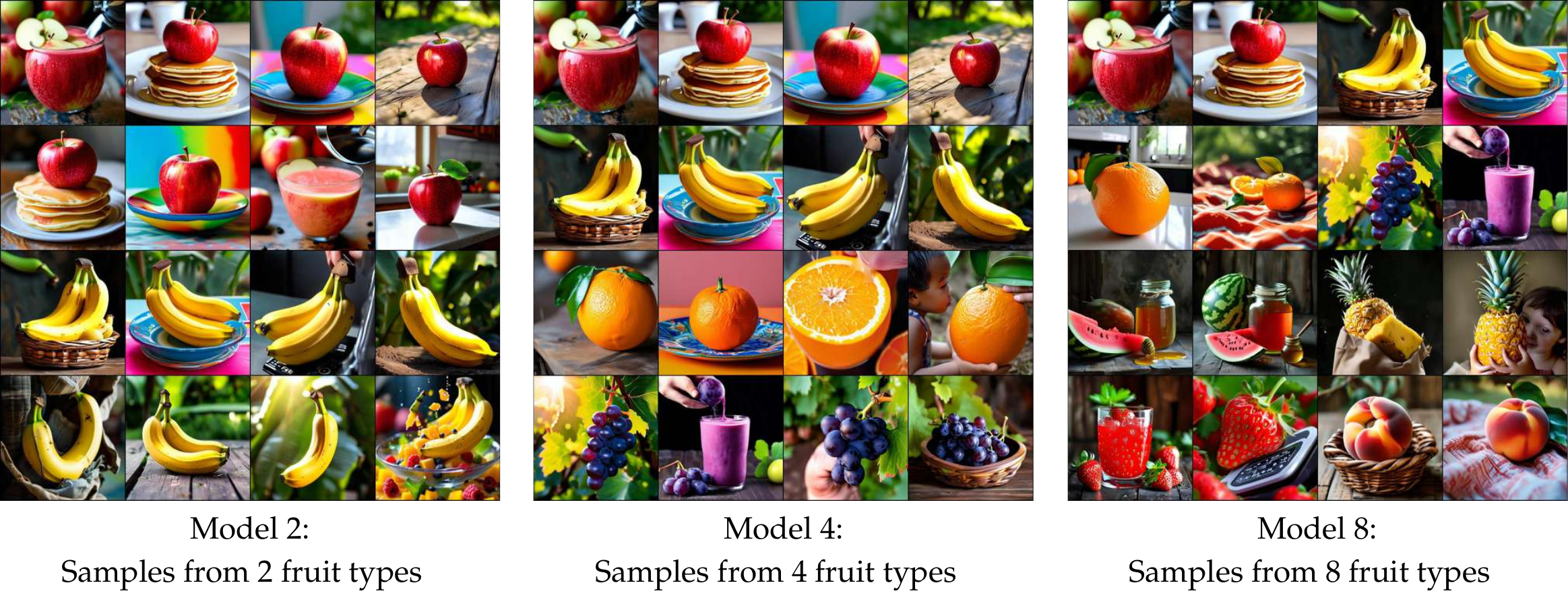}
    \caption{Comparing Conditional-Vendi with Vendi on different fruit types generated by PixArt-$\Sigma$.}
    \label{fig:fruits-pixart-complete}
\end{figure}






\subsection{Convergence Analysis of Conditional-Vendi Score}

To assess the convergence of the Conditional-Vendi and Conditional-RKE scores, we conducted experiments for different sample sizes on samples generated with SDXL and Kandinsky using prompts from the MS-COCO 2014 validation set. We used the cosine similarity for the finite-dimensional kernel and the Gaussian kernel for the infinite-dimensional kernel.
Our results, presented in Figure~\ref{fig:convergence_finite_infinite_kernels_compelte}, show that for RKE, Conditional-RKE converged, while for Conditional-Vendi, the non-truncated score did not converge; our proposed truncated Conditional-Vendi converged with~15000 samples.

\begin{figure*}[t]
    \centering
    \includegraphics[width=\linewidth]{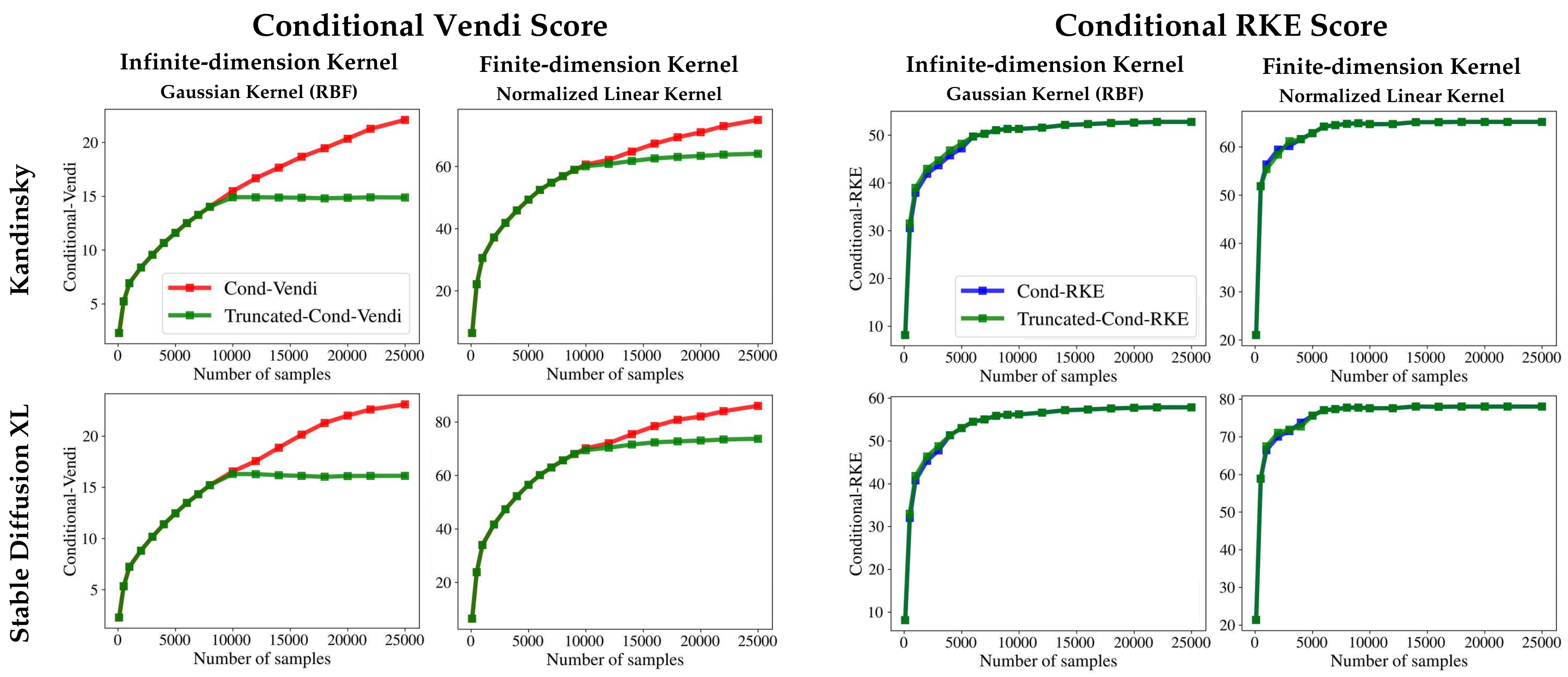}
    
    \caption{Statistical convergence of Conditional-Vendi scores with different sample sizes on data generated by Stable Diffusion-XL and Kandinsky using MS-COCO validation set prompts with finite-dimension cosine similarity and infinite-dimension Gaussian kernel. DINOv2 and CLIP embeddings are used for image and text modalities, respectively.}
    \label{fig:convergence_finite_infinite_kernels_compelte}
\end{figure*}

\subsection{Additional Numerical Evaluation of the Conditional-Vendi Score}
\label{additional_model_diversity_exp}


\begin{figure*}[t]
    \centering
    \includegraphics[width=0.96\linewidth]{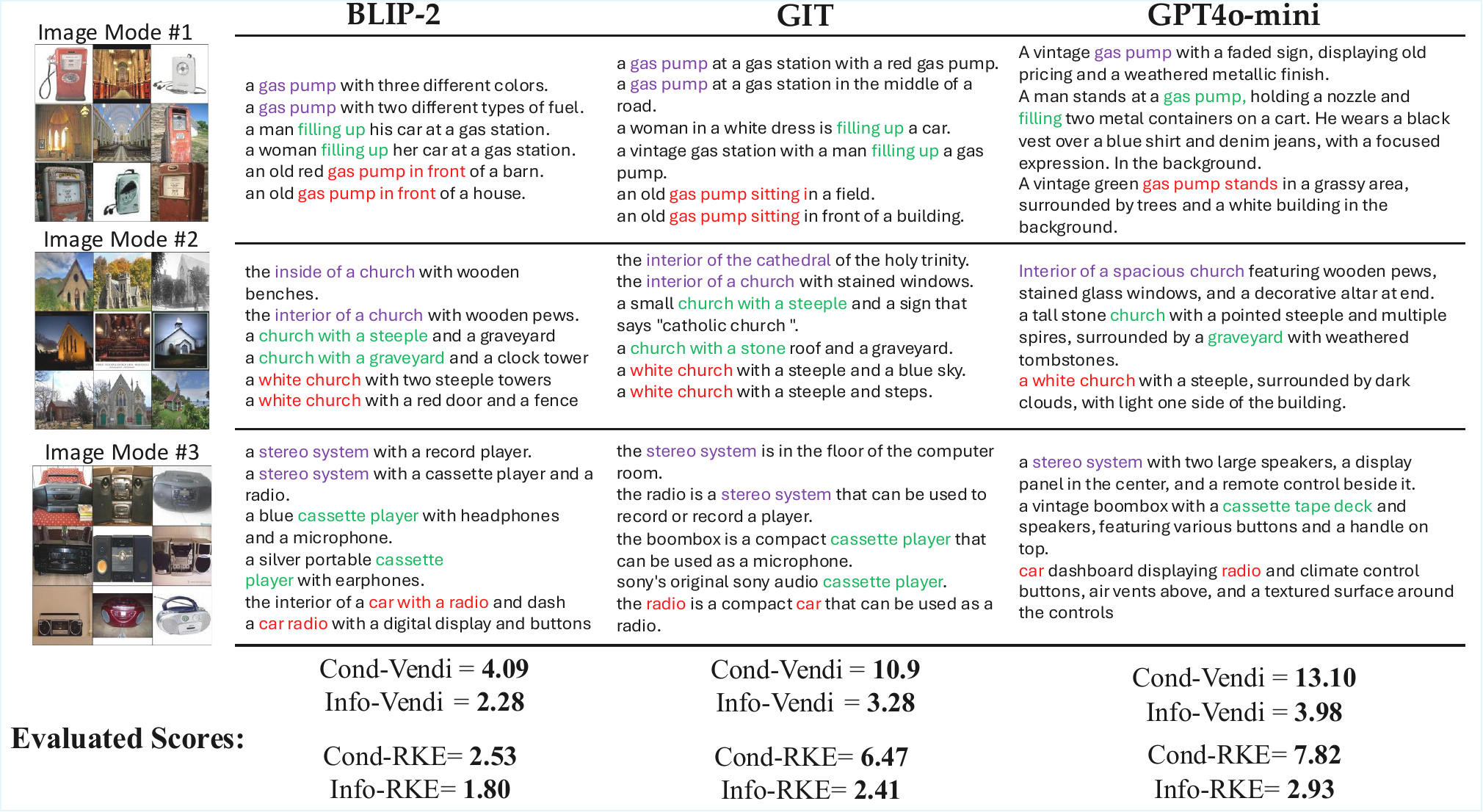}
    \caption{ Conditional-Vendi and Information-Vendi of image-captioning models for 3 image types}
    \label{fig:image-captioning}
\end{figure*}

\textbf{Text-to-Video Model Evaluation.}
For the experiments on video data, to ensure the fairness of our evaluation, we used VBench samples \citep{huang2023vbench}, which generated samples belonging to the 8 content categories. In Figure~\ref{fig:text-to-video-clustering-complete}, we used VideoCrafter-1, Show-1, and Open-Sora-1.2. We observed that VideoCrafter videos look less diverse and, in some cases, may not correlate significantly with the captions when compared to Open-Sora. Confirming this observation, the Conditional-Vendi and Information-Vendi scores were lower for VideoCrafter than those for Open-Sora.

\textbf{Image-Captioning Evaluation.}
For image captioning, we used 10 classes from the ImageNet dataset as input for BLIP-2, GIT and GPT4o-mini. In Figure~\ref{fig:image-captioning}, we compared captions for the top three groups of images: gas pump, church, and cassette player. GIT generated more diverse captions compared to BLIP, which was confirmed by the Conditional-Vendi scores. On the other hand, GPT4o-mini generated longer and more detailed captions compared to GIT, which was also reflected in the evaluated Conditional-Vendi and Information-Vendi scores.

\textbf{Large Language Models Evaluation.}
To evaluate Conditional-Vendi and RKE on LLMs, we varied the temperature parameter and generated 20K short stories with Llama 2 for each temperature setting as shown in Table~\ref{tab:llama_temperature}. We also provided a comparison of the generated prompts in Figure~\ref{fig:llama_samples}. The dataset covered 10 genres, each with 20 distinct subjects and themes. We further tested Conditional-Vendi and Conditional-RKE scores on Gemma 3 \cite{} and Phi 4 Mini \cite{microsoft2025phi4mini}. As shown in Tables~\ref{tab:gemma_temperature} and \ref{tab:qwq_temperature}, both Conditional-Vendi and Conditional-RKE increase with higher temperatures, indicating that the outputs become more diverse.

\begin{figure*}[t]
    \centering
    \includegraphics[width=0.96\linewidth]{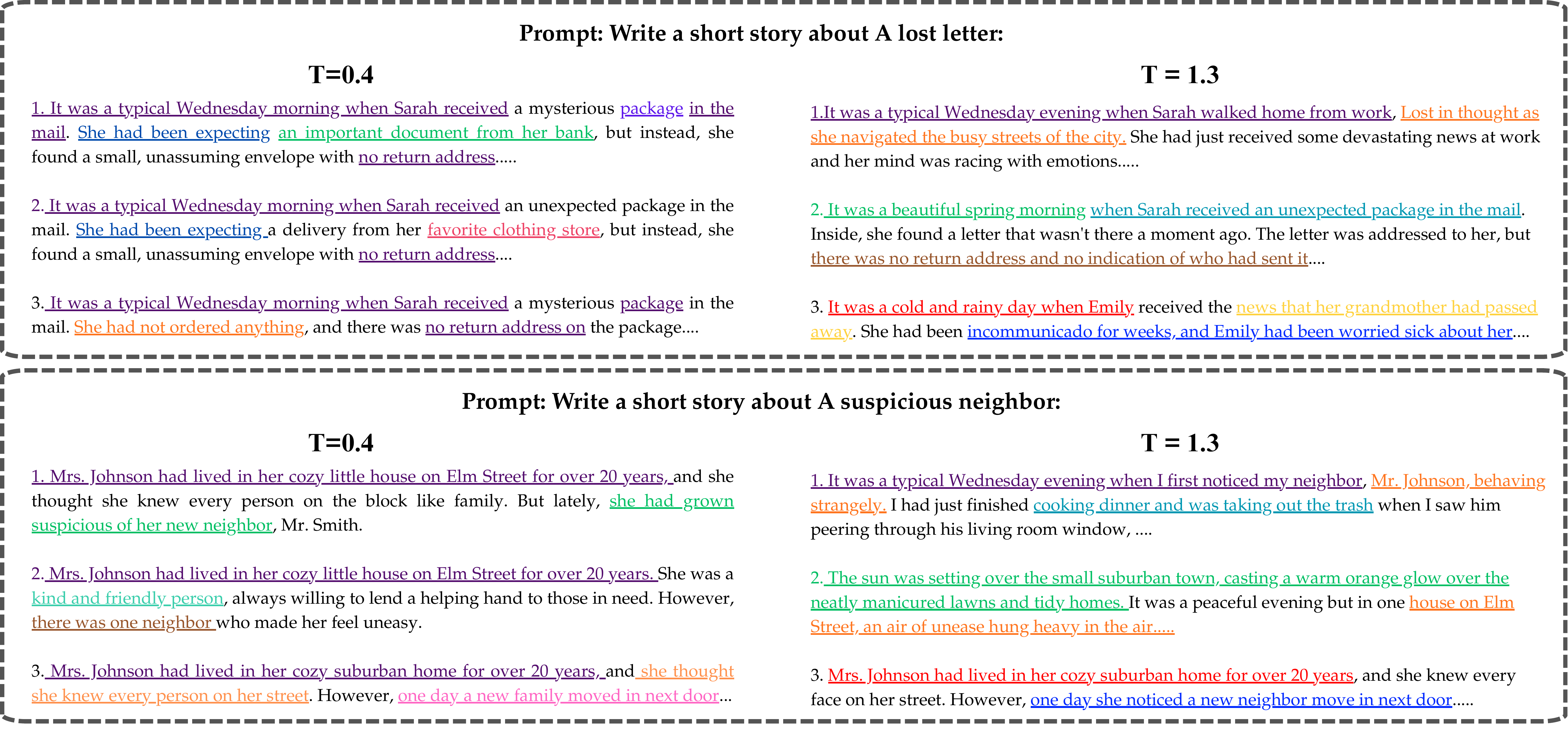}
    \caption{Llama 2 Samples generated with different temperature parameters.}
    \label{fig:llama_samples}
\end{figure*}

\begin{table}[h]
\centering
\caption{Conditional Vendi and RKE Scores evaluated for Gemma 3 with different temperature parameters.}
\begin{tabular}{lcccc}
\toprule
\textbf{Method} & \textbf{$T=0.4$} & \textbf{$T=0.7$} & \textbf{$T=1.0$} & \textbf{$T=1.3$}\\
\midrule
Conditional-Vendi & 40.82 & 42.82 & 44.03 & 48.23 \\
Conditional-RKE & 38.42 & 41.82 & 43.16 & 45.93 \\
\bottomrule
\end{tabular}
\label{tab:gemma_temperature}
\end{table}

\begin{table}[h]
\centering
\caption{Conditional Vendi and RKE Scores evaluated for Phi 4 Mini with different temperature parameters.}
\begin{tabular}{lcccc}
\toprule
\textbf{Method} & \textbf{$T=0.4$} & \textbf{$T=0.7$} & \textbf{$T=1.0$} & \textbf{$T=1.3$}\\
\midrule
Conditional-Vendi & 41.02 & 45.93 & 49.93 & 51.60 \\
Conditional-RKE   & 39.43 & 41.74 & 47.27 & 49.82 \\
\bottomrule
\end{tabular}
\label{tab:qwq_temperature}
\end{table}


    

    

\begin{figure}
    \centering
    \includegraphics[width=1\linewidth]{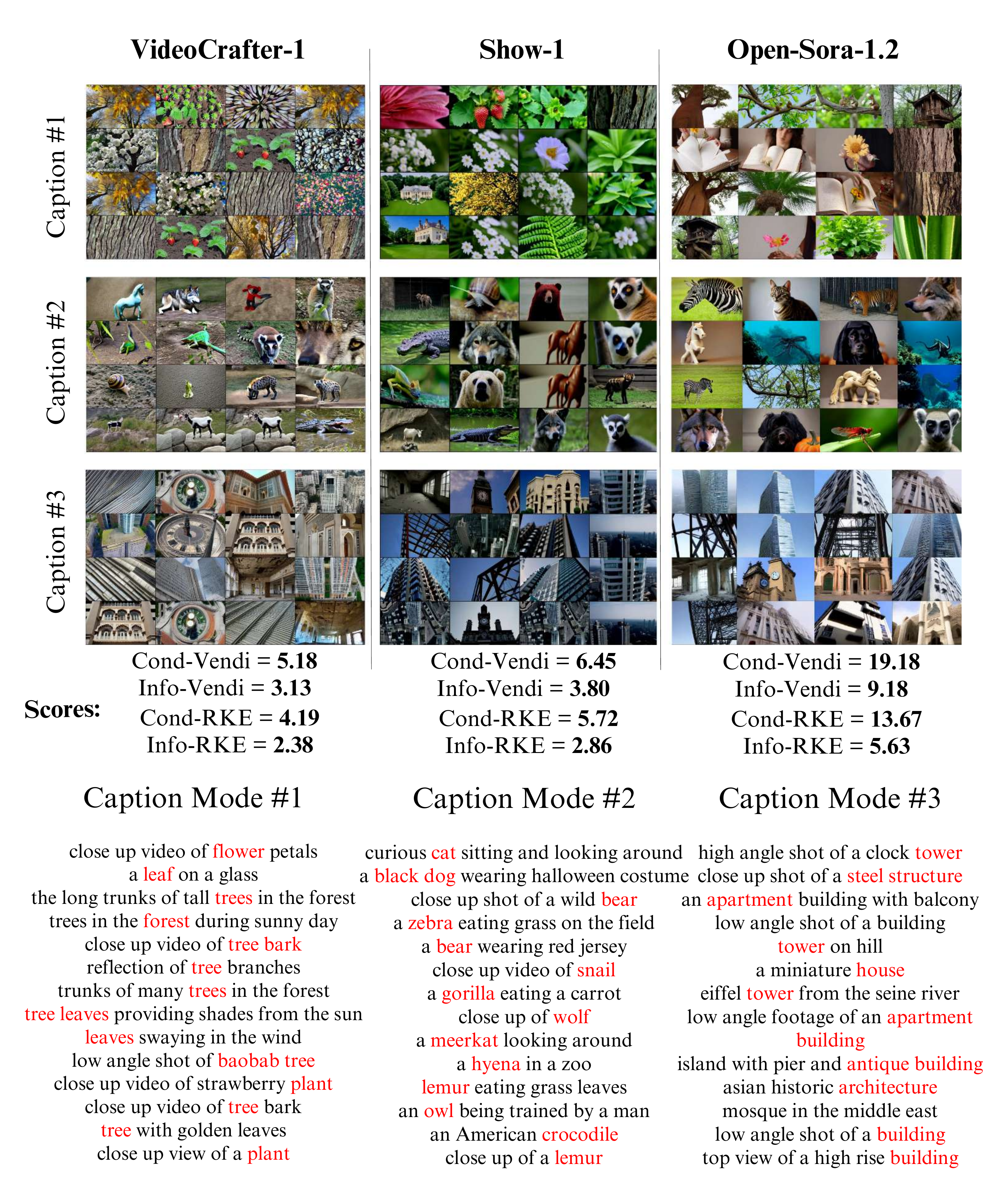}
    \caption{Measuring Conditional-Vendi and Information-Vendi for text-to-video models}
    \label{fig:text-to-video-clustering-complete}
\end{figure}

\textbf{Qualitative results for generative models trained on MS-COCO dataset}
In this section, we provide images and prompts corresponding to Figure~\ref{fig:clustering-text-vendi}. Figure~\ref{fig:mscoco_qualitive} illustrates three clusters obtained by applying KMeans to cluster MS-COCO validation set prompts into 1000 clusters. The images are presented for four generative models. Comparing the prompts with the generated images reveals that FLUX exhibits the highest alignment between text and image, while GigaGAN demonstrates greater diversity but misses some features of the prompts. These observations are further supported by the Conditional-Vendi and Information-Vendi metrics.

\begin{figure*}[t]
    \centering
    \begin{subfigure}{\linewidth}
        \centering
        \includegraphics[width=0.9\linewidth]{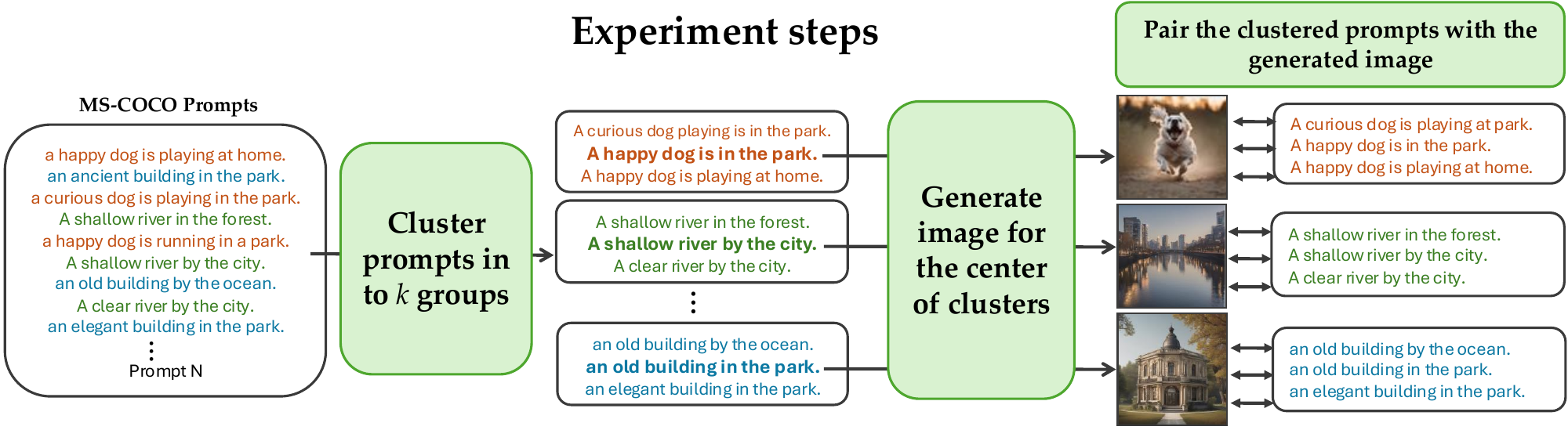}
    \end{subfigure}
    \begin{subfigure}{\linewidth}
        \centering
        \includegraphics[width=\linewidth]{figs/text_clustering_sdxl.pdf}
    \end{subfigure}
    \caption{Conditional and Information Vendi and RKE score comparison across text-to-image models. We clustered MS-COCO prompts into k groups and generated images for each cluster center. Within each cluster, we paired prompts with identical images. The results show increasing diversity and stronger correlation as the number of clusters grows, indicating that clusters become more relevant and diverse with finer partitioning.}
    \label{fig:clustering-text-vendi}
\end{figure*}

\begin{figure}
    \centering
    \includegraphics[width=0.6\linewidth]{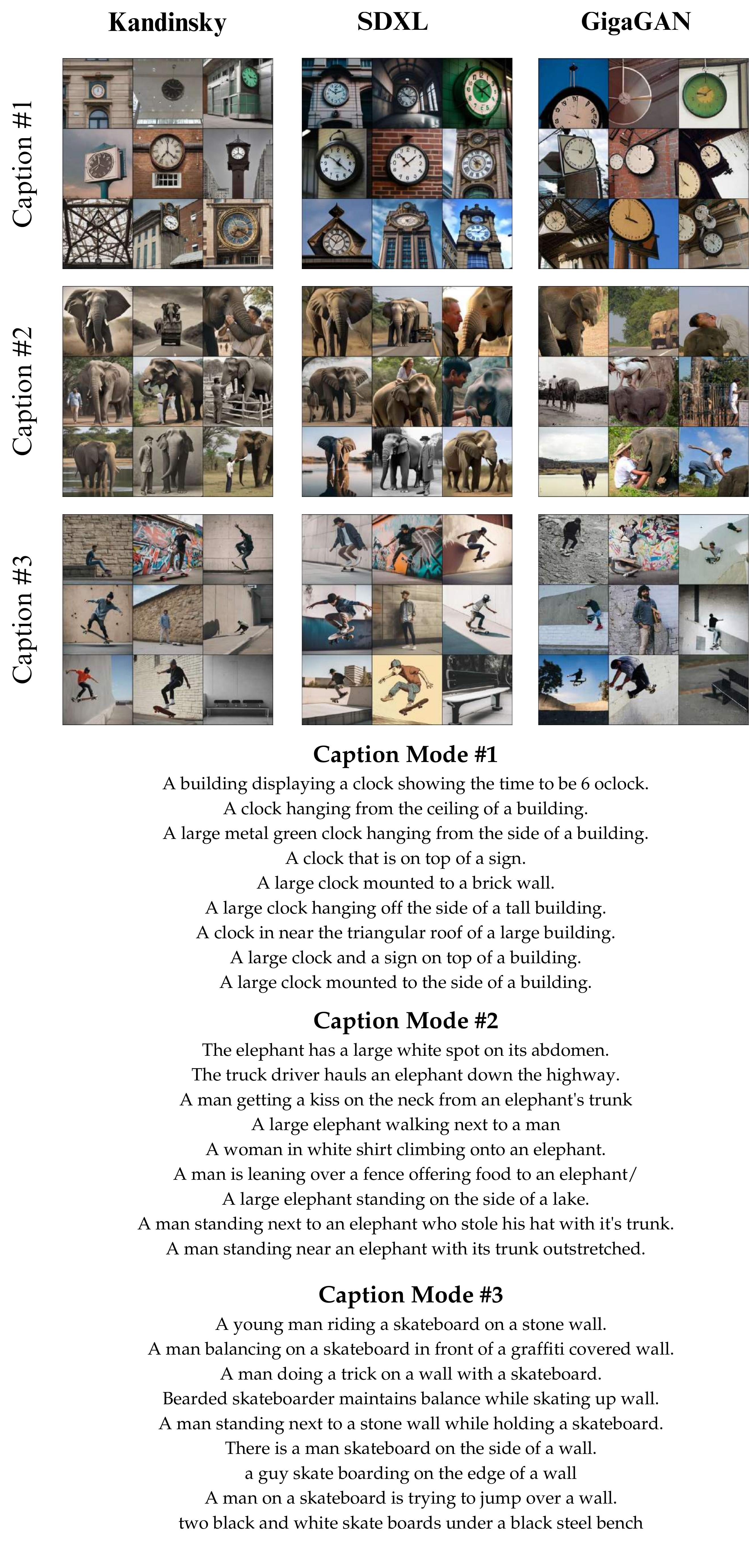}
    \caption{3 clusters of MS-COCO generated samples}
    \label{fig:mscoco_qualitive}
\end{figure}






\thispagestyle{empty}

\end{document}